\documentclass{article} 
\usepackage[margin=1in]{geometry}
\usepackage[utf8]{inputenc} %
\usepackage[T1]{fontenc}    %
\usepackage{lmodern}

\usepackage{etoolbox}
\newcommand{\arxiv}[1]{\iftoggle{colt}{}{#1}}
\newcommand{\colt}[1]{\iftoggle{colt}{#1}{}}
\newtoggle{colt}
\global\togglefalse{colt}

\usepackage[utf8]{inputenc} %
\usepackage[T1]{fontenc}    %
\usepackage{lmodern}

\usepackage{parskip}
\usepackage{xargs}
\usepackage{tabularx} %
\usepackage{booktabs}
\usepackage{cellspace}
\usepackage{amsmath}
\usepackage{amsxtra,amsfonts,amscd,amssymb,bm,bbm,mathrsfs}
\usepackage{xspace}

\usepackage{mathtools}

\usepackage{enumitem}
\usepackage{multirow}
\usepackage{multicol}
\usepackage{booktabs}
\usepackage[toc,page,header]{appendix}
\usepackage{minitoc}
\usepackage{amssymb,amsmath,amsfonts}
\usepackage{mathtools}
\usepackage{graphicx}
\usepackage{epstopdf}
\usepackage{bm}
\usepackage{bbm}
\usepackage{color}
\usepackage{hyperref}
\hypersetup{
  colorlinks,
  breaklinks=true,
  citecolor=blue!70!black,
  linkcolor=blue!70!black,
  urlcolor=blue!70!black
}

\usepackage{multirow}
\usepackage{multicol}
\usepackage{latexsym}
\usepackage{amsmath}
\usepackage{amssymb}
\usepackage{xcolor}
\usepackage{graphicx}
\usepackage{bm}
\usepackage{dsfont}
\usepackage{enumitem}
\usepackage{natbib}

\usepackage{amsthm,thmtools}

\newtheorem{theorem}{Theorem}[section]
\newtheorem{lemma}[theorem]{Lemma}
\newtheorem{corollary}[theorem]{Corollary}
\newtheorem{proposition}[theorem]{Proposition}

\newtheorem{assumption}{Assumption}

\newtheorem{definition}{Definition}
\newtheorem{example}{Example}
\newtheorem{remark}[theorem]{Remark}

\renewenvironment{proof}[1][Proof]{
\paragraph{#1} 
}{
\qed
}

\usepackage[nameinlink,capitalize]{cleveref}

\newcommand{\pref}[1]{\cref{#1}}
\newcommand{\pfref}[1]{Proof of \pref{#1}}

\renewcommand{\eqref}[1]{\texorpdfstring{\hyperref[#1]{Eq. (\ref*{#1})}}{Eq. (\ref*{#1})}}
\crefformat{equation}{#2(#1)#3}
\Crefformat{equation}{#2(#1)#3}

\Crefformat{figure}{#2Figure #1#3}
\Crefname{assumption}{Assumption}{Assumptions}
\Crefformat{assumption}{#2Assumption #1#3}

\usepackage{crossreftools}
\usepackage[algo2e,linesnumbered,vlined,ruled]{algorithm2e}
\usepackage{algorithm}
\usepackage[noend]{algorithmic}

\usepackage{url}
\usepackage{breakcites}
\usepackage{cancel}
\usepackage{enumitem}
\usepackage{comment}
\crefformat{equation}{#2(#1)#3}
\Crefformat{equation}{#2(#1)#3}

\usepackage{booktabs}       %
\usepackage{multirow}
\usepackage{multicol}
\usepackage{makecell}
\usepackage{array}
\newcolumntype{H}{>{\setbox0=\hbox\bgroup}c<{\egroup}@{}}

\usepackage[normalem]{ulem}
\usepackage{exscale}
\usepackage{tikz}
\usepackage{float}
\usepackage{graphicx,graphics}

\allowdisplaybreaks

\let\oldparagraph=\paragraph
\renewcommand\paragraph[1]{\oldparagraph{#1.}}

\usepackage{crossreftools}
\usepackage{setspace}

\DeclareOldFontCommand{\rm}{\normalfont\rmfamily}{\mathrm}

\newcommand{\sups}[1]{^{{\scriptscriptstyle#1}}}

\newcommand{\paren}[1]{{\left( #1 \right)}}
\newcommand{\brac}[1]{{\left[ #1 \right]}}

\newcommand{\alg}{\Alg}

\newcommand{\R}{\mathbb{R}} %

\newcommand{\Ball}{\mathbb{B}^d}

\newcommand{\hth}{\widehat{\theta}}

\newcommand{\lsim}{{\;\raise0.3ex\hbox{$<$\kern-0.75em\raise-1.1ex\hbox{$\sim$}}\;}}
\newcommand{\gsim}{{\;\raise0.3ex\hbox{$>$\kern-0.75em\raise-1.1ex\hbox{$\sim$}}\;}}

\newcommand{\eps}{\varepsilon} 

\newcommand{\RNum}[1]{\uppercase\expandafter{\romannumeral #1\relax}}
\newcommand{\reg}{\mathsf{Reg}}
\newcommand{\Reg}{\mathbf{Reg}}

\DeclareMathOperator*{\argmin}{arg\,min}
\DeclareMathOperator*{\argmax}{arg\,max}
\newcommand{\KL}{D_{\mathrm{KL}}}

\newcommand{\cA}{\mathcal{A}}

\newcommand{\cD}{\mathcal{D}}
\newcommand{\cE}{\mathcal{E}}
\newcommand{\cF}{\mathcal{F}}

\newcommand{\cH}{\mathcal{H}}

\newcommand{\cL}{\mathcal{L}}
\newcommand{\cM}{\mathcal{M}}

\newcommand{\cO}{\mathcal{O}}

\newcommand{\cW}{\mathcal{W}}
\newcommand{\cX}{\mathcal{X}}

\newcommand{\cZ}{\mathcal{Z}}

\newcommand{\En}{\mathbb{E}}

\DeclareFontFamily{U}{mathx}{\hyphenchar\font45}
\DeclareFontShape{U}{mathx}{m}{n}{<-> mathx10}{}
\DeclareSymbolFont{mathx}{U}{mathx}{m}{n}
\DeclareMathAccent{\widebar}{0}{mathx}{"73}

\newcommand{\wh}[1]{\widehat{#1}}

\newcommand{\ldef}{\vcentcolon=}

\newcommand{\Unif}{\mathrm{Unif}}

\newcommand{\poly}{\mathrm{poly}}

\newcommand{\supp}{\mathrm{supp}}

\DeclarePairedDelimiter{\brk}{[}{]}
\DeclarePairedDelimiter{\crl}{\{}{\}}
\DeclarePairedDelimiter{\prn}{(}{)}

\arxiv{
\DeclarePairedDelimiter{\set}{\{}{\}}
}
\colt{
\renewcommand{\set}[1]{\{#1\}}
}

\DeclarePairedDelimiter{\floor}{\lfloor}{\rfloor}

\def\medskip{\vskip 10 pt}
\def\bigskip{\vskip 15 pt}

\def\texitem#1{\par\vspace{5pt}
\noindent\hangindent 20pt
\hbox to 20pt {\hss #1 ~}\ignorespaces}

\newcommandx{\Whp}[1][1=\delta]{With probability at least $1-#1$}
\newcommandx{\whp}[1][1=\delta]{with probability at least $1-#1$}

\newcommand{\DPi}{\Delta(\Pi)}

\newcommand{\Rdd}{\R^{d\times d}}
\newcommand{\PSD}{\mathbb{S}^{d}_+}
\newcommand{\PD}{\mathbb{S}^{d}_{++}}

\newcommand{\Mstar}{M^\star}

\usepackage{pifont}
\newcommand{\no}{\ding{55}}
\newcommand{\yes}{\ding{51}}
\newcommand{\id}{\mathbf{I}}

\newcommand{\RR}{\mathbb{R}}
\newcommand{\NN}{\mathbb{N}}

\newcommand{\EE}{\mathbb{E}}
\newcommand{\PP}{\mathbb{P}}

\newcounter{cnt}
\foreach \num in {1,2,...,26}{%
  \stepcounter{cnt}%
  \expandafter\xdef \csname c\Alph{cnt}\endcsname {\noexpand\mathcal{\Alph{cnt}}}%
  \expandafter\xdef \csname b\Alph{cnt}\endcsname {\noexpand\mathbb{\Alph{cnt}}}%
}

\newcommand{\tr}{\mathrm{tr}}

\newcommand{\Proj}{\operatorname{Proj}}
\newcommand{\diag}{\operatorname{diag}}

\newcommand{\leqsim}{\lesssim}
\newcommand{\geqsim}{\gtrsim}

\newcommand{\nrm}[1]{\left\|#1\right\|}

\newcommand{\abs}[1]{\left|#1\right|}

\DeclarePairedDelimiterX{\ddiv}[2]{(}{)}{%
  #1\;\delimsize\|\;#2%
}
\newcommand{\KLd}{\KL\ddiv}
\newcommand{\chis}{D_{\chi^2}\ddiv}

\newcommand{\indic}[1]{\mathbf{1}\left\{#1\right\}} %

\newcommand{\sign}{\operatorname{sign}}

\newcommand{\DTV}[1]{D_{\mathrm{TV}}\left(#1\right)}

\newcommand{\Alg}{\mathsf{Alg}}

\newcommand{\pis}{\pi^\star}

\newcommand{\defeq}{\mathrel{\mathop:}=}

\newcommand{\normal}[1]{\mathsf{N}(#1)}

\newcommand{\sset}[1]{\left\{#1\right\}}

\newcommandx{\VM}[1][1=M]{V\sups{#1}}

\newcommandx{\fm}[1][1=M]{f\sups{#1}}

\newcommand{\nrmop}[1]{\nrm{#1}_{\rm op}}

\newcommand{\DO}{\Delta(\cO)}

\newcommandx{\PM}[3][1=M,2=\pi,3=\pr]{#3\!\circ\!#1(#2)}

\newcommand{\pr}{\mathsf{Q}}

\newcommand{\aLDP}{$(\alpha,\beta)$-LDP}

\newcommand{\gfunc}{g}
\newcommandx{\gm}[1][1=M]{\gfunc\sups{#1}}

\newcommandx{\risk}[1][1=M]{\mathsf{Risk}^{#1}}

\newcommand{\bSigma}{\boldsymbol{\Sigma}}

\newcommand{\hf}{\hat{f}}

\newcommand{\Bone}{\mathbb{B}^d(1)}
\newcommand{\BR}{\mathbb{B}^d(R)}

\newcommand{\nrmF}[1]{\nrm{#1}_{\mathrm{F}}}

\newcommand{\otheta}{\Bar{\theta}}

\newcommand{\sq}{\sups{1/2}}
\newcommand{\iv}{\sups{-1}}

\newcommand{\iq}{\sups{-2}}
\newcommand{\isq}{\sups{-1/2}}
\newcommand{\tp}{\sups\top}
\newcommand{\sym}{\mathrm{sym}}
\newcommand{\lmax}{\lambda_{\max}}
\newcommand{\lmin}{\lambda_{\min}}
\newcommand{\ths}{\theta^\star}

\newcommand{\bz}{\mathbf{0}}

\newcommandx{\Mcxt}[1][1=M]{#1_{\sf cxt}}

\newcommand{\ea}{e^{\alpha}}

\newcommand{\lr}{\langle}
\newcommand{\rr}{\rangle}

\newcommand{\bigO}[1]{O\paren{#1}}
\newcommand{\tbO}[1]{\tilde{O}\paren{#1}}
\newcommand{\tbOn}[1]{\tilde{O}(#1)}

\newcommand{\Rad}[1]{\mathrm{Rad}\paren{#1}}

\newcommand{\fs}{f^\star}

\newcommandx{\Dl}[1][1=\lf]{\mathsf{D}_{#1}}
\newcommand{\lf}{\ell}

\newcommandx{\DC}[2][1=\Delta]{N_{\mathsf{frac}}(#2,#1)}
\newcommandx{\pds}[1][1=\Delta]{p_{#1}^\star}

\newcommandx{\NM}[2][1=\cM]{N(#1,#2)}

\newcommandx{\pim}[1][1=M]{\pi\sups{#1}}
\newcommandx{\pims}{\pim[\Mstar]}
\newcommandx{\Vm}[1][1=M]{V\sups{#1}}
\newcommandx{\Vmm}[1][1=M]{V\sups{#1}(\pi\sups{#1})}

\newcommandx{\LM}[2][1=M]{L(#1,#2)}

\newcommand{\ind}[1]{_{#1}}

\newcommandx{\pit}[1][1=t]{\pi_{#1}}
\newcommandx{\ppt}[1][1=t]{p_{#1}}
\newcommandx{\qt}[1][1=t]{q_{#1}}
\newcommandx{\prt}[1][1=t]{\pr_{#1}}
\newcommandx{\ot}[1][1=t]{o_{#1}}
\newcommandx{\zt}[1][1=t]{z_{#1}}
\newcommandx{\act}[1][1=t]{a_{#1}}
\newcommandx{\rt}[1][1=t]{r_{#1}}

\newcommandx{\Enmpi}[3][1=M,2=\pi]{\En\sups{#1,#2}\brac{#3}}
\newcommandx{\Emalg}[3][1=M,2=\alg]{\EE\sups{#1,#2}\brac{#3}}
\newcommandx{\Pmalg}[3][1=M,2=\alg]{\PP\sups{#1,#2}\paren{#3}}

\newcommandx{\bpr}[1][1=\lf]{\pr_{#1}}

\newcommandx{\Mpara}[1][1=M]{\theta(#1)}
\newcommandx{\RISK}[2][1=T,2=\alpha]{\mathfrak{M}_{#1,#2}}
\newcommandx{\SC}[2][1=\Delta,2=\alpha]{\mathfrak{C}_{#1,#2}}

\newcommandx{\Ncov}[3][1=\xspace,2=\Delta]{N_{#1}(#3,#2)}

\newcommandx{\SQ}[1][1=M]{\mathsf{STAT}_{#1}^{\tau}}
\newcommandx{\VSTAT}[1][1=M]{\mathsf{VSTAT}_{#1}^{\tau}}
\newcommandx{\GSQ}[1][1=M]{\mathsf{GQ}_{#1}^{\tau}}
\newcommandx{\phq}[2][1=\phi]{#2(#1)}
\newcommandx{\Dph}[2][1={\phi}]{\mathsf{D}_{#1}\paren{#2}}

\newcommand{\aJDP}{$(\alpha,\beta)$-JDP}
\newcommandx{\errlone}[2][1=\ths]{\Epp{|\lr \x, #2-#1\rr|}}
\newcommandx{\errloneg}[2][1=\ths]{\Epp{|\link(\lr \x, #2\rr)-\link(\lr\x,#1\rr)|}}
\newcommandx{\errltwo}[2][1=\ths]{\Epp{\lr \x, #2-#1\rr}^2}

\newcommand{\Lsq}{\cL_{\mathsf{Sq}}}
\newcommand{\Lgl}{\cL_{\mathsf{GLM}}}

\newcommand{\Ex}[1]{\EE_{\x\sim \pph}\brac{#1}}
\newcommand{\Exy}[1]{\EE_{(\x,y)\sim \Mstar}\brac{#1}}

\newcommand{\AlgCIEst}{\mathsf{Regression\_with\_Confidence}}

\newcommand{\Rd}{\R^d}

\newcommand{\cov}{\bSigma}
\newcommand{\covF}{F}

\newcommand{\epk}[1]{_{\scriptscriptstyle{ (#1) }}}
\newcommand{\kk}{\epk{k}}
\newcommand{\kp}{\epk{k+1}}

\newcommand{\kz}{\epk{0}}
\newcommand{\kc}{\epk{K}}

\newcommandx{\uxxu}[2][1=U,2=\xspace]{\frac{#1\x_{#2}\x_{#2}\tp #1}{\|#1 \x_{#2}\|}}
\newcommandx{\usqx}[2][1=U,2=\xspace]{\frac{#1\sq\x_{#2}\x_{#2}\tp #1\sq}{\|#1 \x_{#2}\|}}

\newcommand{\epsN}{\eps_N}
\newcommand{\rangekn}{kN+1,\cdots,(k+1)N}
\newcommand{\sumkn}{\sum_{t=kN+1}^{(k+1)N}}
\newcommand{\til}{\widetilde}

\newcommand{\Lone}{\ensuremath{L_1}}
\newcommand{\Ltwo}{\ensuremath{L_2}}
\newcommand{\x}{\phi}
\newcommand{\pph}{p}
\newcommand{\Epp}[1]{\EE_{\x\sim \pph}#1}
\newcommand{\Ep}[1]{\EE_{\x\sim \pph}\brac{#1}}
\newcommand{\Pp}[1]{\PP_{\x\sim \pph}\paren{#1}}
\newcommandx{\priv}[2][1={\Delta}]{\mathsf{Priv}_{#1}\paren{#2}}
\newcommandx{\sympriv}[2][1={\Delta}]{\mathsf{SymPriv}_{#1}\paren{#2}}

\newcommand{\nrmn}[1]{\big\|#1\big\|}
\newcommand{\nrmopn}[1]{\nrmn{#1}_{\mathrm{op}}}
\newcommand{\absn}[1]{\big|#1\big|}
\newcommand{\link}{\nu}
\newcommand{\mug}{\mu_\link}
\newcommand{\Lipg}{L_\link}
\newcommand{\kpg}{\kappa_\link}
\newcommand{\siga}{\sigma_{\alpha,\beta}}

\newcommandx{\hft}[1][1=\tau]{\widehat{f}^{(#1)}}
\newcommandx{\CI}{\widehat{b}}
\newcommandx{\CIt}[1][1=\tau]{\widehat{b}^{(#1)}}
\renewcommandx{\pit}[1][1=\tau]{\pi^{(#1)}}
\newcommand{\AlgPlan}{\textsf{Confidence-based-Planning}}
\newcommandx{\epj}[1][1=j]{^{(#1)}}
\newcommandx{\phxa}[1][1={x,a}]{\phi(#1)}

\newcommand{\as}{a^\star}
\newcommand{\lmins}{\lmin^\star}

\newcommand{\fhat}{\hat{f}}

\newcommand{\dA}{d_{\cA}}

\newcommandx{\cAxt}[1][1=\tau]{\cA\epj[#1](x)}
\newcommandx{\cAsp}[1][1=\tau]{\cA_{\mathsf{sp}}\epj[#1](x)}

\newcommand{\AlgJDPIGD}{\textsf{DP-SGD}}
\newcommand{\AlgLDPIGD}{\textsf{LDP-Improper-SGD}}

\newcommand{\AlgSQCB}{\mathsf{SquareCB}}

\newcommand{\ellg}{\ell_{\link}}

\newcommand{\FLDP}[1][\lambda]{\cF\sups{\mathsf{LDP}}_{#1}}
\newcommand{\FJDP}[1][{\gamma,\lambda}]{\cF\sups{\mathsf{DP}}_{#1}}

\newcommand{\Wapp}{\frac12\id\preceq \FJDP(W)\preceq 2\id}
\newcommand{\Uapp}{\frac12\id\preceq \FLDP(U)\preceq 2\id}

\newcommand{\gam}{\gamma}

\newcommand{\Wstar}[1][{\gamma,\lambda}]{\mathbf{W}_{#1}}
\newcommand{\wxxw}[1][W]{\frac{#1 \x\x\tp #1}{1+\gam\nrm{#1\x}}}
\newcommandx{\wsqx}[2][1=W,2={}]{\frac{#1\sq \x_{#2}\x_{#2}\tp #1\sq}{1+\gam\nrm{#1\x_{#2}}}}
\newcommand{\Ustar}[1][\lambda]{\mathbf{U}_{#1}}

\newcommand{\aDP}{$(\alpha,\beta)$-DP}

\newcommand{\sfLDP}{{\sf LDP}}
\newcommand{\sfDP}{{\sf DP}}

\newcommand{\cLLDP}{\cL\sups{\sfLDP}_{U,\lambda}}
\newcommand{\cLJDP}{\cL\sups{\sfDP}_{W,\gamma,\lambda}}
\newcommand{\cLLDPg}{\cL\sups{\textsf{LDP-GLM}}_{U,\lambda}}
\newcommand{\cLJDPg}{\cL\sups{\textsf{DP-GLM}}_{W,\gamma,\lambda}}
\newcommand{\moment}[1]{\mathsf{M}_{#1}}

\newcommand{\gap}{\Delta_{\min}}
\newcommand{\Prjm}{\mathbf{P}}
\newcommand{\rank}{\mathrm{rank}}

\newcommand{\dataset}{\cD}
\newcommand{\errpara}{parameter}

\newcommand{\infom}{information matrix}
\newcommand{\infoms}{information matrices}
\newcommand{\Infom}{Information matrix}
\newcommand{\Infoms}{Information matrices}
\newcommand{\iw}{information-weighted}
\newcommand{\Iw}{Information-weighted}
\newcommand{\IW}{Information-Weighted}
\newcolumntype{C}{>{\centering\arraybackslash}X}
\newcommand{\AlgLDPRegression}{\textsf{\IW-GLM-LDP}}
\newcommand{\AlgJDPRegression}{\textsf{\IW-GLM-DP}}
\newcommand{\AlgLDPGD}{\textsf{LDP-SGD}}
\newcommand{\AlgJDPGD}{\textsf{DP-ERM}}
\newcommand{\LDPLU}{\textsf{Spectral-Iteration-LDP}}
\newcommand{\JDPLU}{\textsf{Spectral-Iteration-DP}}

\newcommand{\METHOD}{\IW~Regression}
\newcommand{\Method}{\Iw~regression}
\newcommand{\method}{\iw~regression}
\newcommand{\LDPLinearRegression}{\textsf{\IW-Regression-LDP}}
\newcommand{\JDPLinearRegression}{\textsf{\IW-Regression-DP}}
\newcommand{\hthNW}{\hth_{\sf NW}}

\newcommand{\DPAlg}{\mathscr{A}\sups{\sfDP}_{T,\alpha,\beta}}
\newcommand{\LDPAlg}{\mathscr{A}\sups{\sfLDP}_{T,\alpha,\beta}}
\newcommand{\DPtag}[1][\hth]{#1\in\DPAlg}
\newcommand{\LDPtag}[1][\alg]{#1\in\LDPAlg}

\newcommand{\trd}[2]{\tr(#1\cdot #2^2)}
\newcommand{\PoP}{\mathsf{PoP}_{p,T,(\alpha,\beta)}}

\usepackage{color-edits}
 \addauthor{sr}{red}
 \addauthor{cf}{orange}
 \addauthor{fc}{orange}
 \addauthor{jl}{blue}

\arxiv{
\floatplacement{algorithm}{t}
}

\title{Near-Optimal Private Learning in Linear Contextual Bandits}
\title{The Statistical Complexity of Private Linear Regression: Optimality, Application to Linear Contextual Bandits}
\title{Optimal Private Linear Regression under General Distributions}
\title{The Statistical Limit of Private Linear Regression with Application to Linear Contextual Bandits}
\title{The Statistical Complexity of Private Linear Regression: It is not Covariance Matrix}
\title{Beyond Well-conditioned Covariance: The Statistical Complexity of Private Linear Regression}
\title{Beyond Covariance Matrix: The Statistical Complexity of Private Linear Regression}

\arxiv{
\author{Fan Chen\\{\small \texttt{fanchen@mit.edu}} \and  Jiachun Li\\{\small \texttt{jiach334@mit.edu}} \and Alexander Rakhlin\\{\small \texttt{rakhlin@mit.edu}} \and David Simchi-Levi\\{\small \texttt{dslevi@mit.edu}} }
}

\begin{document}

\maketitle

\begin{abstract}
    We study the statistical complexity of private linear regression under an unknown, potentially ill-conditioned covariate distribution. Somewhat surprisingly, under privacy constraints the intrinsic complexity is \emph{not} captured by the usual covariance matrix but rather its $L_1$ analogues. Building on this insight, we establish minimax convergence rates for both the central and local privacy models and introduce an Information Weighted Regression method that attains the optimal rates.

    As application, in private linear contextual bandits, we propose an efficient algorithm that achieves rate-optimal regret bounds of $\tbO{\sqrt{T}+\frac{1}{\alpha}}$ and $\tbO{\sqrt{T}/\alpha}$ under joint and local $\alpha$-privacy models, respectively. Notably, our results demonstrate that joint privacy comes at almost no additional cost, addressing the open problems posed by \citet{azize2024open}.
\end{abstract}

\section{Introduction}\label{sec:intro}

Linear regression, and in particular Ordinary Least Squares (OLS), is among the most widely used statistical methods. Motivated by privacy concerns in modern applications, a long line of research studies differentially private (DP) linear regression, including statistics perturbation for OLS~\citep{sheffet2017differentially,shariff2018differentially,sheffet2019old,wang2018revisiting,brown2024insufficient}, private ERM~\citep{dwork2010differential-statistics,chaudhuri2011differentially,kifer2012private,smith2017interaction,bassily2014private,bassily2019private,cai2021cost,varshney2022nearly}, and the Propose-Test-Release mechanism~\citep{dwork2009differential,liu2022differential}.  

For private linear regression under a \emph{general} (possibly ill-conditioned and heavy-tail) covariate distribution, the understanding remains limited because most existing analyses focus on \emph{well-conditioned} covariates. Specifically, most known convergence guarantees either scale with the minimum eigenvalue of the covariance matrix $\cov=\EE[\x\x\tp]$~\citep{sheffet2017differentially,shariff2018differentially,wang2018revisiting,cai2021cost,varshney2022nearly}, rely on tail bounds for $\cov\isq\x$ under $\x\sim p$, e.g., sub-Gaussian parameters~\citep{varshney2022nearly,brown2024insufficient} or higher-order moments~\citep{liu2022differential}, or incur \emph{slow} convergence rates~\citep{shariff2018differentially,zheng2020locally}. A recent line of work therefore specializes to Gaussian covariates~\citep{milionis2022differentially,liu2022differential,varshney2022nearly,liu2023near,asi2023robustness,brown2024insufficient,anderson2025sample}. To the best of our knowledge, the characterization of optimal convergence rate is only known for \emph{1-dimensional} locally private linear regression \citep{duchi2024right}.

Towards addressing this insufficiency, in this paper, we develop the minimax theory for private linear regression under general covariates. 
In the local model of differential privacy (LDP)~\citep{kasiviswanathan2011can,duchi2013local}, we show that the minimax squared error of LDP linear regression with $T$ samples is characterized as\footnote{Throughout this paper, unless otherwise specified, we use $\asymp_d$ to suppress polynomial factors of $d$, and use $\asymp$ to suppress constant factors.}
\begin{align}\label{eq:LDP-minimax-demo}
    \inf_{\LDPtag}~~\sup_{\ths\in\Bone} \EE\sups{(p, \ths),\alg} \nrmn{\hth-\ths}_A^2 ~~\asymp_d~~ \frac{1}{\alpha^2 T} \trd{A}{\Ustar[\lambda(T)]}, \qquad \lambda(T)=\frac{1}{\alpha\sqrt{T}},
\end{align}
where the infimum is taken over $\LDPAlg$, the class of all $T$-round locally $(\alpha,\beta)$-private algorithms (\cref{def:LDP}), $A$ is any PSD matrix, and $\Ustar$ is the unique solution to the following matrix equation:
\begin{align}\label{def:U-demo}
    \Ex{\uxxu \cdot \indic{\x\neq 0}}+\lambda U=\id, \qquad U\succ 0.
\end{align}

The matrix $\Ustar$ can be understood as the private analogue of the Fisher information matrix. Specifically, for a non-private linear model with Gaussian noise, the Fisher information matrix is the covariance matrix $\cov$, and classical statistical theory provides
\begin{align}\label{eq:Fisher-demo}
    \inf_{\hth}~~\sup_{\ths\in\Bone} \EE\sups{p, \ths} \nrmn{\hth-\ths}_A^2 ~~\asymp~~ \frac{1}{T} \tr(A\cov\iv), \qquad T\to \infty,
\end{align}
Therefore, $\Ustar$ characterizes the complexity of LDP linear regression in much the same way that the covariance matrix $\cov$ characterizes non-private linear regression. 

\begin{figure}
    \centering
    \usetikzlibrary{arrows.meta,calc}

\begin{tikzpicture}[>=stealth,baseline]

    \fill[blue!20] (-5,1.2) rectangle (-1,0);
    \node at (-3,0.6) {$\Wstar \;\asymp\; (\cov + \lambda^2 I)\isq$};
    
    \draw[<-,thick,blue] (-5,1.2) -- (-1,1.2);
    \draw[thick,blue] (-1,1.2) -- (-1,0);
    \draw[dashed,gray] (-1,0) -- (-1,-1.8);
    
    \fill[red!30] (1,1.2) rectangle (5,0);
    \node at (3,0.6) {$\Wstar \;\asymp\; \gamma \Ustar[\lambda\gamma]$};
    
    \draw[->,thick,red] (1,1.2) -- (5,1.2);
    \draw[thick,red] (1,1.2) -- (1,0);
    \draw[dashed,gray] (1,0) -- (1,-1.8);

    \draw[<->,thick] (-5,0) -- (5,0);
    \draw[<->,thick,gray] (-5,-1) -- (-1,-1);
    \draw[<->,thick,gray] (5,-1) -- (1,-1);
    
    \node at (-3,-0.5) {$\gamma \ll 1 \quad (T \gg 1/\alpha^2)$};
    \node at (3,-0.5) {$\gamma \geq 1 \quad (T \leq 1/\alpha^2)$};
    
    \node at (-3,-1.5) {Privacy for free};
    \node at (3,-1.5) {High privacy};
    
    \useasboundingbox (-5,-1.5) rectangle (5,1.2);
\end{tikzpicture}
    \caption{Illustration of the behavior of the \infom~$\Wstar$, where $\gamma\asymp \frac{1}{\alpha\sqrt{T}}$ and $\lambda\asymp \frac{1}{\sqrt{T}}$ as in \eqref{eq:JDP-minimax-demo}, under different scaling of $(\alpha,T)$. In the regime $T\gg \frac{1}{\alpha^2}$ ($\gamma\ll1$, left), $\Wstar$ scales as $(\cov+\lambda^2\id)\isq$~\cref{eq:W-to-cov} and hence the optimal DP estimators (e.g., \method~estimator) achieve the non-private optimal rate~\cref{eq:Fisher-demo}, i.e., privacy is ``for free''. On the other hand, in the ``high privacy'' regime $T\leq \frac{1}{\alpha^2}$ ($\gamma\geq 1$, right), the ``cost of privacy'' dominates the convergence rate as $\Wstar$ scales as $\gamma \Ustar[\lambda\gamma]$, and the minimax-optimal DP rate reduces to \eqref{eq:JDP-minimax-high-privacy}. Therefore, the definition \cref{def:W-demo} of $\Wstar$ can be interpreted as an interpolation between covariance matrix $\cov$ and the LDP \infom~$\Ustar$. 
    }
\end{figure}

Somewhat surprisingly, under a general covariate distribution $p$, the optimal convergence rate~\cref{eq:LDP-minimax-demo} \emph{cannot} be achieved by naively privatizing OLS, even in the 1-dimensional setting. Instead, we propose the \method~method, which (a) first privately computes a matrix $U$ such that \eqref{def:U-demo} holds approximately, and (b) then performs regression with respect to the \emph{information-weighted} squared loss: 
\begin{align}\label{def:cL-LDP-linear}
    \cL(\theta)\defeq \EE\brk*{ \frac{\prn*{\lr \x, \theta \rr -y}^2}{\nrm{U\x}} }+\lambda \nrm{\theta}^2_{U\iv}.
\end{align}
We note that for any $x\in\R^d$, the weight $\frac{1}{\nrm{U\x}}\asymp \frac1{\nrm{\Ustar\x}}$ can be interpreted as how much ``information'' of $\lr \x, \ths\rr$ is provided by a sample from the linear model (cf. \cref{eq:LDP-minimax-demo}), and hence the name ``information-weighted''.
We present the detailed description and discussion in \cref{sec:LDP-linear}.

In the (central) DP setting, we derive a similar characterization of the minimax-optimal risk under any squared loss:
\begin{align}\label{eq:JDP-minimax-demo}
    \inf_{\DPtag}~~\sup_{\ths\in\Bone} \EE\sups{p, \ths} \nrmn{\hth-\ths}_A^2 ~~\asymp~~ \frac{1}{T} \trd{A}{\Wstar}, \qquad\text{where}~~ \gamma\asymp_{d,\log} \frac{1}{\alpha\sqrt{T}}, ~~\lambda\asymp_d \frac{1}{\sqrt{T}},
\end{align}
where the infimum is taken over $\DPAlg$, the class of all \aDP~estimators (\cref{def:JDP}), and $\Wstar$ is the unique solution to the following matrix equation: 
\begin{align}\label{def:W-demo}
    \Ex{\wxxw}+\lambda W=\id, \qquad W\succ 0.
\end{align}
The matrix $\Wstar$ can be understood as an interpolation between the LDP matrix $\Ustar$ and the (non-private) covariance matrix $\cov$. Indeed, it is necessary that $\Wstar\to \cov\iv$ as $\lambda,\gamma\to 0$, as the optimal DP convergence rate recovers the non-private rate as $T\to\infty$. 
We highlight that~\eqref{eq:JDP-minimax-demo} provides a fine-grained characterization for \emph{any} finite sample size $T$. Hence, it implies a better understanding of the \emph{price of privacy} in DP linear regression, i.e., the smallest sample size under which there are DP procedures that recover the non-private convergence rate~\cref{eq:Fisher-demo}.
Compared to existing work that focuses on specific distributional assumptions (e.g., bounded sub-Gaussian parameters or minimum eigenvalues), our characterization provides a more unified picture of the ``price of privacy'' in DP linear regression.

As a main application, we deploy our \method~approach in generalized linear contextual bandits (\cref{sec:intro-cb}), obtaining an efficient algorithm with rate-optimal regret that addresses several open problems of \citet{azize2024open}.

\paragraph{Extension: Dimension-free linear regression}
Beyond the classical regime where $n\gg d$, we investigate linear regression when the dimension $d$ is prohibitively large or even unbounded~\citep{dwork2024differentially}, a setting connected to the \emph{benign-overfitting} phenomenon~\citep{bartlett2020benign,hastie2022surprises,cheng2024dimension}. We design private estimation procedures that achieve near-optimal dimension-independent convergence rates (\cref{sec:unbounded}). Somewhat surprisingly, any consistent private estimator in this regime must be \emph{improper}. 
As an application, we provide nearly minimax-optimal, dimension-independent regret bounds for private linear contextual bandits.

\subsection{Related work}

Efficient DP linear regression has been extensively studied, with methods ranging from perturbing sufficient statistics~\citep{sheffet2017differentially,shariff2018differentially,sheffet2019old,wang2018revisiting} and ``insufficient'' statistics~\citep{brown2024insufficient} for OLS, to private ERM~\citep{dwork2010differential-statistics,chaudhuri2011differentially,kifer2012private,smith2017interaction,bassily2014private,bassily2019private} (in particular, minimizing the squared loss via private gradient methods~\citep{cai2021cost,varshney2022nearly}), and the Propose-Test-Release framework~\citep{dwork2009differential,liu2022differential}. 
While private ERM methods attain minimax-optimal rates for general convex functions~\citep{bassily2014private,bassily2019private}, in the linear regression setting they either incur dependence on $\lmin(\cov)\iv$ or yield a $\frac{1}{\alpha T}$ rate, both of which are undesirable.

Linear models have also been extensively investigated in the closely related literature of \emph{robust} statistics~\citep{huber1965robust,huber1981robust,dwork2009differential}, with focus on the robustness of the estimators themselves~\citep{broderick2020automatic,kuschnig2021hidden} and of their performance guarantees~\citep{bhatia2015robust,bhatia2017consistent,balakrishnan2017computationally,klivans2018efficient,diakonikolas2019efficient,diakonikolas2019robust} under adversarially corrupted data. 
While DP estimators are naturally robust~\citep{georgiev2022privacy}, the converse holds only under stronger notions of robustness~\citep{hopkins2023robustness}. 
In robust statistics, weighted regression is a commonly adopted technique~\citep{dwork2009differential,he2022nearly,ye2023corruption}.

Parallel to the DP literature, locally private regression is typically analyzed under well-conditioned distributions in
generalized linear models~\citep{smith2017interaction,wang2019estimating} and 
sparse linear models~\citep{zheng2017collect,wang2019sparse,zhu2023improved}.
To the best of our knowledge, prior to this work, the characterization of LDP linear regression under a general distribution was only known in the 1-dimensional (simple) linear model~\citep{duchi2024right}.

\subsection{Application: Optimal regret for private linear contextual bandits}\label{sec:intro-cb}

Contextual bandits provide a natural framework for interactive decision making, with applications spanning numerous real-world domains, including recommendation system~\citep{li2010contextual,agarwal2016making} and healthcare~\citep{tewari2017ads,bastani2020online}. In this setting, a decision maker (or algorithm) sequentially observes contexts, selects actions, and receives rewards. 
Extensive research has established a relatively thorough understanding of the fundamental principles of balancing \emph{exploration} and \emph{exploitation} (that is, maximizing cumulative rewards) in contextual bandits~\citep[etc.]{abbasi2011improved, chu2011contextual, foster2018contextual, foster2020beyond, simchi2020bypassing}.

Formally, in contextual bandits, the learner observes a context $x_t\in\cX$ at each step $t\in[T]$, drawn stochastically as $x_t\sim P$. Based on the context $x_t$ and the history up to step $t$, the learner selects an action $a_t\in\cA$ and observes a reward $r_t\in[-1,1]$ with expected value $\EE[r_t|x_t,a_t]=\fs(x_t,a_t)$. Here, $\fs:\cX\times\cA\to[-1,1]$ is the underlying mean reward function. The learner's performance is measured by its \emph{regret}, defined as
\begin{align*}
    \Reg\defeq \EE\brac{ \sum_{t=1}^T \max_{\as_t\in\cA} \fs(x_t,\as_t) - \fs(x_t,a_t) },
\end{align*}
which measures the gap between the learner's cumulative rewards and that of an optimal policy with the full knowledge of $\fs$.
In generalized linear contextual bandits, a widely studied model, the ground-truth $\fs$ is assumed to take the form
\begin{align*}
    \fs(x,a)=\nu(\lr \phxa, \ths \rr), \qquad \forall (x,a)\in\cX\times\cA,
\end{align*}
where $\nu:[-1,1]\to[-1,1]$ is a known link function, $\phi:\cX\times\cA\to\Bone$ is a known feature map, and $\ths\in\Bone$ is the unknown underlying parameter. 
Despite the growing interest in private algorithms for decision making, the fundamental privacy-utility trade-offs remain poorly understood even for linear contextual bandits (where $\nu(t)=t$ is the identity function and $\fs$ is linear).

This gap stems from the fact that nearly all existing results for private linear regression assume well-conditioned covariate distributions. 
Under the \emph{joint differential privacy} (JDP) model~\citep{kearns2014mechanism,shariff2018differentially} with privacy parameter $\alpha$, the only known regret upper bound~\citet{shariff2018differentially} as $\sqrt{T/\alpha}$. This dependence is particularly undesirable in the high-privacy regime ($\alpha \ll 1$), prompting \citet{azize2024open} to highlight the rate-optimal regret in this setting as an important open question.

In the \emph{local differential privacy} (LDP) model~\citep{kasiviswanathan2011can,duchi2013local}, the best-known regret bound until recently scales as $\min\crl{\sqrt{T}/\lmins, T^{3/4}}$~\citep{zheng2020locally,han2021generalized}, where 
\begin{align}\label{def:lmin}
    \lmins\defeq \min_{\pi}\lmin(\EE^{\pi}\phxa\phxa\tp)
\end{align}
is the minimum eigenvalue of the covariance matrix over \emph{all} linear policies.\footnote{A policy $\pi$ is \emph{linear} if there exists $\theta\in\R^d$ so that $\pi(x)\in\argmax_{x\in\cX} \lr \theta, \phi(x,a)\rr$. For generalized linear contextual bandits, it is clear that the optimal policy must be linear.}
The assumption that $\lmins$ is lower bounded essentially requires that the learner can effectively estimate the ground-truth parameter $\ths$ while executing greedy policies, thereby removing the fundamental challenge of \emph{exploration}. 
This requirement, often termed ``explorability'' or ``diversity,'' posits that the context distribution $P$ is rich enough to be well explored by any greedy policy. Assuming $\lmins$ is lower bounded is very restrictive, as the quantity $1/\lmins$ can be prohibitively large in many practical scenarios. For example, when some direction $v\in\R^d$ is rarely observed, i.e., $\phxa \perp v$ for each $a\in\cA$ with high probability over $x\sim P$.

As we discuss in \cref{ssec:negative}, under worst-case covariate distributions any private regression method must exhibit squared-error guarantees that depend on the minimum eigenvalue of the covariance matrix or else suffer slow rates.
Therefore, for algorithms based on squared-loss regression—such as variants of the LinUCB algorithm~\citep{abbasi2011improved} and the SquareCB algorithm~\citep{foster2018contextual}—achieving optimal rates may inherently require dependence on $1/\lmins$ (cf. \cref{ssec:cb-negative}). 
The limitations of existing algorithmic principles suggest that privacy-utility trade-offs remain poorly understood even in linear contextual bandits, motivating the following question~\citep{azize2024open,li2024optimal}:
\begin{center}
\emph{What is the price of privacy in linear contextual bandits?}
\end{center}

Recent work has made progress toward answering this question under the local DP model. \citet{li2024optimal} propose an alternative approach based on \emph{regression with confidence intervals ($\Lone$-regression)}, achieving regret that scales as $\log^d(T)\sqrt{T}/\alpha$. 
Subsequently, \citet{chen2024private} prove a regret bound of $\sqrt{d^3T}/\alpha$ by analyzing the Decision-Estimation Coefficient (DEC)~\citep{foster2021statistical,foster2023tight}. While this rate is nearly minimax-optimal, it is achieved by the Exploration-by-Optimization algorithm~\citep{lattimore2020exploration,foster2022complexity}, which requires at least \emph{exponential} computational time. These results indicate that the statistical price of local $\alpha$-privacy is at most a multiplicative factor of $\poly(d,\log T)/\alpha$, independent of the \emph{explorability} $\lmins$. However, it remains open whether the optimal $\sqrt{T}$ regret rate can be attained without incurring exponential computational cost.

\paragraph{Our contributions}
Building on \method~method, we make significant progress toward settling the rate-optimal regret of private learning in (generalized) linear contextual bandits. We develop a \emph{computationally efficient} algorithm that achieves regret bounds of $\tbO{d^2\sqrt{T}+\frac{d^{5/2}}{\alpha}}$ and $\tbO{\sqrt{d^5T}/\alpha}$ under joint and local privacy models, respectively. 
Notably, our results demonstrate that joint privacy is almost ``for free'' in generalized linear contextual bandits with stochastic contexts, resolving the open problem of \citet{azize2024open} under this setting.
Our results significantly advance the understanding of privacy-utility trade-offs in contextual bandits. We present a more detailed comparison in \cref{sec:cb}.

\section{Preliminaries}\label{sec:prelim}
\newcommand{\epsapx}{\eps_{\sf apx}}
\newcommand{\Bd}{B}

\paragraph{Linear models}
We focus on linear models throughout the main text and defer extensions to generalized linear models to \cref{appdx:GLM}.
\begin{definition}[Linear model]\label{def:linear-model}
    A distribution $M$ over $\R^d\times \R$ is a linear model with ground truth $\ths$ if, for $(\x,y)\sim M$, we have $\nrm{\x}\leq B$ and $\abs{y}\leq 1$ almost surely. Throughout this paper, we assume the parameter $B\geq 1$ is known.
\end{definition}
Our sole assumption on the covariate distribution is that it is supported on the bounded domain $\Ball(B)$, which is arguably the minimal and simplest condition needed for analyzing linear regression.\footnote{A line of work has investigated DP linear regression with (sub-)Gaussian covariate distributions. Although the covariate vector $\phi$ can be unbounded under such distributions, it holds that $\nrm{\x}\leqsim \nrmF{\cov}\log(1/\delta)$ \whp; thus this setting can be reduced to ours by clipping covariates whose norms exceed a certain threshold. }

\paragraph{Privacy}
Differential privacy is a widely used framework for ensuring privacy in data analysis. We begin by recalling the notion of differentially private (DP) channels before stating the definitions of joint DP and local DP algorithms.

\begin{definition}[DP channel]
Given a latent observation space $\cZ$ and a noisy observation space $\cO$, a channel $\pr$ is a measurable map from $\cZ$ to $\DO$. %
A channel $\pr$ is \aDP~if, for any $z, z'\in\cZ$ and any measurable set $E\subseteq \cO$,
\begin{align*}
    \pr(E|z)\leq e^\alpha\pr(E|z')+\beta.
\end{align*} 
\end{definition}

In this paper, we focus on the regime $\alpha,\beta\in(0,1]$.
A standard example of an \aDP~channel is the Gaussian channel~\citep[see e.g.][]{balle2018improving}.
\begin{definition}[Gaussian channel]\label{def:Guassian-channel}
Suppose that $\alpha,\beta\in(0,1]$, and denote $\siga=\frac{4\sqrt{\log(2.5/\beta)}}{\alpha}$. 
For any given function $F:\cZ\to\R^d$, we let $\Delta(F)\defeq \frac12\sup_{z,z'} \nrm{F(z)-F(z')}_2$. Then, for any $\Delta\geq \Delta(F)$, the channel $\pr(\cdot|z)=\normal{ F(z), \siga^2 \Delta^2 }$ is a $(\alpha,\beta/2)$-DP channel.
\end{definition}

In the following, we denote $\priv[\Delta]{v}=\normal{v, \siga^2\Delta^2}$. %
Further, for any symmetric matrix $V\in\Rdd$, we denote $\sympriv{V}$ as the distribution of $V+Z$, where $Z$ is a symmetric Gaussian random matrix, i.e., $Z_{ij}=Z_{ji}\sim \normal{0,\siga^2\Delta^2}$ independently.

\paragraph{Privacy-preserving algorithms}
We first recall the definition of differential privacy~\citep{dwork2006calibrating}. In \emph{non-interactive} problems, an \emph{algorithm} maps a dataset $\cD=\set{z_1,\cdots,z_T}\in\cZ^T$ to a distribution over the output decision space $\Pi$. Two datasets $\cD=\set{z_1,\cdots,z_T}$ and $\cD'=\set{z_1',\cdots,z_T'}$ are \emph{neighboring} if they differ in at most one entry.

\begin{definition}[DP algorithms]\label{def:JDP}
An algorithm $\alg:\cZ^T\to\Delta(\Pi)$ preserves \aDP~if for any neighboring datasets $\cD, \cD'$ and any measurable set $E\subseteq \DPi$,
\begin{align*}
    \alg(E|\cD)\leq \ea \alg(E|\cD')+\beta.
\end{align*}
\end{definition}

A parallel line of work studies \emph{locally differentially private} (LDP) algorithms~\citep{kasiviswanathan2011can,duchi2013local}. 

\begin{definition}[LDP algorithms]\label{def:LDP}
A \emph{sequential mechanism} $\alg$ is specified by an anonymized observation space $\cO$, a channel $Q: \bigsqcup_{t=1}^T (\cO^{t-1}\times \cZ)\to \DO$, and an output map $q: \cO^T\to\Delta(\Pi)$. A sequential mechanism is locally $(\alpha,\beta)$-private (or simply \aLDP) if for any $t\in[T]$, any sequence of anonymized observations $o_1,\cdots,o_{t-1}$, and any two different data points $z_t, z_t'$ at step $t$, it holds that
\begin{align*}
Q(o_t\in E|o_1,\cdots,o_{t-1},z_t)\leq 
    \ea Q(o_t\in E|o_1,\cdots,o_{t-1},z_t')+\beta, \qquad \forall \text{ measurable set }E\subseteq \cO.
\end{align*}
\end{definition}
With slight abuse of notation, for any sequential mechanism $\alg$, we write $Q_t\sups{\alg}\ldef Q(o_t=\cdot\mid o_1,\cdots,o_{t-1},z_t=\cdot)$, which is a random channel (because $Q_t\sups{\alg}$ depends on the history of observations $o_1,\cdots,o_{t-1}$). Then, $\alg$ is \aLDP~if and only if, deterministically, each $Q_t\sups{\alg}$ is \aDP for all $t\in[T]$.

\newcommandx{\clip}[2][1=R]{\mathsf{clip}_{#1}(#2)}

\paragraph{Miscellaneous notation}
We write $\PSD$ for the set of all $d\times d$ positive semi-definite (PSD) matrices, and $\PD$ for the set of all $d\times d$ positive definite matrices. For any matrix $A$, we define $\sym(A)=(A\tp A)\sq$. We also define the clipping operation
\begin{align*}
    \clip{v}\defeq \max\sset{\min\sset{v,R},-R}\in[-R,R], \qquad \forall v\in\R.
\end{align*}

\newcommand{\covs}{\cov_\star}

\section{Motivating Examples}

\subsection{Case study: Simple linear regression}\label{ssec:OLS}

\newcommand{\SSP}{\mathsf{SSP}(\cD)}
\newcommand{\cDssp}{\cD_{\sf SSP}}

We begin with the following characterization of simple linear regression ($d=1$) under the local privacy model~\citep{duchi2024right}.

\begin{lemma}\label{lem:1d-LDP}
    \newcommand{\simpleerr}{\min\crl*{1,\frac{1}{\EE_{\x\sim p}\abs{\x}}\cdot \frac{1}{\alpha\sqrt{T}}}}
Consider 1-dimensional linear models with covariate distribution $p$ (supported on $[-1,1]$). Suppose that $\alpha\in(0,1)$, $\beta\in[0,\frac1T]$. Then it holds that
\begin{align*}
    C_0\cdot \simpleerr
    \leq \inf_{\LDPtag}~~\sup_{\ths\in[-1,1]}\EE\sups{(p,\ths),\alg}\absn{\hth-\ths}\leq C_1\cdot \simpleerr,
\end{align*}
where $C_0, C_1>0$ are absolute constants. 
\end{lemma}
Thus, even in one-dimensional linear models, the statistical complexity of estimating $\ths$ is \emph{not} governed by the covariance $\cov=\EE_{\x\sim p}\abs{\x}^2$ (which coincides with the Fisher information). Instead, the quantity $\EE_{\x\sim p}\abs{\x}$—the \Lone-covariance—serves as the relevant private analogue of Fisher information for LDP linear regression, as highlighted by \citet{duchi2024right}. To make the discussion concrete, we state in \cref{alg:LDP-simple-linear-regression} a simple algorithm that attains the upper bound in \cref{lem:1d-LDP}.

\newcommand{\Lap}[2]{\mathsf{Laplace}\prn*{#1,#2}}
\begin{algorithm}
    \caption{Optimal Locally Private Simple Linear Regression
    }\label{alg:LDP-simple-linear-regression}
    \begin{algorithmic}[1]
    \REQUIRE Dataset $\cD=\crl*{(\x_t,y_t)}_{t\in[T]}$, privacy parameter $\alpha>0$.
    \FOR{$t=1,\cdots,T$}
    \STATE Privatize
    \begin{align*}
        \til \psi_t\sim \Lap{\sign(\x_t)y_t}{\frac{\alpha}{2}}, \qquad \til \Psi_t \sim \Lap{\abs{\x_t}}{\frac\alpha2}.
    \end{align*}
    \ENDFOR
    \STATE Compute
    \begin{align*}
        \wh \psi=\frac{1}{T}\sum_{t=1}^{T} \til \psi_t, \qquad 
        \wh \Psi=\frac{1}{T}\sum_{t=1}^{T} \til \Psi_t.
    \end{align*}
    \ENSURE Estimator $\hth=\clip[1]{\prn{\wh \Psi}\iv\wh\psi}$.
    \end{algorithmic}
\end{algorithm}

\paragraph{The failure mode of locally private mechanisms}
To enforce local privacy for OLS, a line of work on (generalized) linear models~\citep{smith2017interaction,wang2019estimating} and sparse linear regression~\citep{zheng2017collect,wang2019sparse,zhu2023improved} perturbs the sufficient statistics by injecting noise into $(\x\x\tp,\x y)$ for each sample $(\x,y)$.
However, any locally private algorithm that perturbs OLS in this manner \emph{cannot} attain the minimax-optimal rate in \cref{lem:1d-LDP}. More precisely, we establish the following lower bound for any sequential mechanism (cf. \cref{def:LDP}) that is \emph{smooth} with respect to the covariate $\x$. The proof appears in \cref{appdx:proof-OLS-fail-local}.

\begin{proposition}[Sub-optimality of smooth algorithms]\label{prop:OLS-fail-local-gen}
    Suppose that $d=1$, the covariate distribution $p$ is supported on $[-1,1]$, and $\EE_{\x\sim p}[\x]=0$.  
    For any channel $\pr:[-1,1]^2\to\DO$, we say $\pr$ is $\varrho$-smooth if any $\x, y\in[-1,1]$, it holds that
    \begin{align}
        &~ \chis{\pr(\cdot\mid \x,y)}{\pr(\cdot\mid \x',y)}\leq \varrho^2 (\abs{\x}+\abs{\x'})^2. \label{eq:def-smooth-1} 
    \end{align}
    Suppose that $\alg$ is a sequential mechanism (cf. \cref{def:LDP}) such that $Q_t\sups{\alg}$ is $\varrho$-smooth for all $t\in[T]$. Then it holds that
    \begin{align*}
        \sup_{\ths\in[-1,1]}\EE\sups{(p,\ths),\alg}\absn{\hth-\ths}\geq \frac{1}{16}\min\crl*{\frac{1}{\EE_{\x\sim p}\abs{\x}^2}\cdot \frac{1}{\varrho \sqrt{T}}, 1}. 
    \end{align*}
\end{proposition}

Informally, a channel $\pr$ is $\varrho$-smooth (for some $\varrho<+\infty$) when it behaves smoothly at $\x=0$ as a function of $\x$.
For instance, the following channels are $\varrho$-smooth with $\varrho=O(\alpha)$:\footnote{Note that $\pr$ remains $O(\alpha)$-smooth if we replace $\normal{\cdot,\siga^2}$ with $\Lap{\cdot}{\frac1\alpha}$. }
\begin{align*}
    o=(o_1,o_2,o_3,y)\sim \pr(\cdot\mid \x,y):\qquad o_1\sim \normal{\x, \siga^2},~~~o_2\sim \normal{\x^2, \siga^2},~~~o_3\sim \normal{\x y, \siga^2}.
\end{align*}
This channel underlies many existing locally private linear regression methods, including sparse linear regression~\citep{zheng2017collect,wang2019sparse,zhu2023improved} and generalized linear models~\citep{smith2017interaction,wang2019estimating}. It is also the natural mechanism for privatizing OLS, which relies on $\EE[\x^2]\cdot \theta=\EE[\x y]$. Yet, as \cref{lem:1d-LDP} shows, such algorithms are sub-optimal whenever $\EE_{\x\sim p}\abs{\x}^2\ll \EE_{\x\sim p}\abs{\x}$, i.e., in high signal-to-noise regimes. 

Conversely, the simple LDP procedure in \cref{alg:LDP-simple-linear-regression} matches the upper bound in \cref{lem:1d-LDP} by effectively solving $\EE[\abs{\x}]\cdot \theta=\EE[\sign(\x) y]$. The key difference is that \cref{alg:LDP-simple-linear-regression} is \emph{not} smooth because it privatizes $\sign(\x_t)y_t$. In \cref{sec:LDP-linear}, we develop \emph{\method} for LDP linear regression with any dimension $d\geq 1$, which is a significantly more sophisticated generalization of this idea.

\paragraph{Statistical perturbation under the central model}
As in the local model, we show that a broad class of algorithms that jointly privatize $(\x_t,y_t)$ in a smooth fashion are sub-optimal. The proof appears in \cref{appdx:proof-OLS-fail}.
\begin{lemma}[Sub-optimality of SSP; Central model]\label{lem:OLS-fail}
    Suppose that $d=1$ and the covariate distribution $p$ is supported on $[-1,1]$. 

    Suppose that $\Alg:\cZ^T\to \Delta([-1,1])$ is an algorithm such that for any $t\in[T]$, any two dataset $\cD=\crl*{(\x_t,y_t')}_{t\in[T]}$ and $\cD'=\crl*{(\x_t,y_t')}_{t\in[T]}$ with $y_i=y_i'$ for all $i\in[T]\backslash \crl{t}$, it holds that
    \begin{align}\label{eq:def:central-smooth}
        \DTV{\Alg(\cdot\mid \cD), \Alg(\cdot\mid \cD')}\leq \varrho\abs{\x_t}.
    \end{align}
    Then it holds that
    \begin{align*}
        \sup_{\ths\in[0,1]} \EE\sups{(p,\ths),\alg}\absn{\hth-\ths} \geq \frac{1}{16}\min\crl*{\frac{1}{\EE_{\x\sim p}\abs{\x}^2} \cdot \frac{1}{\varrho T}, 1}.
    \end{align*}
\end{lemma}
To understand why \eqref{eq:def:central-smooth} captures most DP linear regression methods, note that OLS is usually privatized through ``sufficient statistic perturbation''~\citep{vu2009differential,dwork2010differential-statistics,wang2018revisiting,shariff2018differentially}. Given a dataset $\cD=\crl*{(\x_1,y_1),\cdots,(\x_T,y_T)}$, the perturbed statistic is
\begin{align}\label{def:SSP}
    \SSP\ldef \sum_{t=1}^{n} \x_t y_t + \normal{0, \tau^2\id},
\end{align}
where $\tau\geqsim \frac{1}{\alpha}$ ensures that the statistic $\sum_{t=1}^{n} \x_t y_t$ is perturbed sufficiently and privacy is preserved. Then, \eqref{eq:def:central-smooth} holds with $\varrho=O(\alpha)$ as long as the final output is a function of $\cDssp=\crl*{\x_1,\cdots,\x_T,\SSP}$.

Consequently, any algorithm that satisfies \eqref{eq:def:central-smooth} with $\varrho\asymp \alpha$ needs $T\geqsim \frac{1}{\sqrt{\EE_{\x\sim p}\abs{\x}^2}}\cdot \frac{1}{\alpha\eps}$ samples to guarantee $\EE\sups{p,\ths}\nrmn{\hth-\ths}_{\cov}^2\leq \eps^2$. This requirement is clearly suboptimal for well-conditioned distributions. For comparison, for \emph{sub-Gaussian} covariates there exist estimators achieving the same error with sample complexity $\bigO{\frac{1}{\eps^2}+\frac{1}{\alpha\eps}}$~\citep{varshney2022nearly,brown2024insufficient}.

\subsection{Lower bounds under ``heavy-tail'' distributions}\label{ssec:negative}

In the following, we provide a more straightforward illustration of why the covariance matrix fails to capture the complexity of private linear regression. Specifically, we show that under covariate distribution $p$ that is ``heavy-tail'', the optimal convergence rates of private linear regression \emph{have to} scale with the minimum eigenvalue of the covariance matrix $\cov$. Therefore, under such distributions, there is a significant degradation of convergence rates compared to the rates under Gaussian distributions. This separation implies that, to quantify the difficulties of private linear regression under a fixed covariate distribution $p$, a more fine-grained complexity is necessary (c.f. \eqref{def:U-demo} and \eqref{def:W-demo}).

The ``heavy-tail'' distributions of consideration are specified as follows.

\begin{example}[Simple distributions]\label{example:ill-condition-p}
For any PSD matrix $\covs\succ 0$ and parameter $B>0$ such that $\tr(\cov)\leq B^2$, we construct a ``ill-conditioned'' distribution $p=p_{\covs,B}$ with covariance matrix $\covs$.

Consider the eigen-decomposition $\covs=\sum_{j=1}^{d} \rho_j e_je_j\tp$, where $e_1,\cdots,e_d$ are orthonormal vectors, and $\rho_1\geq\cdots\geq\rho_d\geq 0$ are the eigenvalues of $\covs$. We then construct $p=p_{\covs,B}$ as
\begin{align*}
    p(Be_j)=\frac{\rho_j}{B^2}~~\forall j\in[r], \qquad p(0)=1-\frac{1}{B^2}\sum_{j=1}^r \rho_j.
\end{align*}
Under distribution $p$, any ``legal'' ground-truth parameter $\theta$ must belong to the following set:
\begin{align*}
    \Theta_p\defeq \crl*{\theta\in\Bone: \abs{\lr \x,\theta\rr}\leq 1,~\forall \x\in\supp(p)}.
\end{align*}
For any $\theta\in\Theta_p$, we define $M_{\theta,p}$ to be the linear model specified by $\theta$ with covariate distribution $p$: 
\begin{align*}
    (\x,y)\sim M_{p,\theta}:\qquad \x\sim p,~~y\mid \x\sim \Rad{\lr \x, \theta\rr}.
\end{align*}
\end{example}

\newcommand{\yDP}{$(\alpha,\beta)$-yDP}

In the following, we present a folklore lower bound that under the simple distributions, the squared error of any locally private algorithm must scale with the minimum eigenvalue $\rho_d=\lmin(\covs)$. 

\begin{proposition}[Lower bounds under simple distributions]\label{prop:lower-linear-cov}
Let $\alpha\in(0,1]$, $\beta\in[0,\frac{1}{8T}]$, $B\geq 1$, and $\covs\succ 0$ be given. Suppose that $p=p_{\covs,B}$ is the distribution constructed in \cref{example:ill-condition-p}.
Then the following lower bound holds for any $T$-round \aLDP~algorithm $\alg$ with output $\hth$:
\begin{align*}
    \sup_{\ths\in\Theta_p}\EE\sups{(p,\ths),\alg} \nrmn{\hth-\ths}_{\covs}^2\geq c\cdot \max_{i\in[d]}\min\crl*{\rho_i, ~\frac{B^2}{\alpha^2 T}\cdot \frac{1}{\rho_i}},
\end{align*}
where $c>0$ is an absolute constant. Particularly, further assuming that $T\geq \frac{B^2}{\rho_1\alpha^2}$ and $\rho_{i+1}\geq \frac12\rho_i$ for all $i=1,2,\cdots,d-1$, the above lower bound can be simplified to 
\begin{align*}
    \sup_{\ths\in\Theta_p}\EE\sups{(p,\ths),\alg} \nrmn{\hth-\ths}_{\covs}^2\geq c_1\cdot \min\crl*{\frac{B}{\alpha\sqrt{T}}, ~\frac{B^2}{\lmin(\covs)\alpha^2 T}}.
\end{align*}
\end{proposition}

This result shows that under a ``heavy-tail'' distribution $p$, any LDP algorithm must either incur squared error $\Omega\prn*{\frac{1}{\lmin(\cov)\alpha^2 T}}$ or accept a slower rate. In other words, dependence on $\lmin(\cov)$ is asymptotically unavoidable for such distributions. Moreover, for any fixed sample size $T$, \cref{prop:lower-linear-cov} yields a lower bound of $\Omega\prn*{\frac{1}{\alpha\sqrt{T}}}$ when the covariance matrix is $\covs=\rho\id$ with $\rho=\frac{1}{\alpha\sqrt{T}}$. Thus, under the \emph{worst-case} covariate distribution, the optimal convergence rate of any LDP algorithm deteriorates sharply relative to the non-private rate.

\paragraph{Lower bounds for DP linear regression}
To illustrate that degradation is also inherent for DP linear regression under simple distributions, we prove a lower bound for the broader class of algorithms that only privatize the \emph{labels}.
\begin{definition}[yDP algorithms]\label{def:yDP}
Two datasets $\cD=\set{(\x_1,y_1),\cdots,(\x_T,y_T)}$, $\cD'=\set{(\x_1',y_1'),\cdots,(\x_T',y_T')}$ are \emph{y-neighbored} if $\x_i=\x_i'$ for all $i\in[n]$, and there is at most one index $i$ such that $y_i\neq y_i'$.

An algorithm $\alg:(\R^d\times\R)^T\to\Delta(\Pi)$ preserves \yDP~if for any y-neighbored datasets $\cD, \cD'$ and any measurable set $E\subseteq \DPi$,
\begin{align*}
    \alg(E|\cD)\leq \ea \alg(E|\cD')+\beta.
\end{align*}
\end{definition}

\begin{proposition}[DP lower bounds under simple distributions]\label{prop:lower-linear-cov-JDP}
    Let $\alpha\in(0,1]$, $\beta\in[0,\alpha]$, $B\in[1,\sqrt{d}]$, and $\covs\succ 0$ be given. Suppose that $p=p_{\covs,B}$ is the distribution constructed in \cref{example:ill-condition-p}.
    Then the following lower bound holds for any $T$-round \yDP~estimator $\hth$:
    \begin{align*}
        \sup_{\ths\in\Theta_p}\EE\sups{(p,\ths)} \nrmn{\hth-\ths}_{\covs}^2 
            \geq&~ \frac{cB^2}{\alpha^2 T^2}\sum_{i=1}^{d^\star} \frac{1}{\rho_i} ,
    \end{align*}
    where $d^\star=\abs{\crl*{i\in[d]: \rho_i\geq \frac{\sqrt{d}B}{\alpha T}}}$.
\end{proposition}
Proof in \cref{appdx:proof-lower-cov}.
In words, under the simple distribution $p=p_{\covs,B}$, DP linear regression algorithms must incur a ``privacy-dependent'' term that depends on the eigenvalues of $\covs$. Importantly, in the regime $\kappa(\covs)=O(1)$ and $\alpha^{-1}B\gg \sqrt{d}$, this lower-order term is dominated by $\frac{d}{T}$ only when $T\geqsim \frac{B^2}{\lmin(\covs)\alpha^2}$. Therefore, under these worst-case distributions, achieving a vanishing ``price of privacy'' requires a sample complexity that scales with $\lmin(\cov)\iv$ (cf. \cref{ssec:PoP}).

Now consider $B=\sqrt{d}$ and $\covs=\rho\id$. For a ``worst-case'' choice of $\rho$, \cref{prop:lower-linear-cov} implies that there exists a distribution $p$ 
(agnostic to the sample size $T$), any DP linear regression algorithm must incur a squared error of $\Omega\prn*{\frac{d}{\alpha T}}$. This is a significant degradation compared to the rate of non-private linear regression.

\paragraph{Locally perturbed distributions}
We also provide the following lower bound, which states that for any fixed covariate distribution $p$ with covariance matrix $\cov$, there is a ``perturbed'' version of $p$ such that the squared error of any private estimator has to scale with $\lmin(\cov)\iv$.
\newcommand{\emin}{e}
\begin{proposition}\label{prop:lower-perturbed-p}
Let $T\geq 1$, $\alpha, \beta\in(0,1]$. Suppose that $p$ is supported on $\Bone$, $\cov=\EE_{\x\sim p}[\x\x\tp]$ is the covariance matrix of $p$. For any parameter $\rho\in[\lmin(\cov),\lmax(\cov)]$, we pick a unit vector $\emin$ such that $\emin\tp\cov \emin=\rho^2$, and consider the distribution $p'\ldef (1-\rho)p+\rho\delta_{\emin}$, i.e., 
\begin{align*}
    \x\sim p': \qquad\text{with probability }1-\rho, \x\sim p; \quad\text{with probability }\rho, \x=\emin.
\end{align*}
Let $\cov'$ be the covariance matrix of $p'$. Then, the following lower bounds hold.

(a)  For any $T$-round \aLDP~algorithm $\alg$ with output $\hth$, as long as $\beta\leq \frac{1}{T}$, it holds that
\begin{align*}
    \sup_{\ths}\EE\sups{(p',\ths),\alg} \nrmn{\hth-\ths}_{\cov'}^2\geq c\min\crl*{\frac{1}{\rho}\cdot \frac{1}{\alpha^2 T}, \rho}.
\end{align*}

(b) For any $T$-round \yDP~estimator $\hth$, as long as $\beta\leq \alpha$, it holds that
\begin{align*}
    \sup_{\ths\in\Bone}\EE\sups{p',\ths} \nrmn{\hth-\ths}_{\cov'}^2
    \geq&~ c\min\crl*{\frac{1}{\rho}\cdot \frac{1}{\alpha^2 T^2}, \rho}.
\end{align*}
\end{proposition}

\section{\IW~Linear Regression with Local Privacy}\label{sec:LDP-linear}

In this section, we develop our theory for locally private linear regression. Central to the analysis is the operator $\FLDP:\PD\to\PD$, defined as
\footnote{For notational simplicity, in the following we regard $\uxxu=0$ for $\x=0$ and hence omit the indicator $\indic{\x\neq 0}$. }
\begin{align}\label{def:FLDP}
\FLDP(U)\defeq \EE_{\x\sim p}\brac{ \frac{U\x\x\tp U}{\nrm{U\x} }\indic{\x\neq 0} }+\lambda U.
\end{align}
We summarize the basic properties of $\FLDP$ and the matrix $\Ustar$ in the following lemma. Recall that $\Ustar$ is defined as the (unique) solution to $\FLDP(U)=\id$.
\begin{lemma}\label{lem:U-L1-covariance}
For any $\lambda>0$, the solution $\Ustar$ to $\FLDP(U)=\id$ exists and is unique, and $\lambda\mapsto \Ustar$ is continuous. For any $U\in\PSD$ such that $\Uapp$, it holds that $\frac14 U\preceq \Ustar^2\preceq 4U^2$.

Furthermore, it holds that
\begin{align}\label{eq:Ustar-Lone-covariance}
    \nrm{\Ustar\iv \theta}\leq \Ex{\abs{\lr \x, \theta \rr}}+\lambda \nrm{\theta}\leq (\sqrt{d}+1)\nrm{\Ustar\iv \theta}.
\end{align}
\end{lemma}
Consequently, $\Ustar\iq$ serves as an ``\Lone-covariance matrix''\footnote{There is a $(\sqrt{d}+1)$ gap between the lower and upper bounds in \cref{lem:U-L1-covariance}, which is generally unavoidable when converting between the \Lone- and \Ltwo-norms.}, analogous to the standard covariance matrix $\cov$, which satisfies
\begin{align*}
    \lambda^2\nrm{\theta}^2+\Ex{\lr \x,\theta\rr^2}=\nrm{\theta}^2_{\cov+\lambda^2\id}, \qquad \forall \theta\in\R^d.
\end{align*}

\subsection{Locally private \method}\label{ssec:linear-LDP-upper}

We next introduce the \emph{\method} method, a two-step procedure:
\begin{itemize}
\item Step 1: Privately learn a matrix $U$ such that $\frac12\id\preceq \FLDP(U)\preceq 2\id$, where $\lambda\asymp \frac{\sqrt{d}}{\alpha \sqrt{T}}$.
\item Step 2: Optimize the \emph{\iw} squared-loss objective:
\begin{align}\label{def:loss-LDP}
    \cLLDP(\theta)\defeq \EE\brk*{ \frac{\prn*{\lr \x, \theta \rr -y}^2}{\nrm{U\x}} }+ \lambda \nrm{\theta}^2_{U\iv}.
\end{align}
\end{itemize}
At a high level, Step 1 uses an iterative spectral algorithm to approximately solve $\FLDP(U)=\id$ from privatized samples (\cref{alg:U-LDP}). Step 2 can then be carried out privately because $\frac12 U\iq \preceq \nabla^2 \cLLDP(\theta)\preceq 2 U\iq$, i.e., $\cLLDP$ is a well-conditioned quadratic objective under the change of variables $w=U\theta$. 
We now detail each step.

\paragraph{Step 1: Learning $\Ustar$ privately}
What makes this step challenging is that $\FLDP(U)=\id$ is a \emph{non-linear} equation, and it is even unclear how to solve it given full knowledge of the covariate distribution $p$. 
Our key observation is that the following spectral iterates converge to $\Ustar$:
\begin{align}\label{eqn:U-spectral-exact}
    \covF\kk=\Ep{\uxxu[U\kk]}+\lambda U\kk, \qquad
    U\kp=\sym(\covF\kk\isq U\kk),
\end{align}
starting from $U\kz=\id$. Specifically, it holds that
\begin{align*}
    \lmin(\covF\kk)\sq\leq \lmin(\covF\kp)\leq \lmax(\covF\kp)\leq \lmax(\covF\kk)\sq.
\end{align*}
Therefore, the exact spectral iteration \cref{eqn:U-spectral-exact} converges to $\Ustar$ at a \emph{quadratic} rate, implying that $\bigO{\log\log(1/\lambda)}$ iterations are enough to achieve $\eps$-accuracy.
Inspired by \eqref{eqn:U-spectral-exact}, we propose \cref{alg:U-LDP}, which privately approximates the spectral iteration with batched samples.

\begin{algorithm}
\caption{Subroutine $\LDPLU$}\label{alg:U-LDP}
\begin{algorithmic}
\REQUIRE Dataset $\cD=\crl*{(\x_t,y_t)}_{t\in[T]}$, \errpara~$\delta\in(0,1)$.
\REQUIRE Number of epochs $K\geq 1$, batch size $N=\floor{\frac{T}{K}}$, parameters $(\lambda\kz,\cdots,\lambda\kc)$.
\STATE Initialize $U\kz=\id$.
\FOR{$k=0,\cdots,K-1$}
    \FOR{$t=\rangekn$}
        \STATE Observe $\x_t\sim p$ and compute $V_t=\usqx[U\kk][t]$.
        \STATE Privatize $\til V_t\sim \sympriv[\Bd]{V_t}$.    
    \ENDFOR
    \STATE Compute $H\kk=\frac1N\sumkn \til V_t$.
    \STATE Update
    \begin{align*}
        \covF\kk=U\kk\sq H\kk U\kk\sq+\lambda\kk U\kk, \qquad
        U\kp=\sym(\covF\kk\isq U\kk).
    \end{align*}
\ENDFOR
\ENSURE Weight $(U\kc, \lambda\kc)$.
\end{algorithmic}
\end{algorithm}

\cref{alg:U-LDP} clearly preserves \aLDP.
We next show that its iterates converge to an approximate solution of $\FLDP(U)=\id$; see \cref{appdx:proof-U-LDP} for details.

\begin{proposition}\label{prop:alg-U-LDP}
Let $T\geq 1, K\geq 1$, $\delta\in(0,1)$, $\lambda\in(0,1]$. Suppose that
\begin{align*}
    \lambda\geq CK\Bd\siga\sqrt{\frac{d+\log(K/\delta)}{N}},
\end{align*}
where $C$ is a large absolute constant, and \cref{alg:U-LDP} is instantiated with parameters $\lambda\kk=\frac{2k+1}{2K+1}\lambda$. Then \whp,
\begin{align*}
    \exp\paren{ -\frac{\log(1/\lambda\kz)}{2^{k-1}} }\id \preceq \Ep{ \uxxu[U\kk] }+\lambda\kk U\kk \preceq \exp\paren{ \frac{8}{k} }\id.
\end{align*}
In particular, as long as $K\geq \max\sset{\log\log(1/\lambda\kz),12}$, \cref{alg:U-LDP} outputs $U$ such that $\Uapp$ \whp.
\end{proposition}

\newcommand{\hths}{\wh{\theta}^\star}

\paragraph{Step 2: Solving \iw~least squares}
Given the matrix $U$ and parameter $\lambda$ from Step 1, the second step is to approximate the minimizer of $\cLLDP$. Equivalently, the minimizer is the solution of the following linear equation:
\begin{align}\label{eq:L-LDP-solution}
    \prn*{\EE\brk*{\frac{U\x\x\tp}{\nrm{U\x}}}+\lambda} \cdot \theta = \EE\brk*{\frac{U\x}{\nrm{U\x}}y}.
\end{align}
The condition $\Uapp$ implies that this equation is \emph{well-conditioned} in the following sense, enabling efficient private approximation.
\begin{lemma}[Stability]\label{lem:L-LDP-solution}
Suppose that $\Uapp$ holds, $(\wh\Psi, \wh\psi)$ satisfies
\begin{align*}
    \nrmop{\wh\Psi-\EE\brk*{\frac{U\x\x\tp}{\nrm{U\x}}}}\leq \frac{\lambda}{8}, \qquad
    \nrm{\wh\psi-\EE\brk*{\frac{U\x}{\nrm{U\x}}y}}\leq \frac{\lambda}{2}.
\end{align*}
Then the following holds:

(a) All the singular values of $(\wh\Psi+\lambda\id)U$ belong to $[\frac14,4]$, and in particular, $\wh\Psi+\lambda\id$ is invertible.

(b) For $\hth=\argmin_{\theta\in\Bone}\nrmn{\prn{\wh \Psi+\lambda\id}\theta- \wh \psi}$, it holds that $\nrm{U\iv(\hth-\ths)}\leq 8\lambda$, %

\end{lemma}
This observation motivates \cref{alg:LDP-L1-regression}, which performs locally private linear regression using the privatized estimators $(\wh\Psi, \wh\psi)$.

\begin{algorithm}
\caption{Locally Private \METHOD~($\LDPLinearRegression$)
}\label{alg:LDP-linear-regression}
\begin{algorithmic}[1]
\REQUIRE Dataset $\cD=\crl*{(\x_t,y_t)}_{t\in[T]}$, \errpara~$\delta\in(0,1)$.
\STATE Set $N=\frac{T}{2}$.
\STATE Set $(U,\lambda)\leftarrow \LDPLU(\crl*{(\x_t,y_t)}_{t\in[N]},\delta)$.
\FOR{$t=N+1,\cdots,T$}
\STATE Compute
\begin{align*}
    \psi_t=\frac{U\x_t}{\nrm{U\x_t}}\cdot y_t, \qquad
    \Psi_t=\frac{U\x_t \x_t\tp}{\nrm{U\x_t}}.
\end{align*}
\STATE Privatize $\til \psi_t\sim \priv[2]{\psi_t}, \til \Psi_t \sim \priv[2B]{\Psi_t}$.
\ENDFOR
\STATE Compute
\begin{align*}
    \wh \psi=\frac{1}{N}\sum_{t=N+1}^{T} \til \psi_t, \qquad 
    \wh \Psi=\frac{1}{N}\sum_{t=N+1}^{T} \til \Psi_t.
\end{align*}
\STATE Set $\hth=\argmin_{\theta\in\Bone}\nrmn{\prn{\wh \Psi+\lambda\id}\theta- \wh \psi}$.
\ENSURE Weight $(U,\lambda)$, estimator $\hth$.
\end{algorithmic}
\end{algorithm}

\begin{theorem}\label{thm:LDP-linear-regression-full}
Let $T\geq 1, \delta\in(0,1)$. Then \cref{alg:LDP-linear-regression} preserves \aLDP, and it can be suitably instantiated so that, \whp[2\delta], the returned matrix $U$ satisfies $\Uapp$, and the returned estimator $\hth$ satisfies
\begin{align*}
    \nrm{U\iv\prn*{\hth-\ths}}\leq 8\lambda, \quad\text{where}\quad \lambda=&~ \tbO{ \Bd\siga\sqrt{\frac{d+\log(1/\delta)}{T}} },
\end{align*}
and $\tbO{\cdot}$ hides polynomial factors of $\log\log(T)$. 
\end{theorem}

The detailed proof is deferred to \cref{appdx:proof-LDP-linear-regression}. \cref{appdx:LDP-l1-regression} presents the analogous algorithm for generalized linear regression (\cref{alg:LDP-L1-regression}).

\paragraph{Implications}
We next discuss the implications of \cref{thm:LDP-linear-regression-full}.
By \cref{thm:LDP-linear-regression-full} and \cref{lem:U-L1-covariance}, we have the following high-probability guarantees of the estimator $\hth$ produced by \cref{alg:LDP-linear-regression}:

(a) Square loss convergence: It holds that
\begin{align*}
    \Ex{\lr \x, \hth-\ths\rr^2}=\nrmn{\hth-\ths}_{\cov}^2
    \leq \tbO{\frac{dB^2}{\alpha^2 T}} \cdot \nrmop{\cov\isq \Ustar}^2.
\end{align*}

(b) Uniform point-wise estimation guarantee (confidence interval): For any vector $\x\in\R^d$, it holds that
\begin{align*}
    \absn{\lr \x, \hth-\ths \rr}\leq \tbO{\sqrt{\frac{dB^2}{\alpha^2 T}}} \cdot \nrm{\Ustar \x}.
\end{align*}

In \cref{ssec:U-lower-bounds}, we argue that these guarantees are nearly minimax-optimal in the \emph{distribution-specific} sense.

\paragraph{\Lone-convergence guarantee} In addition to the squared-error guarantees, we have the following \Lone-convergence rate (by \cref{lem:U-L1-covariance}):
\begin{align}\label{eq:LDP-Lone}
    \EE_{\x\sim p}\abs{\lr \x, \hth-\ths\rr}\leq \sqrt{d}\nrmn{\Ustar\iv(\hth-\ths)}=\tbO{\frac{dB}{\alpha\sqrt{T}}},
\end{align}
which is known to be minimax-optimal~\citep{chen2024private}.
Remarkably, this \Lone-convergence rate is \emph{independent} of the structure of the covariate distribution $p$, in contrast to the \Ltwo-convergence rate, which necessarily depends on the ``condition number'' $\nrmop{\cov\isq \Ustar}$ highlighted above. Consequently, the guarantee is especially useful for LDP linear regression under general, ill-conditioned covariate distributions. 

As an application, we use the \Lone-guarantees to obtain rate-optimal regret in linear contextual bandits (\cref{sec:cb}). %

\subsection{Statistical optimality}\label{ssec:U-lower-bounds}

In this section, we fix an arbitrary covariate distribution $p$ and establish lower bounds for \emph{any} locally private algorithm that may even know $p$ exactly.

The key ingredients for our lower bound is \eqref{eq:Ustar-Lone-covariance}. Particularly, for two linear models $M_{p,\theta}$, $M_{p,0}$ with covariate distribution $p$ and label $y\in\crl*{-1,1}$, \eqref{eq:Ustar-Lone-covariance} implies that
\begin{align*}
    \DTV{M_{p,\theta}, M_{p,0}}=\EE_{x\sim p}\abs{\lr \x, \theta\rr}\leq \sqrt{d}\nrm{\Ustar\iv \theta}.
\end{align*}
Therefore, using the strong data-processing inequality of locally private channels~\citep{duchi2013local}, the proof of the following lower bound is straightforward (detailed in \cref{appdx:proof-LDP-lower}).
\begin{theorem}[Distribution-specific lower bounds]\label{lem:LDP-lower-bound-any}
Fix a covariate distribution $p$ that is supported on $\Ball(\Bd)$, $T\geq 1$, matrix $A$, and privacy parameter $\alpha\in(0,1]$, $\beta\in[0,\frac{1}{8T}]$. 

(a) Suppose that $\lambda\geq\frac{B}{\alpha\sqrt{dT}}$. Then there exists a parameter $\theta\in \Ball(1/B)$ such that 
\begin{align*}
\inf_{\LDPtag}\sup_{\ths\in\set{0,\theta}}\EE\sups{\ths,\alg} \nrmn{\hth-\ths}_A^2\geq \frac{c}{d}\cdot \frac{\nrmop{A\sq \Ustar}^2}{\alpha^2 T},
\end{align*}
where $c>0$ is an absolute constant.

(b) Suppose that $\lambda\geq\frac{B}{\alpha\sqrt{T}}$. Then
\begin{align*}
    \inf_{\LDPtag}\sup_{\ths\in\Ball(\frac{1}{2B})}\EE\sups{\ths,\alg} \nrmn{\hth-\ths}_A^2\geq \frac{c}{d}\cdot \frac{\trd{A}{\Ustar}}{\alpha^2 T}.
\end{align*}
\end{theorem}

Both of the lower bounds above apply to the setting where the covariate distribution $p$ is known to the algorithm.
We further note that the lower bound (a) above hold even in the \emph{local-minimax} setting where the parameter $\ths$ is known to belong to a set $\crl{0,\ths}$~\citep{duchi2024right}. 

As a corollary, we provide the following characterization of the minimax-optimal convergence rate of locally private linear regression.

\begin{corollary}[Minimax-optimality]\label{cor:LDP-minimax}
Let $T\geq 1$, privacy parameter $\alpha\in(0,1]$, $\beta\in[0,\frac{1}{8T}]$, and PSD matrix $A$ be given. Suppose that the covariate $p$ is supported on $\Bone$. Then it holds that
\begin{align*}
    \frac{c}{d}\cdot \frac{\trd{A}{\Ustar}}{\alpha^2 T}\leq \inf_{\LDPtag}\sup_{\ths\in\Bone}\EE\sups{\ths,\alg} \nrmn{\hth-\ths}_A^2\leq C\cdot \frac{\trd{A}{\Ustar}}{\alpha^2 T},
\end{align*}
where $\lambda=\frac{1}{\alpha\sqrt{T}}$, and where $C,c>0$ are absolute constants.
\end{corollary}

Therefore, up to a factor of $O(d)$, the minimax estimation error under the loss $\nrm{\cdot}_A^2$ is completely characterized by the quantity $\trd{A}{\Ustar}$.

\begin{remark}
    The upper bound of \cref{cor:LDP-minimax} is slightly tighter than \cref{thm:LDP-linear-regression-full}, because in the setting where the covariate distribution $p$ is known (or the learner is additionally given public samples from $p$), the spectral iteration~\eqref{eqn:U-spectral-exact} can be approximated more accurately, leading to a better convergence rate of \method~(details in \cref{appdx:proof-LDP-minimax-optimal}). 
\end{remark}

\paragraph{Implications for light-tailed distributions} We show that for light-tailed distributions, $\Ustar$ scales approximately as $(\cov+\lambda)\isq$. We introduce the following measure of the tail behavior of a distribution $p$:
\begin{align}\label{def:kappa-p}
    \kappa(p)\ldef \inf\crl*{M: \EE_{\x\sim p}\brk*{ \x\x\tp \indic{\nrm{\x}_{\cov^\dagger} \leq M}}\succeq \frac12\cov}.
\end{align}

\begin{lemma}[\Infom~under light-tailed distributions]
It holds that
\begin{align*}
    \frac14(\cov+\lambda^2\id)^{-1}\preceq \Ustar^2\preceq 4\kappa(p)^2(\cov+\lambda^2\id)^{-1}.
\end{align*}
\end{lemma}

We note that when $\cov\isq \x$ is $K$-sub-Gaussian under the distribution $p$, it holds that $\kappa(p)\leq O(K\sqrt{d})$. Furthermore, for any real number $s>0$, denote $\moment{s}\defeq \Ex{\nrm{\x}_{\cov^\dagger}^s}$. Then $\kappa(p)\leq \prn*{2\moment{s}}^{\frac{1}{s-2}}$ for any $s>2$. Therefore, when $p$ is sub-Gaussian or hyper-contractive, $\kappa(p)$ can be well controlled.
The detailed discussion is deferred to \cref{appdx:proof-light-tail}.

\section{\IW~Linear Regression with Global Privacy}\label{sec:DP-linear}

We begin by introducing the operator $\FJDP: \PSD\to\PSD$, defined as
\begin{align}\label{def:FJDP}
    \FJDP(W)\defeq \Ex{\frac{W\x\x\tp W}{1+\gam\nrm{W\x}} }+\lambda W.
\end{align}
Recall that $\Wstar$ is defined as the (unique) solution of $\FJDP(W)=\id$. We summarize the properties of $\FJDP$ and $\Wstar$ in the following lemma.

\begin{lemma}\label{lem:Wstar-properties}
For any $\gamma,\lambda>0$, the solution $\Wstar$ to $\FJDP(W)=\id$ exists and is unique, and the map $(\gamma,\lambda)\mapsto \Wstar$ is continuous. 

Further, for any $W\in\PSD$ such that $\Wapp$, it holds that $\frac14 W^2\preceq \Wstar^2\preceq 4W^2$.
\end{lemma}

Additionally, $\Wstar^2\succeq (\cov+\lambda^2 \id)^{-1}$.

\subsection{Differentially Private \method}\label{ssec:linear-JDP-upper}
We now develop the \emph{\method} method, a two-step procedure:
\begin{itemize}
\item Step 1: Privately learn a matrix $W$ such that $\frac12\id\preceq \FJDP(W)\preceq 2\id$, where $\lambda\asymp \frac{1}{\sqrt{T}}, \gamma\asymp \frac{\sqrt{d}}{\alpha\sqrt{T}}$.
\item Step 2: Optimize the \emph{\iw} squared-loss objective:
\begin{align}\label{def:loss-JDP}
    \cLJDP(\theta)\defeq \frac12\EE\brk*{ \frac{\prn*{\lr \x, \theta \rr -y}^2}{1+\gamma\nrm{W\x}} }+\frac{\lambda}{2} \nrm{\theta}^2_{W\iv}.
\end{align}
\end{itemize}

We describe the two steps in more detail below.

\paragraph{Step 1: Learning $\Wstar$ with DP}
Analogously to the local model, the following spectral iterates converge rapidly to $\Wstar$:
\begin{align}\label{eqn:W-spectral-exact}
    \covF\kk=\Ep{\wxxw[W\kk]}+\lambda W\kk, \qquad
    W\kp=\sym(\covF\kk\isq W\kk).
\end{align}
Specifically, it can be shown that (details in \cref{sec:properties-U-W})
\begin{align*}
    \min\crl*{ \lmin(\covF\kk)\sq , 1} \leq \lmin(\covF\kp)\leq \lmax(\covF\kp)\leq \max\crl*{ \lmax(\covF\kk)\sq, 1}.
\end{align*}
Therefore, the exact spectral update \eqref{eqn:W-spectral-exact} converges to $\Wstar$ at a quadratic rate.
We next propose $\JDPLU$ (\cref{alg:W-JDP}), which privately approximates the iteration \eqref{eqn:W-spectral-exact}.

\begin{algorithm}
\caption{Subroutine $\JDPLU$ %
}\label{alg:W-JDP}
\begin{algorithmic}
\REQUIRE Dataset $\cD=\sset{(\x_t,y_t)}_{t\in[T]}$, \errpara~$\delta\in(0,1)$.
\REQUIRE Number of epochs $K\geq 1$, batch size $N=\floor{\frac{T}{K}}$, and parameters $(\gamma,\lambda)$.
\STATE Initialize $W\kz=\id$.
\FOR{$k=0,\cdots,K-1$}
    \STATE Compute the estimate on the $k$-th data split $\dataset\kk=\sset{(\x_t,y_t)}_{t\in[kN+1,(k+1)N]}$:
    \begin{align}\label{eq:spectral-approx-JDP}
        H\kk=\frac1N\sumkn \wsqx[W\kk][t].
    \end{align}
    \STATE Privatize $\til H\kk\sim \sympriv[\frac{B}{\gamma N}]{ H\kk }$.
    \STATE Update
    \begin{align*}
        \covF\kk=W\kk\sq \til H\kk W\kk\sq+\lambda W\kk, \qquad
        W\kp=\sym(\covF\kk\isq W\kk).
    \end{align*}
\ENDFOR
\ENSURE $W=W\kc$, parameter $(\gamma,\lambda)$.
\end{algorithmic}
\end{algorithm}

\cref{alg:W-JDP} is \aDP by composition property of DP mechanisms; see \cref{appdx:JDP-verify}. We present its guarantee below.

\begin{proposition}\label{prop:JDP-W}
Let $\alpha,\beta,\delta\in(0,1]$. Subroutine $\JDPLU$~(\cref{alg:W-JDP}) preserves \aJDP. Suppose that the parameters $(\gamma,\lambda)$ in \cref{alg:W-JDP} are suitably chosen so that 
\begin{align*}
    \gamma \lambda\geq C \frac{\siga B\sqrt{d+ \log(K/\delta)}}{N}, 
\end{align*}
and $K\geq \max\crl*{\log\log(1/\lambda),4}$. Then, \whp, the matrix $W=W\kc$ returned by \cref{alg:W-JDP} satisfies $\Wapp$.
\end{proposition}

The proof is presented in \cref{appdx:proof-JDP-W}.

\paragraph{Step 2: \Method~with DP}
Given a matrix $W$ such that $\Wapp$, it remains to approximate the minimizer of $\cLJDP$ privately, which is analogous to the local privacy setting. Equivalently, we solve the following linear equation:
\begin{align}\label{eq:L-JDP-solution}
    \prn*{ \EE\brk*{\frac{W\x\x\tp}{1+\gamma\nrm{W\x}}}+\lambda\id } \theta=\EE\brk*{\frac{W\x }{1+\gamma\nrm{W\x}}y},
\end{align}
As each data point $(\x,y)$ is weighted by $\frac{1}{1+\gamma \nrm{W\x}}$, both expectations can be privately estimated by sufficient statistic perturbation (with sensitivity $\frac1\gamma$). Further, similar to \cref{lem:L-LDP-solution}, we can show that \eqref{eq:L-JDP-solution} is also ``stable'' as long as $\Wapp$.
Based on the observations above, we propose $\JDPLinearRegression$ (\cref{alg:JDP-linear-regression}), which has the following guarantee.

\begin{algorithm}
\caption{Differentially Private \Method~($\JDPLinearRegression$)
}\label{alg:JDP-linear-regression}
\begin{algorithmic}[1]
\REQUIRE Dataset $\dataset=\sset{(\x_t,y_t)}_{t\in[T]}$, \errpara~$\delta\in(0,1)$.
\STATE Split the dataset $\dataset=\dataset_0 \cup \cD_1$ equally.
\STATE Set $(W,\gamma,\lambda) \leftarrow \JDPLU(\dataset_0,\delta)$.
\STATE Compute
\begin{align*}
    \psi=\frac{1}{|\cD_1|}\sum_{(\x,y)\in \cD_1}\frac{W\x}{1+\gamma\nrm{W\x}}\cdot y, \qquad
    \Psi=\frac{1}{|\cD_1|}\sum_{(\x,y)\in \cD_1} \frac{W\x\x\tp}{1+\gamma\nrm{W\x}}.
\end{align*}
\STATE Privatize $\til \psi\sim \priv[\frac{2}{\gamma T}]{\psi}, \til \Psi \sim \priv[\frac{2B}{\gamma T}]{\Psi}$.
\STATE Set $\hth=\prn{\til \Psi+\lambda\id}\iv \til \psi$.
\ENSURE Estimator $\hth$, $(W,\gamma,\lambda)$.
\end{algorithmic}
\end{algorithm}

\begin{theorem}[DP linear regression]\label{thm:JDP-linear-regression-full}
Let $T\geq 1, \delta\in(0,1)$, $\alpha,\beta\in(0,1]$. Then \cref{alg:JDP-linear-regression} preserves \aDP. Further, the parameters $(\lambda,\gamma)$ can be chosen as
\begin{align*}
    \lambda=\frac{1}{\sqrt{T}}, \qquad
    \gamma=\tbO{B\siga \frac{\sqrt{d}+\log(1/\delta)}{T}},
\end{align*}
where $\tbO{\cdot}$ hides polynomial factors of $\log\log(T)$, so that, \whp[2\delta], the following holds.

(a) The returned matrix $W$ satisfies $\Wapp$, and the returned estimator $\hth$ satisfies
\begin{align*}
    \nrm{W\iv\prn*{\hth-\ths}}\leq\bigO{\sqrt{\frac{d\log(1/\delta)}{T}}}.
\end{align*}

(b) For any fixed PSD matrix $A$, \whp,
\begin{align*}
    \nrm{\hth-\ths}_{A}^2\leq \bigO{ \frac{\log(1/\delta)}{T}} \cdot \trd{A}{\Wstar}.
\end{align*}
\end{theorem}

As a direct corollary of \cref{alg:JDP-linear-regression}, we have the following guarantees for the estimator $\hth$ produced by \cref{alg:JDP-linear-regression}:

(1) Squared-loss guarantee: \whp[3\delta],
\begin{align*}
    \nrmn{\hth-\ths}_{\cov}\leq \bigO{ \sqrt{\frac{\log(1/\delta)}{T}}} \cdot \nrmF{\cov\sq \Wstar}.
\end{align*}

(2) Pointwise confidence interval: For any fixed vector $\x\in\R^d$, \whp[3\delta],
\begin{align*}
    \abs{\lr \x, \hth-\ths\rr}\leq \bigO{ \sqrt{\frac{\log(1/\delta)}{T}}} \cdot \nrm{\Wstar\x}.
\end{align*}

\paragraph{$\Lone$-convergence guarantee}
Using the fact that $\EE_{\x\sim p}\nrm{\Wstar \x}\leq \sqrt{d}+d\gamma$ (\cref{lem:E-nrm-W-bound}), we also obtain a distribution-independent guarantee for the \Lone~error. To derive the tightest bound, in the following proposition we adopt a slightly different choice of $(\lambda,\gamma)$ in \cref{alg:JDP-linear-regression}. 
\begin{proposition}[\Lone-convergence]
Suppose that the linear model is well-specified, and the parameter $(\lambda,\gamma)$ of \cref{alg:JDP-linear-regression} are chosen as
\begin{align*}
    \lambda=\sqrt{\frac{d\log(1/\delta)}{T}}, \qquad
    \gamma=\tbO{\frac{\Bd\siga}{\sqrt{T}}}.
\end{align*}
Then, \whp[2\delta], it holds that $\Wapp$, and the estimator $\hth$ satisfies %
\begin{align*}
    \EE_{\x\sim p}{\absn{\lr \x, \hth-\ths\rr}}\leq 
    \bigO{\sqrt{\frac{d\log(1/\delta)}{T}}}+\tbO{\Bd\siga\frac{d\sqrt{\log(1/\delta)}}{T}}.
\end{align*}
\end{proposition}

The first term matches the minimax-optimal convergence rate of non-private \Lone-regression (given by OLS), and the second term (``price-of-privacy'') is of lower order. To the best of our knowledge, this \Lone~convergence guarantee is new and parallels the results in the LDP setting~\cref{eq:LDP-Lone}. 

\subsection{Statistical optimality}

We next show that $\Wstar$ also yields lower bounds for \emph{any} differentially private algorithm, demonstrating that $\Wstar$ captures the statistical difficulty of the underlying covariate distribution, and that our \method~method (\cref{alg:JDP-linear-regression}) is nearly optimal.
Specifically, in \cref{lem:JDP-lower-bound-any}, we prove lower bounds for algorithms that (i) privatize only the \emph{labels} (\cref{def:yDP}), and (ii) may know the distribution $p$ in advance. See \cref{appdx:proof-W-lower} for details.

\newcommand{\ulambda}{\underline{\lambda}}
\newcommand{\ugamma}{\underline{\gamma}}
\newcommand{\olambda}{\overline{\lambda}}
\newcommand{\ogamma}{\overline{\gamma}}
\newcommand{\lambdas}{\lambda^\star}
\newcommand{\gammas}{\gamma^\star}
\begin{theorem}[Lower bound of estimation error with DP]\label{lem:JDP-lower-bound-any}
Let $T\geq 1$, $\alpha,\beta\in(0,1]$, the covariate distribution $p$, and a PSD matrix $A$ be given. Suppose that $\beta\leq \alpha$ and $p$ is supported on the ball $\Ball(B)$. Then the following holds.

(a) For any $\ulambda\geq\frac{B}{\sqrt{T}}$ and $\ugamma\leq\frac{1}{\alpha\sqrt{dT}}$, there exists a parameter $\theta$ with $\nrm{\theta}\leq \frac1B$ such that for any $T$-round \yDP~estimator $\hth$, it holds that
\begin{align*}
    \sup_{\ths\in\set{0,\theta}}\EE\sups{\ths,p} \nrmn{\hth-\ths}_A\geq \frac{c}{\sqrt{T}}\nrmop{A\sq \Wstar[\ugamma,\ulambda]},
\end{align*}
where $c>0$ is an absolute constant.

(b) For any $\ulambda\geq \frac{B\sqrt{d}}{\sqrt{T}}$ and $\ugamma\leq \frac{1}{\alpha\sqrt{dT}}$, any $T$-round \yDP~estimator $\hth$, it holds that
\begin{align*}
    \sup_{\ths:\nrm{\ths}\leq \frac{1}{B}}\EE\sups{\ths,p} \nrmn{\hth-\ths}_{A}^2\geq \frac{c}{T}\trd{A}{\Wstar[\ugamma,\ulambda]}.
\end{align*}
\end{theorem}

Combining the lower bounds above with the upper bounds of \cref{thm:JDP-linear-regression-full} yields the following corollary (proved in \cref{appdx:proof-JDP-minimax-optimal}).\footnote{The upper bound in \cref{prop:JDP-minimax-optimal} is slightly sharper than \cref{thm:JDP-linear-regression-full} because the covariate distribution $p$ is fixed, allowing $\Wstar$ to be approximated more accurately.}
\newcommand{\ogammaval}{\sqrt{\frac{\log(1/\beta)}{\alpha^2 T}}}
\begin{corollary}\label{prop:JDP-minimax-optimal}
Let privacy parameters $\alpha\in(0,1], \beta\in(0,\frac12]$, covariate distribution $p$ and a PSD matrix $A$ be given. Suppose that $\beta\leq \alpha$ and $p$ is supported on $\Ball(B)$. Then there are functions $\gammas, \lambdas: \NN\to \R_{\geq 0}$ such that
\begin{align*}
    \frac{1}{\sqrt{T}}\leq \lambdas(T)\leq \frac{B\sqrt{d}}{\sqrt{T}}, \qquad
    \frac{1}{\sqrt{\alpha^2dT}}\leq \gammas(T)\leq \ogammaval,
\end{align*}
such that the following holds for $T\geq 1$:
\begin{align*}
    \frac{c}{T}\cdot \trd{A}{\Wstar[\gammas(T),\lambdas(T)]}\leq \inf_{\DPtag}\sup_{\ths\in\Ball(1/B)}\EE\sups{\ths,p} \nrmn{\hth-\ths}_{A}^2 \leq\frac{C}{T}\cdot \trd{A}{\Wstar[\gammas(T),\lambdas(T)]},
\end{align*}
where $C,c>0$ are absolute constants.
\end{corollary}

Therefore, $\Wstar$ acts as a \emph{distribution-specific} complexity measure, capturing the statistical difficulty of private linear regression under $p$. 

While $\gammas(T)\asymp_d \frac{1}{\alpha\sqrt{T}}$ and $\lambdas(T)\asymp_d \frac{1}{\sqrt{T}}$ pin down the dependence on $(T,\alpha)$, tightening the dependence on $d$ and $\log(1/\beta)$ remains open, especially when $p$ is unknown.

\subsection{The price of privacy in DP linear regression}\label{ssec:PoP}

In the literature on differentially private learning, the ``\emph{price of privacy}'' quantifies the additional sample size needed to recover the non-private convergence rate. For private linear regression we define
\begin{align}\label{eq:PoP}
    \PoP\ldef \frac{\inf_{\DPtag}~~\sup_{\ths\in\Bone} \EE\sups{p, \ths} \nrmn{\hth-\ths}_{\cov}^2 }{d/T}.
\end{align}
In words, $\PoP$ compares the squared error of the optimal \aDP~estimator with $T$ samples to that of the optimal non-private estimator (which equals $\frac{d}{T}$, e.g., for OLS). By \cref{prop:JDP-minimax-optimal},
\begin{align*}
    \PoP\asymp \frac{\trd{\cov}{\Wstar[\gammas,\lambdas]}}{d}, \qquad\text{where}~~ \gammas\asymp_d \frac{1}{\alpha\sqrt{T}}, ~~\lambdas\asymp_d\frac{1}{\sqrt{T}}
\end{align*}
From the definition~\cref{def:W-demo} of $\Wstar$, it is clear that when $\cov\succ 0$, we have $\Wstar\to \cov\isq$ as $\lambda,\gamma\to 0$, and hence $\PoP=\bigO{1}$ for sufficiently large $T$. More specifically, in \cref{lem:kappa-p} we show that for any covariate distribution $p$,
\begin{align}\label{eq:W-to-cov}
    \nrmop{\cov\sq \Wstar}\leq 2+2\gamma\kappa(p), \qquad \text{where}~\kappa(p) \ldef \inf\crl*{M: \EE_{\x\sim p}\brk*{ \x\x\tp \indic{\nrm{\x}_{\cov^\dagger} \leq M}}\succeq \frac12\cov}.
\end{align}
Therefore, $\PoP=\bigO{1}$ as long as $T\geqsim \frac{\kappa(p)^2\log(1/\beta)}{\alpha^2}$. 
To the best of our knowledge, this is the smallest known threshold for $\PoP=\bigO{1}$ for general distributions $p$, %
as $\kappa(p)$ can be bounded by commonly considered ``condition numbers'', including the minimum eigenvalue of the covariance matrix $\cov$~\citep{sheffet2017differentially,shariff2018differentially,wang2018revisiting,cai2021cost,varshney2022nearly}, the sub-Gaussian parameter, and the hyper-contractive parameter~\citep{brown2024insufficient,liu2022differential}. %

It remains open to determine the \emph{explicit} and \emph{tight} threshold of $T$ for the \method~estimator to achieve the $\frac{d}{T}$ rate. Such a characterization would require a more explicit description of the behavior of $\Wstar$ as $(\gamma,\lambda)$ vary; we leave this as an important direction for understanding the privacy-utility tradeoff.

\paragraph{High-privacy regime}
The discussion above mainly focuses on the regime $T\gg \frac{1}{\alpha^2}$. By contrast, in the \emph{high-privacy} regime $\alpha\leqsim \frac{1}{\sqrt{T}}$, the characterization in \cref{prop:JDP-minimax-optimal} can be further simplified using the relation between $\Wstar$ and $\Ustar[\gamma\lambda]$. %
\begin{lemma}\label{lem:Wstar-Ustar}
For any $\gamma\geq 0$ and $\lambda>0$, it holds that
\begin{align*}
    \gam^2\Ustar[\gamma\lambda]^2\preceq \Wstar^2\preceq 16(1+\gam)^2\Ustar[(1+\gamma)\lambda]^2.
\end{align*}
\end{lemma}

Therefore, when $\gamma\geq 1$, the behavior of $\Wstar$ is (up to constant) equivalent to $\gamma \Ustar[\gamma\lambda]$. This corresponds to the ``high-privacy'' regime $\alpha\leq \frac{1}{\sqrt{T}}$.

Specifically, suppose that $\beta\leq \alpha$ and $p$ is supported on $\Bone$ for simplicity. Then, for any $T\leq \frac{1}{\alpha^2}$, \cref{prop:JDP-minimax-optimal} implies the following characterization:
\begin{align}\label{eq:JDP-minimax-high-privacy}
    \frac{c'}{d}\cdot \frac{\trd{A}{\Ustar[\lambda(T)]}}{\alpha^2 T^2}\leq \inf_{\DPtag}\sup_{\ths\in\Bone}\EE\sups{\ths,p} \nrmn{\hth-\ths}_{A}^2 \leq C'\log(1/\beta) \cdot \frac{\trd{A}{\Ustar[\lambda(T)]}}{\alpha^2 T^2},
\end{align}
where $\lambda(T)=\frac{1}{\alpha^2T}$.

\section{Application: Rate-Optimal Regret in Private Linear Contextual Bandits}\label{sec:cb}

As a main application, we leverage \method~from \cref{sec:LDP-linear} and \cref{sec:DP-linear} to design an efficient algorithm for private linear contextual bandits that achieves rate-optimal regret.
We begin by recalling the notions of \emph{joint} and \emph{local} privacy for interactive algorithms (\cref{ssec:interactive-privacy}). Next we revisit why most existing private algorithms for linear contextual bandits fail to attain optimal regret without the strong explorability condition $\lmins>0$ (\cref{ssec:cb-negative}). Finally, \cref{ssec:cb-algs} presents our \method-based algorithm that achieves rate-optimal regret.

\newcommand{\redp}[1]{{\color{red} #1}}

\begin{table}
\centering
\renewcommand{\arraystretch}{1.4}
\begin{tabular}{|c|c|c|Hc|Hc|}
\hline
Setting & Results & Regret bound & Gen & Opt & Adv & Poly \\
\hline
\multirow{4}{*}{Joint DP} & \citet{shariff2018differentially} & $d\sqrt{T}+d^{3/4}\redp{\sqrt{\frac{T}{\alpha}}}$ & \yes & \no & \yes & \yes \\\cline{2-7}
 & \cref{thm:regret-upper-JDP} & $d^2\sqrt{T}+\frac{d^{5/2}}{\alpha}$ & \yes & \yes & \no & \yes \\\cline{2-7}
 & \cref{thm:regret-upper-JDP-better}, $|\cA|=\bigO{1}$ & $\sqrt{dT}+\frac{d^{3/2}}{\alpha}$ & \no & \yes & \no & \yes \\\cline{2-7}
 & Lower bound~\citep{he2022reduction} & $\sqrt{dT\log|\cA|}+\frac{d}{\alpha}$ & / & /  & / & /\\
\hline
\multirow{6}{*}{Local DP} & \citet{zheng2020locally} & $\redp{(dT)^{3/4}}/\alpha$ & \yes & \no & \yes & \yes \\\cline{2-7}
     & \citet{han2021generalized}$^\dagger$ & $\sqrt{d\log|\cA|\cdot T}/(\redp{\lmins} \alpha)$ & \yes & \no & \yes & \yes \\\cline{2-7}
     & \citet{li2024optimal} & $|\cA|^2\redp{\log^d(T)}\sqrt{T}/\alpha$ & \no & \yes & \no & \no \\\cline{2-7}
     & \citet{chen2024private} & $ \sqrt{d^3T}/\alpha$ & \yes & \yes & \yes & \no \\\cline{2-7}
     & \cref{thm:regret-upper-LDP} & $\sqrt{d^5T}/\alpha$ & \yes & \yes &  \no & \yes \\\cline{2-7}
     & Lower bound~\citep{chen2024private} & $\sqrt{d^2T}/\alpha$ & / & / & / & / \\
\hline
\end{tabular}
\caption{Summary of the existing results for private learning in (generalized) linear contextual bandits. We highlight the undesirable factors from previous regret bounds in red.
The ``Opt'' column indicates whether the regret bounds attain the optimal rate of $\sqrt{T}+\frac{1}{\alpha}$ under JDP and $\sqrt{T}/\alpha$ under LDP in terms of $(T,\alpha)$, up to $\poly(d)$ and logarithmic factors.
The ``Poly'' column indicates whether the algorithms operate with amortized time complexity of $\poly(d)$.
$\dagger$The results of \citet{han2021generalized} rely on assumptions that are slightly stronger than the ``explorability'' condition $\lmins>0$, and their regret bound is always lower bounded by $\sqrt{d\log|\cA|\cdot T}/(\lmins \alpha)$ stated in the table. Note that $\lmins\leq \frac{1}{d}$ always holds.
}
\label{tab:comp}
\end{table}

\paragraph{Comparison with existing results}
\cref{tab:comp} summarizes known regret bounds for learning in (generalized) linear contextual bandits under joint or local privacy. For clarity, we specialize all results to linear contextual bandits and omit poly-logarithmic factors.

\subsection{Privacy in interactive settings}\label{ssec:interactive-privacy}

\newcommand{\bfa}{\boldsymbol{a}}

In a $T$-round contextual bandit problem, an interactive algorithm $\alg$ is described by a sequence of mappings $\set{ \pi_t(\cdot\mid\cdot) }$, where the $t$-th mapping $\pi\ind{t}(\cdot\mid{}\cH\ind{t-1},x_t)$ selects an action $a_t\in\cA$ based on the history $\cH\ind{t-1}=\set{ (x_s,a_s,r_s) }_{s\leq t-1}$ and the current context $x_t$. 

For each $t\in[T]$, we write $\bfa_{-t}=(a_1,\cdots,a_{t-1},a_{t+1},\cdots,a_T)$ to denote the sequence of all actions taken by the algorithm except the action of the $t$-th round.
The following notion of joint DP for interactive algorithms was introduced by \citet{shariff2018differentially}, building on \citet{dwork2010differential,kearns2014mechanism}.

\newcommand{\ggt}{>t}
\begin{definition}[Joint DP]\label{def:JDP-interactive}
An algorithm $\alg$ \emph{preserves \aJDP} (or is \emph{\aJDP}) if, for every $t\in[T]$ and any two neighboring datasets $\cD=\set{(x_s,r_s)}_{s\in[T]}$ and $\cD'=\set{(x_s',r_s')}_{s\in[T]}$ differing only at round $t$, it holds that 
\begin{align*}
    \PP\sups{\alg}\paren{ \bfa_{-t}\in E|\cD }\leq \ea\PP\sups{\alg}\paren{ \bfa_{-t}\in E|\cD' }+\beta, \qquad \forall E\subseteq \cA^{T-1},
\end{align*}
where the probability $\PP\sups{\alg}$ accounts only for the randomness of $\alg$, i.e., for any dataset $\cD=\set{(x_s,r_s)}_{s\in[T]}$,
\begin{align*}
    \PP\sups{\alg}\paren{ a_1,\cdots,a_T|\cD }=\prod_{t'=1}^T \pi_{t'}(a_{t'}|(x_s,a_s,r_s)_{s<t'},x_{t'}), \qquad \forall (a_1,\cdots,a_T)\in\cA^T.
\end{align*}
\end{definition}

\cref{def:JDP-interactive} is the widely-adopted definition of privacy-preserving procedures in the literature of contextual bandits~\citep{shariff2018differentially, azize2024open} and episodic RL~\citep{vietri2020private}, and it can be interpreted as following.
Assume that a malicious adversary is trying to identify the private information $z_t=(x_t,a_t,r_t)$ of the unit which is treated at round $t\in[T]$, and it can adversarially design the context $x_{s}$ and reward $r_s$ for all other rounds $s\neq t$. An algorithm is private if the adversary cannot infer the private information  $(x_t,a_t,r_t)$ from the output actions $\bfa_{-t}$ of the algorithm no matter %
how the adversary designs the input dataset 
$\cD_{-t}=\{(x_s,r_s)\}_{s\neq t}$. 
In the meanwhile, the action $a_t$ unavoidably contains information about $x_t$, otherwise no non-trivial regret guarantee could be attained, as proved by~\citet{shariff2018differentially}.

Parallel to the above joint DP model, we also study contextual bandits under the local DP (LDP) model~\citep{zheng2020locally,han2021generalized,li2024optimal}. To define LDP interactive algorithms in the simplest form, we adopt the following formulation~\citep{chen2024private}.

\begin{definition}[Local DP]
In contextual bandits problems, a $T$-round algorithm $\alg$ is said to preserve \aLDP~if it adopts the following protocol for each round $t=1,...,T$:
\begin{itemize}
  \setlength{\parskip}{2pt}
    \item $\alg$ selects a decision $\pi_t:\cX\to\Delta(\cA)$ and a \aDP~channel $\pr_t$ based on the history $\cH_{t-1}$.
    \item The environment generates the context $x_t$, action $a_t\sim \pi_t(x_t)$, and reward $r_t$. %
    \item $\alg$ receives a noisy observation $o_t$ drawn as $o_t\sim \pr_t(\cdot|x_t,a_t,r_t)$.
\end{itemize}
\end{definition}

In other words, an algorithm that preserves local DP never has direct access to the observation $(x_t,a_t,r_t)$, and only the ``privatized'' observation $o_t$ is revealed. Therefore, an LDP algorithm has to adaptively select both the decision $\pi_t$ and also the private channel $\pr_t$ to obtain information from the environment.

\subsection{Insufficiency of Existing Algorithms}\label{ssec:cb-negative}

In general, existing algorithmic principles for learning contextual bandits mostly rely (explicitly or implicitly) on regression subroutines that, given a sequence of observations $\set{(x_t,a_t,r_t)}_{t\in[]}$, produce a reward estimate $\fhat$ %
with bounded mean-square error (MSE):
\begin{align*}
    \EE_{(x,a)\sim \cD} (\fhat(x,a)-\fs(x,a) )^2\leq \cE(N)^2,
\end{align*}
where $\cE:\NN\to \RR_{\geq0}$ is the MSE convergence rate of the estimator $\fhat$.

For non-private linear contextual bandits, it is well-known that regression-based estimators (e.g., OLS) achieve the optimal rate of $\cE(T)^2\asymp \frac{d}{T}$. This $T^{-1}$-rate is central to the regret analysis of the classical LinUCB algorithm and its variants ~\citep{abbasi2011improved,li2019nearly,bastani2020online}. Roughly speaking, their analysis gives a regret bound of $L\cdot T\cdot \cE(T)$ for any estimator with MSE rate $\cE$, where $L\leq O(\sqrt{d})$ can be regarded as the complexity of exploration. 
Further, for contextual bandits with a general reward function class,
the recent regression-oracle based algorithms~\citep{foster2020beyond,simchi2020bypassing} achieve an \emph{oracle-based regret bounds} scaling with $\tbO{|\cA|T\cdot \cE(T)}$. %
Overall, for these approaches, achieving a $T^{-1}$ convergence rate under \Ltwo-error is crucial
to obtain $\sqrt{T}$-regret.

However, as we have shown in \cref{ssec:negative}, for linear models with ``heavy-tailed'' covariate distributions, privacy \emph{does} lead to degradation under \Ltwo-error. In the following proposition, we reformulate the lower bounds in terms of the regression oracles.

\begin{proposition}[Lower bounds for regression oracle]\label{prop:lower-linear-est}
Suppose that $T\geq 1$, $\alpha,\beta\in(0,1]$, $\lmins\in(0,1]$. Suppose that $\alg$ is a regression oracle that, given any dataset $\cD=\crl*{(\x_t,y_t)}_{t\in[T]}$ generated i.i.d from a linear model with $\lmin(\cov)\geq \lmins$, it returns an estimated value function $\fhat:\R^d\to \R$ such that
\begin{align*}
    \EE\sups{\alg}\brk*{ \EE_{\x\sim p}(\fhat(\x)-\lr\x,\ths\rr)^2 }\leq\cE(T)^2.
\end{align*} 

(a) If $\alg$ is \aLDP~with $\beta\leq \frac{1}{8T}$, then it must holds that
\begin{align*}
    \cE(T)^2\geq c\min\crl*{ \frac{1}{\lmins}\cdot \frac{1}{\alpha^2T}, \frac{1}{\alpha\sqrt{T}} }.
\end{align*}

(b) If $\alg$ is \yDP~with $\beta\leq\alpha$, then it must holds that
\begin{align*}
    \cE(T)^2\geq c\min\crl*{ \frac{1}{\lmins}\cdot \frac{d}{\alpha^2T^2}, \frac{\sqrt{d}}{\alpha T} }.
\end{align*}

\end{proposition}

In words, the \Ltwo-error of any private regression algorithm must either scale with the minimum-eigenvalue lower bound $\lmins$ or suffer a \emph{slow} rate. 
Translated to regret, the proposition shows that any oracle-based bound of the form $L\cdot T\cdot \cE(T)$ must scale as $\Omega\prn*{\min\crl*{\frac{L}{\alpha\sqrt{\lmins}}, \frac{Ld^{1/4}\sqrt{T}}{\alpha}}}$ under joint DP and $\Omega\prn*{\min\crl*{\frac{L\sqrt{T}}{\alpha\sqrt{\lmins}}, \frac{LT^{3/4}}{\sqrt{\alpha}}}}$ under local DP.
Hence, within square-loss–based frameworks, dependence on $\lmins$—the minimum eigenvalue of the covariance matrix induced by any policy (cf. \eqref{def:lmin}) is typically unavoidable. Avoiding the explorability condition $\lmins>0$ therefore requires new algorithmic or analytical ideas.

\paragraph{Alternative approach: Regression with confidence intervals}%

As an alternative to standard regression-based approaches~\citep{foster2018practical,foster2020beyond}, \citet{li2024optimal} propose an action-elimination framework based on regression with an $\Lone$-error guarantee and additional confidence-interval structure. The key observation is that, while the negative results (\cref{prop:lower-linear-est}) rule out regression oracles with $O(1/T)$ convergence under \Ltwo-error, such oracles are \emph{not} necessary for algorithm design. 
Specifically, \citet{li2024optimal} consider weaker regression subroutines with confidence intervals (\emph{$\Lone$-regression} for short), defined as follows:

\begin{definition}[$\Lone$-regression oracle]\label{def:L1-oracle}
Let $N\geq 1$, $\delta\in(0,1)$.
In contextual bandits, a $\Lone$-regression oracle is an $N$-round algorithm $\alg$ that outputs an estimate of the reward function $\fhat: \cX\times \cA \to [-1,1]$ and a confidence bound $\CI: \cX\times \cA \to \R_{\geq 0}$, such that for any fixed policy $\pi:\cX\to \Delta(\cA)$, given data $(x_t,a_t,r_t)$ generated independently as
\begin{align*}
\textstyle
    x_t\sim P, \quad
    a_t\sim \pi(\cdot|x_t), \quad
    \EE[r_t|x_t,a_t]=\fs(x_t,a_t), \qquad t\in[N]
\end{align*}
the following holds \whp:

(1) (Confidence interval) $\abs{\fhat(x,a)-\fs(x,a)}\leq \CI(x,a)$ for all $(x,a)\in\cX\times\cA$. 

(2) ($\Lone$-performance bound) $\EE_{x\sim P, a\sim \pi(x)} \brac{  \CI(x,a) }\leq \cE_\delta(N)$.

\end{definition}

With a private $\Lone$-regression oracle, \citet{li2024optimal} adopt an algorithm based on action elimination that achieves a regret bound scaling with $\tbO{ |\cA|^2 \cdot T \cdot \cE_{\delta}(T) }$. Therefore, the framework opens the door for a $\sqrt{T}$-regret by developing $\Lone$-regression oracle with $\cE_\delta(T)=\tbOn{1/\sqrt{T}}$. 
However, the $\Lone$-regression oracle of \citet{li2024optimal} is based on iterative private PCA and layered private linear regression, achieving the rate $\cE_{\delta}(T) \leq \tbO{
\frac{\log^d(T)}{\alpha\sqrt{T}}}$ and hence leading to a $\log^d(T)\sqrt{T}$-regret that is exponential of the dimension $d$. This regret bound is useful when the dimension $d=\bigO{1}$ is of constant order. 
In \cref{ssec:cb-algs}, we improve the action elimination framework of \citet{li2024optimal} and adopt our \method~method to provide rate-optimal private regret bounds.

We also note that the prior work \citep{chen2024private} provides a significantly improved regret of $\sqrt{d^3T}/\alpha$ through analyzing the Decision-Estimation Coefficient (DEC). 
However, such a guarantee is again achieved by the Exploration-by-Optimization algorithm,
and it is unknown whether the ideas may lead to a computationally efficient algorithm.

\subsection{Private Action Elimination Algorithm and Guarantees}\label{ssec:cb-algs}

For the simplicity of presentation, in the following, we mainly focused on linear contextual bandits under joint DP (the optimal regret in this setting is an open problem~\citep{azize2024open}). We defer the details of the local DP setting to \cref{appdx:CB}.

Our algorithm (\cref{alg:batch-cb-JDP}) is epoch-based (with a given epoch schedule $1=T\epj[0]<T\epj[1]<\cdots<T\epj[J]=T$), and it iteratively builds estimations of the ground truth reward function $\fs$ and plans according to the estimations. The algorithm, which consists of an estimation procedure and a planning procedure described as follows, is similar to that of \citet{li2024optimal}. 
\colt{The formal definition is deferred to \cref{appdx:spanner}, and we briefly introduce the main ideas here.}
\arxiv{

\begin{algorithm}
\begin{algorithmic}[1]
\REQUIRE Round $T\geq 1$, parameter $\delta\in(0,1)$, epoch schedule $1=T\epj[0]<T\epj[1]<\cdots<T\epj[J]=T$.
\STATE Initialize $\hft[0]\equiv 0, \CIt[0]\equiv 1$.
\FOR{$j=0,1,\cdots,J-1$}
\STATE Set $\pit[j]\leftarrow\AlgPlan(\set{ (\hft,\CIt)  }_{\tau<j})$ and initialize the dataset $\dataset\epj=\sset{}$.
\FOR{$t=T\epj,\cdots,T\epj[j+1]-1$}
\STATE Observe context $x_t\sim P$, select action $a_t\sim \pit[j](x_t)$, and receive reward $r_t$. 
\STATE Update the dataset $\dataset\epj \leftarrow \dataset\epj \cup \sset{ (\phxa[x_t,a_t], r_t) }$.
\ENDFOR
\STATE Set $\hth\epj, (W\epj, \gamma\epj, \lambda\epj)\leftarrow \JDPLinearRegression(\cD\epj,\delta)$.
\STATE Update the estimated reward function and confidence bound as
\begin{align*}
    \hft[j](x,a)=\lr \phxa, \hth\epj \rr, \qquad
    \CIt[j](x,a)=\min\crl*{8\lambda \epj \cdot \nrmn{W\epj \phxa}, 2}, \qquad \forall (x,a)\in\cX\times\cA.
\end{align*}
\ENDFOR
\end{algorithmic}
\caption{Joint DP Action Elimination}\label{alg:batch-cb-JDP}
\end{algorithm}

}

\paragraph{Estimation procedure}
For the $j$-th epoch, the estimated reward function $\hft[j]$ and confidence radius $\CIt[j]$ are produced by the subroutine $\JDPLinearRegression$. By \cref{thm:JDP-linear-regression-full}, the subroutine can be instantiated so that, with high probability,
\begin{align}\label{eq:CI-bound}
\textstyle
    \fs(x,a)\in \brac{ \hft[j](x,a)-\CIt[j](x,a), \hft[j](x,a)+\CIt[j](x,a) }, \qquad \forall x\in\cX, a\in\cA.
\end{align}

\begin{algorithm}
\begin{algorithmic}
\REQUIRE Confidence interval descriptions $(\hft[0], \CIt[0]),\cdots,(\hft[j-1],\CIt[j-1])$ up to the $j$th epoch.
\FOR{context $x\in\cX$}
\STATE Set $\cAxt[0]=\cA$.
\FOR{$\tau=1,\cdots,j-1$}
\STATE Set
\begin{align}\label{eq:eliminate}
    \cAxt=\sset{ a\in\cAxt[\tau-1]: \hft(x,a)+\CIt(x,a)\geq \max_{a'\in\cA}\hft(x,a')-\CIt(x,a') }.
\end{align}
\ENDFOR
\STATE Compute a spanner $\cAsp[j]$ of $\cAxt[j-1]$ (\cref{def:spanner}).
\STATE Set $\pit[j](x)=\Unif(\cAsp[j])$.
\ENDFOR
\ENSURE Policy $\pit[j]$.
\end{algorithmic}
\caption{Subroutine $\AlgPlan$ %
}\label{alg:plan}
\end{algorithm}

\paragraph{Planning procedure}
The policy $\pit[j]$ of the $j$-th epoch is built on the confidence intervals in \cref{eq:CI-bound} from previous epochs, i.e., $(\hft[0], \CIt[0]),\ldots,(\hft[j-1],\CIt[j-1])$. Given these estimates, the subroutine $\AlgPlan$ (\cref{alg:plan}) eliminates suboptimal arms for each context $x\in\cX$ according to \eqref{eq:eliminate}, and outputs $\pi\epj(x)$ based on a \emph{spanner} of the remaining actions.\footnote{Defined in \cref{def:spanner}. In particular, when $|\cA|=\bigO{1}$, we can take $\cAsp[j]=\cAxt[j-1]$ and set $\pi\epj(x)=\Unif(\cAxt[j-1])$.}
This procedure implicitly encourages exploration, as it uses optimistic estimates (UCB) of each arm’s value.

As a remark, in implementing \cref{alg:batch-cb-JDP}, we do \emph{not} need to range over every context $x\in\cX$ to form $\pit[j]$. It suffices to compute $\pit[j](x_t)$ for each round $t$ in the $j$-th epoch (according to \eqref{eq:eliminate}). Hence, \cref{alg:batch-cb-JDP} can be implemented with amortized time complexity $\poly(d,|\cA|)$. Further discussion is deferred to \cref{appdx:spanner}.

\paragraph{Regret guarantee}
We state the upper bound of \cref{alg:batch-cb-JDP} in terms of the dimensionality of the per-context action space: 
\begin{align*}
    \dA\defeq \max_{x\in\cX} \dim\paren{\set{ \phi(x,a): a\in\cA }}.
\end{align*}
Note that $\dA\leq \min\set{d,|\cA|}$. To achieve a tight regret bound, our choice of $(\gamma\epj, \lambda\epj)$ is slightly different from \cref{thm:JDP-linear-regression-full}.
The details are presented in \cref{appdx:proof-regret-upper-JDP-better}.

\newcommand{\cbreglower}{\siga \dA d(\sqrt{d}+\log T)\log T}
\begin{theorem}[Regret upper bound in linear contextual bandits]\label{thm:regret-upper-JDP-better}
\cref{alg:batch-cb-JDP} preserves \aJDP. Suppose that the subroutine $\JDPLinearRegression$ is instantiated following \cref{thm:JDP-linear-regression-full}, then 
\begin{align*}
    \Reg\leqsim \sqrt{\dA^3 dT\log T}+\tbO{\cbreglower},
\end{align*}
where $\tbO{\cdot}$ hides polynomial factors of $\log\log T$.
\end{theorem}
\cref{thm:regret-upper-JDP-better} provides the first DP regret bound of order $\sqrt{T}+\frac{1}{\alpha}$ for linear contextual bandits. Therefore, joint privacy in generalized linear contextual bandits is almost \emph{for free}, as the second term (the ``price of privacy'') grows as $\tbO{1/\alpha}$ and is lower order as $T\to\infty$, addressing several questions raised by \citet{azize2024open}.

Particularly, when $\dA=\tbO{1}$, we obtain an upper bound of $\tbO{\sqrt{dT}+\frac{d^{3/2}}{\alpha}}$, where the leading term matches the optimal non-private regret of $\sqrt{dT}$. In this setting, our upper bound nearly matches the lower bound $\Omega\prn*{\sqrt{dT}+\frac{d}{\alpha}}$~\citep{he2022reduction}. 

In \cref{appdx:proof-JDP-gen-linear}, we present the guarantee of \cref{alg:batch-cb-JDP} on generalized linear contextual bandits, obtaining similar regret bound.

\paragraph{Logarithmic regret with gap}
In addition, under a standard gap assumption between optimal and suboptimal actions, \cref{alg:batch-cb-JDP} achieves logarithmic regret.
\begin{assumption}\label{asmp:gap}
There is a parameter $\gap>0$ such that for any $x\in\cX$, any $a\in\cA$, either $a\in\argmax \fs(x,\cdot)$ (i.e., $a$ is an optimal arm under $x$), or it holds that
\begin{align*}
    \fs(x,\pis(x))\geq \fs(x,a)+\gap.
\end{align*}
\end{assumption}

\begin{theorem}\label{thm:JDP-cb-gap}
Under \cref{asmp:gap}, \cref{alg:batch-cb-JDP} (instantiated as in \cref{thm:regret-upper-JDP-better}) achieves
\begin{align*}
    \Reg\leqsim \frac{\dA^3d\log^2 T}{\gap}+\tbO{\siga \dA d\sqrt{d+\log T}\cdot\log T},
\end{align*}
where $\tbO{\cdot}$ hides polynomial factors of $\log\log T$.
\end{theorem}
Up to $\poly(d,\log\log T,\log(1/\beta))$ factors, the above regret bound scales as $\frac{\log^2(T)}{\gap}+\frac{\log^{3/2}(T)}{\alpha}$, improving upon the result of \citet{shariff2018differentially} that scales as $\Omega\prn*{\frac{\log^2(T)}{\alpha \gap}}$.

\paragraph{Extension to the local DP setting}
By using $\LDPLinearRegression$ (\cref{alg:LDP-linear-regression}) as the estimation subroutine, we can adapt \cref{alg:batch-cb-JDP} so that it preserves local privacy. We state the corresponding regret bound and defer details to \cref{appdx:regret-upper-LDP}.
\begin{theorem}[Regret upper bound under LDP]\label{thm:regret-upper-LDP}
A variant of \cref{alg:batch-cb-JDP} with estimation subroutine $\LDPLinearRegression$ (detailed in \cref{appdx:regret-upper-LDP}) preserves \aLDP~and achieves
\begin{align*}
\textstyle
    \Reg\leq \tbO{\siga \dA\sqrt{d^3T\log T} }.
\end{align*}
\end{theorem}
In general, \cref{thm:regret-upper-LDP} provides a regret bound of $\sqrt{d^5T}/\alpha$. In the worst case ($\dA=d$), this is a factor-$d$ worse than \citet{chen2024private}, but our algorithm is computationally efficient, and the number of switches (or changes) of the deployed decision–channel pair $(\pi_t, \pr_t)$ is bounded by $\bigO{\log T\cdot \log\log T}$.

\arxiv{
\section{Extension: Dimension-free Linear Regression}\label{sec:unbounded}

In \cref{sec:LDP-linear} and \cref{sec:DP-linear}, we characterized near-optimal rates for private linear regression when the dimension $d$ is \emph{bounded}. 
We now turn to the regime where $d$ is prohibitively large or even \emph{unbounded}, as in models parameterized by a reproducing kernel Hilbert space (RKHS).\footnote{Our approach naturally applies to learning in RKHS. However, to avoid measure-theoretic subtleties, we present algorithms only for finite-dimensional spaces.}

This regime is substantially more challenging, as the following lower bounds demonstrate. Proofs appear in \citet[Appendix C]{chen2024private}. 

\newcommand{\Sd}{\mathbb{S}^{d-1}}
\begin{proposition}\label{prop:unbounded-lower}
Let $d\geq 1$. %
For each $\theta\in\Sd$ in the $d$-dimensional unit sphere, we consider the linear model $M_\theta$:
\begin{align*}
    (\x,y)\sim M_\theta: \qquad \x=\theta, y=1.
\end{align*}
Then, there is an absolute constant $c>0$ such that for any parameter $R\in[1,c\sqrt{d}]$, the following holds:

(a) Suppose that $\alg$ is a $T$-round $(\alpha,0)$-DP algorithm with output $\nrmn{\hth}\leq R$. Then it holds that
\begin{align*}
    \sup_{\ths\in\Sd}\EE^{M_{\ths},\alg} \absn{\lr \ths,\hth\rr -1 }\geq c, \qquad \text{unless }T\geq \frac{cd}{R^2\alpha}.
\end{align*}

(b) Suppose that $\alg$ is a $T$-round \aLDP~algorithm with output $\nrmn{\hth}\leq R$. Then it holds that
\begin{align*}
    \sup_{\ths\in\Sd}\EE^{M_{\ths},\alg} \absn{\lr \ths,\hth\rr -1 }\geq c, \qquad \text{unless }T\geq c\min\sset{\frac{d}{R^2\alpha^2}, \frac{1}{\beta}}.
\end{align*}
\end{proposition}

Note that for each linear model $M_\theta$, the covariance matrix $\EE_{\PP_\theta}[\x\x\tp]=\theta\theta\tp$ is rank 1. %
Therefore, \cref{prop:unbounded-lower} has two immediate implications for \emph{pure JDP} and LDP linear regression with unknown covariate distributions:
\begin{itemize}
\item[(1)] Estimating the covariance matrix requires $\Omega(d)$ samples, even when it is known to have rank 1.
\item[(2)] A proper estimator of the parameter $\ths$ also requires $\Omega(d)$ samples to achieve non-trivial error.
Conversely, any non-trivial estimator $\hth$ must have norm $\nrmn{\hth}\geq \Omega(\sqrt{d/T})$ (with non-trivial probability).
\end{itemize}
Therefore, when $d\gg T$, to achieve estimation guarantees with $T$ samples, the estimator must be highly \emph{improper} and, importantly, cannot rely on estimating the covariance matrix.

Motivated by these observations, we design improper private procedures with \emph{dimension-free} guarantees for private linear regression.
We then apply these procedures to obtain dimension-free regret bounds in private linear contextual bandits.

\subsection{Private improper batched SGD}\label{ssec:l1-dim-free}

\newcommand{\pa}[1]{\theta\epk{#1}}
\newcommand{\gd}[1]{g\epk{#1}}
\newcommand{\ogd}[1]{\bar{g}\epk{#1}}
\newcommand{\xt}{\x_t}
\newcommand{\yt}{y_t}

\newcommandx{\nt}[1][1=t]{\zeta_{#1}}
\newcommand{\gt}{g_t}
\newcommand{\tgt}{\Tilde{g}_t}
\newcommandx{\avgtk}[1][1=k]{\frac1N\sum_{t=#1 N+1}^{(#1+1)N}}

We begin with the DP setting. For any non-private estimator $\hth$ with \emph{sensitivity} $s$, it is well-known that the estimator $\til \theta=\hth+\zeta$ ensures \aDP~with noise $\zeta\sim \normal{0,\eps^2\id}$ and parameter $\eps:=s\cdot \siga$. The key idea is that, while $\nrm{\zeta}\asymp \eps\sqrt{d}$, we have $\nrm{\zeta}_{\bSigma}\leqsim \eps$ with high probability, where $\bSigma\defeq \Epp{\x\x\tp}$ is the covariance matrix.
Therefore, to ensure JDP, it is sufficient to privatize a non-private estimator with low sensitivity. 

Based on this observation, we consider the projected stochastic gradient descent (PSGD) on the standard square-loss:
\begin{align*}
    \Lsq(\theta)=\frac12\Exy{\paren{\lr \x,\theta\rr-y}^2}.
\end{align*}
By directly privatizing its average iterate (\cref{alg:JDP-improper-GD}), we can achieve a near-optimal convergence rate (detailed in \cref{appdx:improper-JDP}).

\begin{theorem}[Dimension-free DP regression]\label{thm:improper-JDP}
Let $T\geq1, \delta\in(0,1)$. \cref{alg:JDP-improper-GD} preserves \aDP, and with a suitably chosen parameter $K$, it ensures that \whp,
\begin{align*}
    \Epp{\lr \x,\hth-\ths\rr^2}=\nrmn{\hth-\ths}_{\bSigma}^2\leqsim \sqrt{\frac{\log(1/\delta)}{T}}+\paren{\frac{\siga \sqrt{\log(1/\delta)}}{T}}^{2/3}.
\end{align*}
\end{theorem}

Therefore, \cref{alg:JDP-SGD} achieves the \emph{dimension-independent} convergence rate of $\frac{1}{\sqrt{T}}+\frac{1}{(\alpha T)^{2/3}}$. In non-private linear models, the $T^{-1/2}$-rate of convergence is known to be minimax-optimal (for dimension $d\gg T$) and can be achieved by vanilla stochastic gradient descent.

\paragraph{Local DP}
The situation under the local DP model is much more subtle, as \cref{prop:unbounded-lower} rules out the possiblity of directly privatizing PSGD.
The reason why privatized SGD does not work under LDP is that in this case,
it is typically necessary to add a noise vector $\zeta$ to the gradient that has norm scaling with $\Omega(\sqrt{d})$ to privatize it (e.g., when $\zeta$ is the Gaussian noise). In other words, the privatized gradient estimator has norm scaling with $\sqrt{d}$. Hence, after a single step of gradient descent, the iterate falls outside the unit ball $\Bone$, and projection back to $\Bone$ leads to a large bias. %

Instead, we consider performing privatized batch SGD \emph{without} projection, and we apply a careful clipping on the gradient estimator (\cref{alg:LDP-improper-GD}). This is based on extending the aforementioned observation: When $\zeta\sim \normal{0,\eps^2\id}$, while $\EE\nrmn{\zeta}^2=d\eps^2$, for the covariate $\x\sim p$ that is independent of $\zeta$, the random variable $\lr \x,\zeta\rr|\x\sim \normal{0,\eps^2\nrm{\x}^2}$ is a zero-mean Gaussian random variable conditional on $\x$. Particularly, it holds that $\abs{\lr \x, \zeta\rr}\leqsim \eps\nrm{\x}$ with high probability (with respect to the randomness of the noise $\zeta$ and $\x\sim p$). Using this idea, we can show that \eqref{eq:clip-grad} provides an estimator of $\nabla \Lsq(\theta\kk)$ with small bias for all epochs.

\begin{algorithm}
\caption{Locally Private Clipped SGD}\label{alg:LDP-improper-GD}
\begin{algorithmic}[1]
\REQUIRE Dataset $\dataset=\crl*{(\x_t,y_t)}_{t\in[T]}$.
\REQUIRE Epoch $K\geq 1$, batch size $N=\floor{\frac{T}{K}}$, stepsize $\eta=1$, parameter $R=2$.
\STATE Initialize $\theta\kz=\bz$.
\FOR{$k=0,\cdots,K-1$}
    \FOR{$t=\rangekn$}
        \STATE Compute the gradient estimator
        \begin{align}\label{eq:clip-grad}
            g_t=\xt\paren{ \clip{\lr \theta\kk,\xt \rr}-\yt }
        \end{align}
        \STATE Privatize $\til g_t\sim \priv[R+1]{g_t}$.    
    \ENDFOR
    \STATE Compute the batched gradient estimator $\til g\kk=\avgtk \tgt$ and update
\begin{align*}
    \theta\kp=\theta \kk-\eta \til g\kk.
\end{align*}
\ENDFOR
\ENSURE $\hth=\theta\kc$.
\end{algorithmic}
\end{algorithm}

With a careful analysis that bounds the bias introduced by the clipping operation~\cref{eq:clip-grad}, we obtain the following guarantee for \cref{alg:LDP-improper-GD}.

\begin{theorem}[Dimension-free LDP regression]\label{thm:improper-LDP}
\cref{alg:LDP-improper-GD} preserves \aLDP. For $T\geq 1, \delta\in(0,1)$, with a suitable number of epochs $K\geq 1$, \cref{alg:LDP-improper-GD} returns $\hth$ that \whp,
\begin{align*}
    \Epp{ \lr \x, \hth-\ths \rr^2 }=\nrmn{\hth-\ths}_{\bSigma}^2\leqsim \paren{\frac{\siga\log(T/\delta)}{T}}^{1/3}.
\end{align*}
\end{theorem}

This establishes a convergence rate of $\tbO{(\alpha^2T)^{-1/3}}$ under the square loss. As shown by \citet[Corollary I.8]{chen2024private}, any LDP algorithm must incur an $\Lone$-error of
\begin{align*}
    \Omega\paren{\min\sset{\frac{d}{\sqrt{\alpha^2 T}},\paren{\frac{1}{\alpha^2 T}}^{1/6}}}.
\end{align*}
Therefore, in the dimension-free setting, \cref{alg:LDP-improper-GD} achieves the minimax-optimal \emph{dimension-free} convergence rate under $\Lone$-error, and hence is also minimax-optimal under \Ltwo-error. Further, \cref{thm:LDP-linear-regression-full} and \cref{thm:improper-LDP} together provide near-optimal $\Lone$-estimation error for the full range $T\in[1,\infty)$ under LDP.

\subsection{Application: Linear contextual bandits with dimension-free regret bound} 

As an application, we use the dimension-free procedures developed in \cref{ssec:l1-dim-free} as subroutines for learning linear contextual bandits.
We invoke the $\AlgSQCB$ algorithm \citep{abe1999associative,foster2020beyond,simchi2020bypassing}, which has a regret guarantee given any \emph{offline regression oracle} with an \Ltwo-error bound. By instantiating the regression oracle as \cref{alg:JDP-improper-GD} or \cref{alg:LDP-improper-GD}, we obtain the following private regret bounds. Details are presented in \cref{appdx:square-cb}.

\begin{theorem}[Dimension-free regret bounds]\label{thm:regret-dim-free}
Let $T\geq 1$ and suppose $\cA$ is finite.

(1) Suppose that $\AlgSQCB$ (\cref{alg:square-cb}) is instantiated with the regression subroutine \cref{alg:JDP-improper-GD}. Then $\AlgSQCB$ preserves \aJDP~and it ensures 
\begin{align*}
    \Reg\leq \sqrt{|\cA|}\cdot \tbO{T^{3/4}+\siga^{1/3}T^{2/3}}.
\end{align*}

(2) Suppose that $\AlgSQCB$ (\cref{alg:square-cb}) is instantiated with the regression subroutine \cref{alg:LDP-improper-GD}. Then $\AlgSQCB$ preserves \aLDP~and it ensures 
\begin{align*}
    \Reg\leq \sqrt{|\cA|}\cdot \tbO{\siga^{1/3}T^{5/6}}.
\end{align*}
\end{theorem}

To the best of our knowledge, such dimension-free regret bounds are new in private contextual bandits. Under JDP, the regret rate is $\tbO{T^{3/4}+\alpha^{-1/3}T^{2/3}}$, and the first term matches the optimal dimension-free $T^{3/4}$ rate in non-private linear contextual bandits~\citep{abe1999associative,foster2020beyond}, implying that privacy is almost ``for free'' in this setting. Furthermore, the LDP regret bound scales as $\alpha^{-1/3}T^{5/6}$, which nearly matches the minimax lower bound \citep[Corollary I.15]{chen2024private}.

}

\section*{Acknowledgements} 
We thank Adam Smith and Achraf Azize for helpful discussions.
FC and AR acknowledge support from ARO through award W911NF-21-1-0328, as well as Simons Foundation and the NSF through awards DMS-2031883 and PHY-2019786. 

\bibliographystyle{abbrvnat}
\bibliography{ref.bib}

\newpage
\appendix

\tableofcontents

\section{Additional Discussion}\label{appdx:motivation}
\subsection{Interpretation of the \infoms: an information-theoretic perspective}\label{appdx:lower-overview}

For lower bounds, we consider linear models with a fixed covariate distribution $p$ and label $y\in\set{-1,1}$, i.e., for any fixed covariate distribution $p$ and $\theta\in\Theta_p\ldef\crl*{\theta\in\Bone: \abs{\lr \x,\theta\rr}\leq \frac12\forall \x\in \supp(p)}$, we consider the linear model $M_{p,\theta}$ defined as
\begin{align*}
    (\x,y)\sim M_{p,\theta}: \quad \x\sim p, ~~y\mid \x\sim \Rad{\lr \x, \theta\rr},
\end{align*}
where $\Rad{\mu}$ is the distribution over $\crl{-1,1}$ with $\PP(y=1)=1-\PP(y=-1)=\frac{1+\mu}{2}$. 

\paragraph{Intuitions for the lower bounds}
Our lower bounds use standard information-theoretic arguments that reduce estimation to two-point hypothesis testing. 
As we detail below, in our argument, the \infoms~$\Ustar$ and $\Wstar$ play a role similar to the Fisher information matrix (see also \cref{tab:cmp-infoms}).

\begin{table}[h]
\centering
\caption{Information matrices are analogous to the Fisher information matrix, in the sense that they provide upper bounds on the corresponding divergences.}\label{tab:cmp-infoms}
\small
\renewcommand{\arraystretch}{1.5}
\begin{tabularx}{\textwidth}{>{\centering\arraybackslash}p{2cm}|>{\centering\arraybackslash}p{5cm}|>{\centering\arraybackslash}p{4.3cm}|>{\centering\arraybackslash}p{3.5cm}}
\hline
\textbf{Setting }
& \textbf{Threshold for indistinguishable $\theta$ and $\theta+\Delta$} 
& \textbf{Directional Information} 
& \textbf{Lower bound for directional risk} \\
\hline

\textbf{Local DP} 
& $\bigl(\alpha \, \EE[|\phi^\top \Delta|]\bigr)^2 \ll \tfrac{1}{T}$ 
& $\EE[|\phi^\top v|] \asymp \|\Ustar^{-1} v\|$ 
& $\EE[\langle \hat{\theta}-\theta, v \rangle^2] 
   \gtrsim \tfrac{1}{\alpha^2 T}\|\Ustar v\|^2$ \\
\hline

\textbf{Central DP} 
& $\EE\!\left[\min\{(\phi^\top \Delta)^2,\; \alpha|\phi^\top \Delta|\}\right]\ll \tfrac{1}{T}$ 
& $\EE\!\left[\min\{(\phi^\top v)^2,\; \alpha|\phi^\top v|\}\right] 
   \lesssim \|\Wstar^{-1} v\|^2 + \alpha\gamma\|\Wstar^{-1} v\|$ 
& $\EE[\langle \hat{\theta}-\theta, v \rangle^2] 
   \gtrsim \tfrac{1}{T}\|\Wstar v\|^2$ \\
\hline

\textbf{No Privacy} 
& $\EE[(\phi^\top \Delta)^2] \ll \tfrac{1}{T}$ 
& $\EE[(\phi^\top v)^2] = \|v\|_{\cov}^2$ 
& $\EE[\langle \hat{\theta}-\theta, v \rangle^2] 
   \gtrsim \tfrac{1}{T}\|\cov^{-1/2} v\|^2$ \\
\hline
\end{tabularx}
\end{table}

We start with the local DP lower bounds. It is well-known that two models $M$ and $M'$ are indistinguishable by LDP algorithms with $T$ samples if $\DTV{M,M'}\ll \frac{1}{\alpha\sqrt{T}}$~\citep{duchi2013local}. On the other hand, for linear models, we have shown that $\Ustar$ characterizes the TV distance between models $M_{p,\theta}$ and $M_{p,\theta+v\Delta}$ (\cref{lem:U-L1-covariance}):
\begin{align*}
    \nrm{\Ustar\iv v}
    \leq \DTV{M_{p,\theta},M_{p,\theta+v}}+\lambda\nrm{v}\leq (\sqrt{d}+1)\nrm{\Ustar\iv v}.
\end{align*}
Therefore, for any unit vector $e$, the model $M_{p,0}$ and $M_{p,\delta \Ustar e}$ are indistinguishable for $\abs{\delta}\leqsim \frac{1}{\alpha\sqrt{dT}}$. This implies \cref{lem:LDP-lower-bound-any} immediately.

A similar argument applies to the central DP setting. From general characterizations of DP hypothesis testing~\citep{canonne2019structure,asi2024universally}, we can derive the following divergence between two linear models under the central privacy model:
\begin{align*}
    D_\alpha(M_{p,\theta},M_{p,\theta+v})\defeq \EE_{\x\sim p} \min\crl*{\abs{\lr \x, v\rr}^2, \alpha \abs{\lr \x, v\rr}}.
\end{align*}
As shown by \citet{canonne2019structure}, if $D_\alpha(M,M')\ll \frac{1}{T}$, then $M$ and $M'$ are statistically \emph{indistinguishable} under $\alpha$-DP algorithm with $T$ samples.\footnote{Conversely, if $D_\alpha(M,M')\gg \frac{1}{T}$, there is an $\alpha$-DP test that distinguishes $M$ and $M'$ with high probability using $T$ samples~\citep{canonne2019structure}.}
To prove the lower bounds of \cref{lem:JDP-lower-bound-any}, we can then upper bound the divergence $D$ using the \infom~$\Wstar$ (\cref{lem:E-nrm-W-bound}):
\begin{align*}
    D_\alpha(M_{p,\theta},M_{p,\theta+v})=\Ex{\min\crl*{\abs{\lr \x, v\rr}^2, \alpha\abs{\lr \x, v\rr}}}
    \leq \nrm{(\Wstar)\iv v}^2+\alpha\sqrt{d}\gamma\nrm{(\Wstar)\iv v}.
\end{align*}

Formally, we state the data-processing inequalities for private algorithms as follows.

\newcommand{\Divr}[1][\alpha]{D_{\alpha}}

\paragraph{Data-processing inequality for local DP algorithms}
We utilize the following lemma for our LDP lower bounds.
\begin{lemma}\label{lem:LDP-chain-rule}
Fix $\alpha,\beta\in[0,1]$. Suppose that $\alg$ is a $T$-round \aLDP~algorithm for linear models. For any linear model $M_{p,\theta}$, let $\PP\sups{(p,\theta),\alg}$ be the law of $(o_1,\cdots,o_T,\pi)$, where $\pi$ is the output of $\alg$ after $T$ rounds of interaction. Then, for any two parameters $\theta, \otheta\in\Theta_p$, we have
\begin{align*}
    \DTV{\PP\sups{(p,\theta),\alg},\PP\sups{(p,\otheta),\alg}}\leq&~ 4\alpha\sqrt{T}\cdot \DTV{M_{p,\theta},M_{p,\otheta}} +2\beta T \\
    =&~ 2\alpha\sqrt{T}\cdot \EE_{\x\sim p}\abs{\lr \x, \theta-\otheta\rr } +2\beta T.
\end{align*}
\end{lemma}
\cref{lem:LDP-chain-rule} is a direct corollary of the strong data-processing inequality~\citep{duchi2013local} and Lemma 25 of \citep{duchi2019lower} (see also Lemma B.2 of \citet{chen2024private}).

\paragraph{DP algorithms}
In our DP lower bound arguments (cf. \cref{appdx:proof-W-lower}), we utilize the following result.
\begin{proposition}\label{lem:DP-chain-rule}
    Fix an \yDP~algorithm $\alg: \cZ^T\to \DPi$. For any linear model $M_\theta$, we let $\PP\sups{(p,\theta),\alg}$ be the distribution of $(z_1,\cdots,z_T,\pi)$ under i.i.d $z_1,\cdots,z_T\sim M$ and $\pi\sim \alg(z_1,\cdots,z_T)$. 
    Then for any $\theta, \otheta \in \Theta_p$, it holds that
    \begin{align*}
        \DTV{\PP\sups{(p,\theta),\alg}(\pi=\cdot),\PP\sups{(p,\otheta),\alg}(\pi=\cdot) } \leq \sqrt{T\eps}+2T\eps+\prn*{e^{T\eps}-1}\frac{\beta}{\alpha},
    \end{align*}
    where we denote
    \begin{align*}
        \eps\ldef \EE_{\x\sim p} \min\crl*{\abs{\lr \x, \theta-\otheta\rr}^2, \alpha \abs{\lr \x, \theta-\otheta\rr}}.
    \end{align*}
    In particular, when $\beta\leq \alpha$ and $\eps\leq \frac{1}{16T}$, it holds that for any event $E\subseteq \Pi$,
    \begin{align*}
        \PP\sups{(p,\theta),\alg}(\pi\in E)+\PP\sups{(p,\otheta),\alg}(\pi\not\in E)\geq \frac12.
    \end{align*}
\end{proposition}
For completeness, we present the proof of \cref{lem:DP-chain-rule} in \cref{appdx:proof-DP-chain-rule}.

\subsection{Motivation of the \infoms~from the perspective of upper bounds}\label{appdx:U-explained}

\eqref{def:U-demo} was first introduced by two of the authors when analyzing the Decision-Estimation Coefficient (DEC) of locally private linear models and linear contextual bandits~\citep{chen2024private}. In the following, we illustrate the high-level ideas of why the \infom~$\Ustar$ (and also $\Wstar$) is useful for proving upper bounds.

\paragraph{Failure of privatizing OLS naively}
We first recall that to privatize OLS in the local privacy model, a natural way is to consider solving the following equation privately (cf. \cref{ssec:OLS}):
\begin{align*}
    \prn*{\EE[\x\x\tp]+\lambda\id}\cdot \theta=\EE[\x y].
\end{align*}
Given $T$ samples, by statistically perturbing $(\x_t\x_t\tp,\x_t y_t)$ for each $t\in[T]$, we can obtain the $\alpha$-LDP perturbed statistics $(\wh\cov, \wh\psi)$ such that $\nrmn{\wh\cov-\cov}\leq \eps$, $\nrmn{\wh\psi-\EE[\x y]}\leq \eps$, where $\eps\leqsim_d \frac{1}{\alpha\sqrt{T}}$. However, by solving $\hth=\argmin_{\theta\in\Bone}\nrmn{(\wh\cov+\lambda\id)\theta-\wh\psi}$, it can only be guaranteed that
\begin{align*}
    \nrmn{(\cov+\lambda\id)(\hth-\ths)}\leqsim \eps+\lambda.
\end{align*}
Hence, through this analysis, 
we can at best guarantee
the convergence rate $\nrmn{\hth-\ths}_{\cov}\leqsim_d \frac{1}{\sqrt{\lmin(\cov)}}\cdot \frac{1}{\alpha\sqrt{T}}$ of the privatized OLS.
Note that the dependence on $\lmin(\cov)$ is generally unavoidable for approaches based on perturbing $\x_t y_t$, as shown by \cref{prop:OLS-fail-local-gen}. 

\paragraph{An attempt to improve OLS}
To improve upon OLS, consider solving the following equation privately:
\begin{align}
    \prn*{\EE[F(\x)\x\tp]+\lambda\id}\cdot \theta=\EE[F(\x) y],
\end{align} 
where $F:\Bone\to\Bone$ is a (possibly non-linear) transformation. Then, given $T$ samples, by perturbing the statistics $(F(\x_t)\x_t\tp, F(\x_t)y_t)$ for each $t\in[T]$, we obtain an $\alpha$-LDP estimator $\hth$ such that
\begin{align*}
    \nrmn{(\EE[F(\x)\x\tp]+\lambda\id)(\hth-\ths)}\leqsim \eps+\lambda.
\end{align*}
To choose the map $F$, we aim for $\hth$ to admit a reasonable upper bound on the \emph{\Lone-error} $\EE_{\x\sim p}\absn{\lr \x,\hth-\ths\rr}$. It is sufficient to ensure that $\EE_{\x\sim p}\nrm{U_F\x}$ is bounded, where $U_F\iv\ldef \EE[F(\x)\x\tp]+\lambda\id$. 
The key idea of ~\citet{chen2024private} is to find a transformation $F$ such that $F(\x)=\frac{U_F\x}{\nrm{U_F\x}}$ (which automatically ensures $\EE\nrm{U_F\x}\leq d$). This is equivalent to solving \eqref{def:U-demo}, and a solution exists by Brouwer's fixed-point theorem. 

To summarize, we can construct an LDP estimator with \Lone-error guarantee using a solution to \eqref{def:U-demo} (which requires knowledge of the covariate distribution $p$, cf. \cref{appdx:proof-LDP-minimax-optimal}). Analogously, we can construct a DP estimator using a solution to \eqref{def:W-demo} (cf. \cref{appdx:proof-JDP-minimax-optimal}).

The solution to \eqref{def:U-demo} is also sufficient for proving upper bounds for the Exploration-by-Optimization (ExO) algorithm in private linear models %
and linear contextual bandits~\citep{chen2024private}. 
However, ExO requires solving a high-dimensional min–max optimization problem each round and only provides an \Lone-error guarantee for LDP linear regression. 
These limitations motivate the developments in this paper.

\subsection{Key properties of the \infoms}\label{sec:properties-U-W}

In this section, we present several key properties of the \infoms. The results are based on the fact that the operators $\FLDP$ and $\FJDP$ are ``monotone''.
\begin{lemma}\label{lem:F-monotone}
Fix $\gamma\geq 0$, $\lambda>0$. For any $U, V\succ 0$, it holds that
\begin{align*}
    U\iv \FLDP(U) U\iv \succeq&~ \frac{1}{\nrmop{UV\iv}}\cdot V\iv \FLDP(V) V\iv, \\
    U\iv \FJDP(U) U\iv \succeq&~ \frac{1}{\max\crl{\nrmop{UV\iv},1}}\cdot V\iv \FJDP(V) V\iv.
\end{align*}
Hence, for any $C>0$, it holds that %
\begin{align*}
    \FLDP(U)\preceq C\FLDP(V)\quad\Rightarrow\quad \nrmop{UV\iv}\leq C, \quad\Rightarrow\quad U\preceq CV,
\end{align*}
and
\begin{align*}
    \FJDP(U)\preceq C\FJDP(V)\quad\Rightarrow\quad \nrmop{UV\iv}\leq \max\crl{C, \sqrt{C}} \quad\Rightarrow\quad  U\preceq \max\crl{C,\sqrt{C}}\cdot V.
\end{align*}
In particular, both $\FLDP$ and $\FJDP$ are injective.
\end{lemma}

\paragraph{Fast convergence of the exact spectral iterations}
The spectral iteration~\cref{eqn:U-spectral-exact} can be rewritten compactly as
\begin{align*}
    F\kk=\FLDP(U\kk), \qquad U\kp=\sym(F\kk\isq U\kk).
\end{align*}
The iteration~\cref{eqn:W-spectral-exact} can be rewritten similarly.
Therefore, applying \cref{lem:F-monotone}, we establish the convergence rate of \eqref{eqn:U-spectral-exact} and \eqref{eqn:W-spectral-exact}.
\begin{lemma}\label{lem:spec-fast-converge}
Suppose that the sequence $(U\kk,F\kk)$ is generated by \eqref{eqn:U-spectral-exact}. Then it holds that
\begin{align*}
    \lmin(\covF\kk)\sq\leq \lmin(\covF\kp)\leq \lmax(\covF\kp)\leq \lmax(\covF\kk)\sq.
\end{align*}
In particular, with $U\kz=\id$, it holds that $\exp\prn*{-\frac{\log(1/\lambda)}{2^k}}\id\preceq \FLDP(U\kk)\preceq \exp\prn*{\frac{\lambda}{2^k}}\id$.

Similarly, for the sequence $(W\kk,F\kk)$ that is generated by \eqref{eqn:W-spectral-exact}, we have
\begin{align*}
    \min\crl*{ \lmin(\covF\kk)\sq , 1} \leq \lmin(\covF\kp)\leq \lmax(\covF\kp)\leq \max\crl*{ \lmax(\covF\kk)\sq, 1}.
\end{align*}
\end{lemma}
With \cref{lem:F-monotone} and \cref{lem:spec-fast-converge}, the following facts are straightforward:
\begin{itemize}
    \item For any $\lambda>0$, there is a unique solution $\Ustar$ to the equation $\FLDP(U)=\id$. 
    \item For any $\lambda>0$ and $\gamma\geq0$, there is a unique solution $\Wstar$ to the equation $\FJDP(W)=\id$. 
\end{itemize}
Further, it is also clear that both $\lambda\mapsto \Ustar$ and $(\gamma,\lambda)\mapsto \Wstar$ are continuous. 

\paragraph{Properties of the matrix $\Wstar$}
As an analogue to \cref{lem:U-L1-covariance}, we show that $\Wstar$ provides upper bounds on the \Lone-error and also the $D_\alpha$ divergence (cf. \cref{appdx:lower-overview}).
\begin{lemma}\label{lem:E-nrm-W-bound}
It holds that
\begin{align*}
    \Ex{\abs{\lr \x, \theta\rr}}\leq (1+\sqrt{d}\gamma)\nrm{(\Wstar)\iv \theta},
\end{align*}
and
\begin{align*}
    \Ex{\min\crl*{\abs{\lr \x, \theta\rr}^2, \alpha\abs{\lr \x, \theta\rr}}}
    \leq \nrm{(\Wstar)\iv \theta}^2+\alpha\sqrt{d}\gamma\nrm{(\Wstar)\iv \theta}.
\end{align*}
Furthermore, for any $W$ such that $\FJDP(W)\preceq 2\id$, it holds that
\begin{align*}
    \Ex{\nrm{W \x}}\leq \sqrt{2d}+d\gamma,
\end{align*}
and for any $G>0$, it holds that
\begin{align*}
    \Ex{\nrm{W\x}\indic{\nrm{W\x}>G}}\leq 2d\prn*{G^{-1}+\gamma}.
\end{align*}
\end{lemma}

\subsection{\Infoms~of light-tailed distributions}\label{appdx:proof-light-tail}

We have the following general bound. 

\begin{lemma}[Light-tailed distributions]\label{lem:kappa-p}
For any $c\in[0,1]$, we define
\begin{align}
    \kappa_c(p)\ldef \inf\crl*{M: \EE_{\x\sim p}\brk*{ \x\x\tp \indic{\nrm{\x}_{\cov^\dagger} \leq M}}\succeq c\cov}.
\end{align}
In particular, $\kappa(p)=\kappa_{1/2}(p)$.
Then it holds that
\begin{align*}
    \nrmop{(\cov+\lambda^2\id)\sq \Wstar}\leq \frac{1}{\sqrt{c}}+\frac{\gamma\kappa_c(p)}{c}, \qquad
    \nrmop{(\cov+\lambda\id)\sq \Ustar}\leq \frac{\kappa_c(p)}{c}.
\end{align*}
\end{lemma}

As an application, note that for a $K$-sub-Gaussian distribution $p$, it holds that $\kappa(p)\leq O(K\sqrt{d})$. Furthermore, for any real number $s>0$, let $\moment{s}\defeq \Ex{\nrm{\x}_{\cov^\dagger}^s}$. Then $\kappa(p)\leq \prn*{2\moment{s}}^{\frac{1}{s-2}}$ for any $s>2$. Therefore, when $p$ is sub-Gaussian or hypercontractive, $\kappa(p)$ can be well controlled.

\subsection{Proofs from \cref{sec:properties-U-W}}

\begin{proof}[\pfref{lem:F-monotone}]
For any $\x\in\Rd$, we have $\nrm{U\x}=\nrm{UV\iv Vx}\leq \nrmop{UV\iv}\nrm{V\x}$. Furthermore, by \cref{lem:matrix-monotone}, we also have
\begin{align*}
    U\iv\succeq \frac{1}{\nrmop{UV\iv}} V\iv.
\end{align*}
Therefore,
\begin{align*}
    U\iv \FLDP(U) U\iv
    =&~\Ex{\frac{\x\x\tp}{\nrm{U\x}}}+\lambda U\iv \\
    \succeq&~ \frac{1}{\nrmop{UV\iv}}\prn*{\Ex{\frac{\x\x\tp}{\nrm{V\x}}}+\lambda V\iv}
    =\frac{1}{\nrmop{UV\iv}}\cdot V\iv \FLDP(V) V\iv.
\end{align*}
Similarly, for $\FJDP$, we use
\begin{align*}
    1+\gamma\nrm{U\x}\leq 1+\gamma\nrmop{UV\iv}\nrm{V\x}\leq \max\crl{\nrmop{UV\iv},1}(1+\gamma\nrm{V\x}),
\end{align*}
which implies
\begin{align*}
    U\iv \FJDP(U) U\iv
    =&~\Ex{\frac{\x\x\tp}{1+\gamma\nrm{U\x}}}+\lambda U\iv \\
    \succeq&~ \frac{1}{\max\crl{\nrmop{UV\iv},1}}\prn*{\Ex{\frac{\x\x\tp}{1+\gamma\nrm{V\x}}}+\lambda V\iv} \\
    =&~ \frac{1}{\max\crl{\nrmop{UV\iv},1}}\cdot V\iv \FJDP(V) V\iv.
\end{align*}
To prove the second statement, we pick a unit vector $v$ such that $UV\iv v=\nrmop{UV\iv} v$, and then
\begin{align*}
    v\tp \FLDP(U) v\geq \frac{1}{\nrmop{UV\iv}}\cdot v\tp U V\iv \FLDP(V) V\iv U v
    =\nrmop{UV\iv} v\tp \FLDP(V) v.
\end{align*}
Under the condition that $\FLDP(U)\preceq C\FLDP(V)$, we have (using $v\tp \FLDP(V) v>0$)
\begin{align*}
    \nrmop{UV\iv}\leq \frac{v\tp \FLDP(U) v}{v\tp \FLDP(V) v}\leq C.
\end{align*}
Similarly, under the condition $\FJDP(U)\preceq C\FJDP(V)$, we also have $\min\crl{\nrmop{UV\iv}^2,\nrmop{UV\iv}}\leq C$, which implies $\nrmop{UV\iv}\leq \max\crl{C,\sqrt{C}}$ and $U\preceq \max\crl{C,\sqrt{C}}\cdot V$ (\cref{lem:matrix-monotone}).
This gives the desired statement.
\end{proof}

\begin{proof}[\pfref{lem:spec-fast-converge}]
We analyze \cref{eqn:U-spectral-exact} through the lens of \cref{lem:F-monotone}. Using $U\kp^2=U\kk F\kk\iv U\kk$, we apply \cref{lem:F-monotone} with $U=U\kk$ and $V=U\kp$:
\begin{align*}
    \frac{1}{\nrmop{U\kk U\kp\iv}} \cdot \FLDP(U\kp)\preceq  U\kp U\kk\iv F\kk U\kp\iv U\kp=  U\kp (U\kp^2)\iv U\kp = \id.
\end{align*}
Similarly, we apply \cref{lem:F-monotone} with $U=U\kp$ and $V=U\kk$:
\begin{align*}
    F\kp=\FLDP(U\kp)\succeq \frac{1}{\nrmop{U\kp U\kk\iv}}\cdot  U\kp U\kk\iv F\kk U\kk\iv U\kp=\frac{1}{\nrmop{U\kp U\kk\iv}}\cdot \id.
\end{align*}
Note that
\begin{align*}
    &~\nrmop{U\kk U\kp\iv}=\nrmop{U\kp\iv U\kk}=\sqrt{\lmax(U\kk U\kp\iq U\kk )}=\sqrt{\lmax(F\kk)}, \\
    &~\nrmop{U\kp U\kk\iv}=\sqrt{\lmax(U\kk\iv U\kp^2 U\kk\iv)}=\sqrt{\lmax(F\kk\iv)}=\frac{1}{\sqrt{\lmin(F\kk)}}.
\end{align*}
This immediately implies the first equation. The convergence rate then follows from $\lambda\id\preceq F\kz \preceq (1+\lambda)\id$.

The proof of convergence for \eqref{eqn:W-spectral-exact} is analogous and hence omitted.
\end{proof}

\begin{proof}[\pfref{lem:kappa-p}]
Consider $W=w(\cov+\lambda^2\id)\isq$ and $w=\frac{1}{\sqrt{c}}+\frac{\gamma\kappa_c(p)}{c}$. Then it holds that
\begin{align*}
    W\iv \FJDP(W)W\iv -\lambda W\iv =&~ \Ex{\frac{\x\x\tp}{1+\gamma\nrm{W\x}}}
    \succeq \Ex{\frac{\x\x\tp}{1+\gamma w \nrm{\x}_{\cov^\dagger}}} \\
    \succeq&~ \frac{1}{1+\gamma w \kappa_c(p)}\Ex{\x\x\tp \indic{\nrm{\x}_{\cov^\dagger} \leq M}} \succeq \frac{c}{1+\gamma w \kappa_c(p)} \cov.
\end{align*}
Hence, using $\cov=W\iq -\lambda \id$ and $W\iv\succeq \frac{\lambda}{w}\id$, we have
\begin{align*}
    \FJDP(W)\succeq \frac{c}{1+\gamma w \kappa_c(p)} W\cov W + \lambda W \succeq \id
\end{align*}
and thus $\nrmop{\Wstar W\iv}\leq 1$, or equivalently,
\begin{align*}
    \nrmop{(\cov+\lambda^2\id)\sq \Wstar}\leq \frac{1}{\sqrt{c}}+\frac{\gamma\kappa_c(p)}{c}.
\end{align*}
Similarly, we consider $U=u(\cov+\lambda \id)\isq$. Then by definition,
\begin{align*}
    U\iv \FLDP(U)U\iv -\lambda U\iv =&~ \Ex{\frac{\x\x\tp}{\nrm{U\x}}}
    \succeq \frac{1}{U}\Ex{\frac{\x\x\tp}{\nrm{\x}_{\cov^\dagger}}} \\
    \succeq&~ \frac{1}{u \kappa_c(p)}\Ex{\x\x\tp \indic{\nrm{\x}_{\cov^\dagger} \leq \kappa_c(p) }} \succeq \frac{c}{u \kappa_c(p)} \cov.
\end{align*}
Hence, we have $\FLDP(U)\succeq \frac{c}{u \kappa_c(p)} U\cov U + \lambda U \succeq \id$, and by \cref{lem:F-monotone}, $\nrmop{(\cov+\lambda\id)\sq \Ustar}\leq u$. This is the desired result.
\end{proof}

\paragraph{Proof of \cref{lem:E-nrm-W-bound}}
We note that $\FJDP(\Wstar)=\id$, which implies
\begin{align*}
    \Ex{\indic{\nrm{\Wstar \x}\leq \gamma^{-1}} \Wstar \x\x\tp \Wstar}\preceq \id, \qquad
    \Ex{\indic{\nrm{\Wstar \x}\geq \gamma^{-1}} \frac{\Wstar \x\x\tp \Wstar}{\nrmn{\Wstar\x}}}\preceq \gamma\id.
\end{align*}
Therefore, for $v=(\Wstar)\iv \theta$, we have
\begin{align*}
    \Ex{\indic{\nrm{\Wstar \x}\leq \gamma^{-1}}\lr \Wstar \x, v\rr^2}
    =v\tp\Ex{\indic{\nrm{\Wstar \x}\leq \gamma^{-1}} \Wstar \x\x\tp \Wstar} v \leq \nrm{v}^2.
\end{align*}
Similarly, using $\Ex{\indic{\nrmn{\Wstar \x}\geq \gamma^{-1}}\nrmn{\Wstar\x}}\leq \tr(\gamma\id)\leq \gamma d$, we have
\begin{align*}
    &~\Ex{\indic{\nrm{\Wstar \x}\geq \gamma^{-1}}\abs{\lr \Wstar\x, v\rr}} \\
    \leq&~ \sqrt{\Ex{\indic{\nrmn{\Wstar \x}\geq \gamma^{-1}}\nrmn{\Wstar\x}}\cdot \Ex{\indic{\nrmn{\Wstar \x}\geq \gamma^{-1}}\frac{\lr \Wstar\x, v\rr^2}{\nrmn{\Wstar \x}}}} \\
    \leq&~ \sqrt{\gamma d\cdot \gamma\nrm{v}^2 }=\sqrt{d}\gamma \nrm{v}.
\end{align*}
Therefore, we have
\begin{align*}
    \Ex{\abs{\lr \x, \theta\rr}}
    =&~ \Ex{\abs{\lr \Wstar\x, v\rr}} \\
    \leq &~ \Ex{\indic{\nrm{\Wstar \x}\leq \gamma^{-1}}\abs{\lr \x, \theta\rr}}+\Ex{\indic{\nrm{\Wstar \x}\geq \gamma^{-1}}\abs{\lr \x, \theta\rr}} \\
    \leq&~ \nrm{v}+\sqrt{d}\gamma\nrm{v}
    = (1+\sqrt{d}\gamma)\nrm{(\Wstar)\iv \theta}
\end{align*}
Similarly,
\begin{align*}
    \Ex{\min\crl*{\abs{\lr \x, \theta\rr}^2, \alpha\abs{\lr \x, \theta\rr}}}
    \leq &~ \Ex{\indic{\nrm{\Wstar \x}\leq \gamma^{-1}}\abs{\lr \x, \theta\rr}^2}+\alpha\Ex{\indic{\nrm{\Wstar \x}\geq \gamma^{-1}}\abs{\lr \x, \theta\rr}} \\
    \leq&~ \nrm{v}^2+\alpha\sqrt{d}\gamma\nrm{v}
    = \nrm{(\Wstar)\iv \theta}^2+\alpha\sqrt{d}\gamma\nrm{(\Wstar)\iv \theta}
\end{align*}

Next, for any $W$ such that $\FJDP(W)\preceq 2\id$, we can similarly bound
\begin{align*}
    \Ex{\indic{\nrm{W \x}\leq \gamma^{-1}} W \x\x\tp W}\preceq \id, \qquad
    \Ex{\indic{\nrm{W \x}\geq \gamma^{-1}} \frac{W \x\x\tp W}{\nrmn{W\x}}}\preceq 2\gamma\id,
\end{align*}
and hence by taking trace,
\begin{align*}
    \Ex{\indic{\nrm{W \x}\leq \gamma^{-1}}\nrm{W\x}^2}\leq \tr(2\id)=2d, \qquad
    \Ex{\indic{\nrm{W \x}\geq \gamma^{-1}} \nrm{W\x}}\leq \tr(2\gamma \id)=2\gamma d.
\end{align*}
This immediately implies $\Ex{\nrm{W\x}}\leq \sqrt{2d}+\gamma d$. Furthermore, we can also bound
\begin{align*}
    \Ex{\indic{\nrm{W \x}\in[G,\gamma^{-1}]}\nrm{W\x}}
    \leq \Ex{\indic{\nrm{W \x}\in[G,\gamma^{-1}]}\frac{\nrm{W\x}^2}{G}}\leq \frac{2d}{G},
\end{align*}
and hence $\Ex{\nrm{W\x}\indic{\nrm{W\x}>G}}\leq 2d\prn*{G^{-1}+\gamma}$.
\qed

\subsection{Proof of \cref{lem:U-L1-covariance}}

For simplicity, we denote $U=\Ustar$ and $v=U\iv \theta$. For any vector $v\in\R^d$, we have
\begin{align*}
    \Ex{ \abs{\lr \x, Uv\rr} }^2
    \leq \Ex{ \nrm{U\x} } \cdot \Ex{\frac{\lr \x, Uv\rr^2}{\nrm{U\x}}}.
\end{align*}
Note that $\Ex{\uxxu}\preceq \id$, and hence we have $\Ex{ \nrm{U\x} }\leq d$ and $\Ex{\frac{\lr \x, Uv\rr^2}{\nrm{U\x}}}\leq \nrm{v}^2$. Therefore, it holds that
\begin{align*}
    \Ex{ \abs{\lr \x, Uv\rr} }\leq \sqrt{d}\nrm{v}.
\end{align*}
Further, note that $\lambda U\preceq \id$, and hence $\lambda \nrm{U v}\leq \nrm{v}$. Therefore, substituting $\theta=U v$, we have shown that 
\begin{align*}
    \Ex{\abs{\lr \x, \theta \rr}}+\lambda \nrm{\theta}\leq (\sqrt{d}+1)\nrm{U\iv \theta}.
\end{align*}
For the first inequality, we note that by $\FLDP(U)=\id$, we have
\begin{align*}
    v=\Ex{\uxxu}v+\lambda Uv
    =\Ex{\frac{U\x}{\nrm{U\x}}\lr U\x, v\rr}+\lambda Uv.
\end{align*}
Hence, taking norms on both sides, we have
\begin{align*}
    \nrm{v}
    =\nrm{\Ex{\frac{U\x}{\nrm{U\x}}\lr U\x, v\rr}+\lambda Uv}
    \leq \Ex{\abs{\lr U\x, v\rr}}+\lambda\nrm{Uv}.
\end{align*}
Substituting $\theta=U v$ gives the desired inequality.
\qed

\subsection{Proof of \cref{lem:Wstar-properties}}\label{appdx:proof-Wstar-properties}

The existence and uniqueness of $\Wstar$ are proven in \cref{sec:properties-U-W}.

Next, for any $W\in\PD$ such that $\Wapp$, \cref{lem:F-monotone} implies that $\nrm{\Wstar W\iv}\leq 2$ and $\nrm{W \Wstar\iv}\leq 2$, and hence by \cref{lem:matrix-monotone}, it holds that $\frac12 W\preceq \Wstar\preceq 2W$.

Finally, we note that for $W\ldef (\cov+\lambda\id)\isq$, it holds that $\FJDP(W)\preceq \id$. Hence $\nrm{W\Wstar\iv}\leq 1$, or equivalently, $\Wstar^2\succeq (\cov+\lambda\id)\iv$.
\qed

\begin{proof}[\pfref{lem:Wstar-Ustar}]
By definition, for any $W\in\PD$,
\begin{align*}
    \FJDP(W)=\Ex{\wxxw}+\lambda W\preceq \frac1\gamma \Ex{\uxxu[W]}+\lambda W=\frac{1}{\gamma} \FLDP[\gamma\lambda](W).
\end{align*}
Therefore, $\FLDP[\lambda\gamma](\Wstar)\succeq \gamma\id$, implying $\nrmop{\Ustar[\gamma\lambda] \Wstar\iv}\leq \frac{1}{\gamma}$ (by \cref{lem:F-monotone}).

Conversely, for any $U\in\PD$, fix any parameter $c\in[0,1]$, we bound
\begin{align*}
    \FJDP(U)-\lambda U=&~ \Ex{\wxxw[U]}
    \succeq \Ex{\wxxw[U]\indic{\nrm{U\x}\geq c}} \\
    \succeq&~ \Ex{\frac{U\x\x\tp U}{(c\iv+\gamma)\nrm{U\x}}\indic{\nrm{U\x}\geq c}}
\end{align*}
Further, note that $\Ex{\frac{U\x\x\tp U}{\nrm{U\x}}\indic{\nrm{U\x}\leq c}} \preceq c\id$, we can conclude that
\begin{align*}
    \FJDP(U)\succeq \frac{1}{c\iv+\gamma}\prn*{\FLDP[\lambda'](U)-c\id},
\end{align*}
where $\lambda'=(1+\gamma)\lambda$. Therefore, $\FJDP(\Ustar[\lambda'])\succeq \frac{1-c}{c\iv+\gamma}\id$, and hence
\begin{align*}
    \nrmop{\Wstar \Ustar[\lambda']\iv}\leq \frac{c\iv+\gamma}{1-c}.
\end{align*}
Choosing $c=\frac12$ completes the proof.
\end{proof}

\section{Technical tools}

\subsection{Concentration}

\begin{lemma}[Gaussian concentration]\label{lem:Gaussian-concen}
For the random vector $z\in\R^d$ with i.i.d. $\normal{0,\sigma^2}$ entries, \whp, 
\begin{align*}
    \frac{\nrm{z}^2}{\sigma^2}\leq d+2\sqrt{d\log(1/\delta)}+2\log(1/\delta)\leq 2d+3\log(1/\delta).
\end{align*}
For the random matrix $Z\in\Rdd$ with i.i.d. Gaussian entries $Z_{ij}\sim \normal{0,\sigma^2}$ $(1\leq i, j\leq d)$, \whp,
\begin{align*}
    \nrmop{Z}\leq C\sigma\sqrt{d+\log(1/\delta)},
\end{align*}
where $C$ is an absolute constant. Similarly, for the random symmetric matrix $Z\in\Rdd$ with i.i.d. Gaussian entries $Z_{ij}=Z_{ji}\sim \normal{0,\sigma^2}$ $(1\leq i\leq j\leq d)$, we also have \whp~that
\begin{align*}
    \nrmop{Z}\leq C\sigma\sqrt{d+\log(1/\delta)},
\end{align*}
\end{lemma}

\begin{lemma}[Vector concentration inequality ({\citet{martinez2024empirical}})]\label{lem:vec-concen}
    Let $X_1, X_2, \dots, X_n$ be independent zero-mean vectors in $\R^{d}$. Suppose that $\nrm{X_i} \leq B$ almost surely.
    Then, for all \( t \geq 0 \),
    \[
    \mathbb{P} \left( \nrm{ \sum_{i=1}^n X_i } \geq t \right) \leq 2\exp \left( - \frac{t^2}{2nB^2} \right).
    \]
    Furthermore, we have the following Bernstein's inequality: 
    \[
    \mathbb{P} \left( \nrm{ \sum_{i=1}^n X_i } \geq t \right) \leq 2\exp \left( - \frac{t^2 / 2}{\sum_{i=1}^n \EE\nrm{X_i}^2+ Bt/3} \right), \qquad \forall t\geq 0.
    \]
\end{lemma}

\begin{lemma}[Matrix Bernstein's inequality ({\citet[Theorem 6.1.1]{tropp2015introduction}})]\label{lem:Bernstein}
Let $X_1, X_2, \dots, X_n$ be independent zero-mean matrices in $\R^{d_1\times d_2}$. Suppose that $\nrmop{X_i} \leq B$ almost surely, and let
\begin{align*}
    \sigma^2 = \max\crl*{\nrmop{\sum_{i=1}^n\mathbb{E}[X_iX_i\tp]}, \nrmop{\sum_{i=1}^n\mathbb{E}[X_i\tp X_i]}}.
\end{align*}
Then, for all \( t \geq 0 \),
\[
\mathbb{P} \left( \nrmop{ \sum_{i=1}^n X_i } \geq t \right) \leq (d_1+d_2) \exp \left( - \frac{t^2 / 2}{\sigma^2 + Bt/3} \right).
\]
\end{lemma}

\begin{lemma}\label{lem:cov-concen}
Suppose that $V_1,\cdots,V_n$ are i.i.d. positive semi-definite random matrices such that $\EE[V_i]=\Sigma$ and $\nrmop{V_i}\leq B$ almost surely. Then for any fixed parameter $c\in(0,1]$, \whp:
\begin{align*}
    (1-c)\Sigma-\frac{4B\log(2d/\delta)}{nc}\id \preceq \frac1n\sum_{i=1}^n V_i\preceq (1+c)\Sigma+\frac{4B\log(2d/\delta)}{nc}.
\end{align*}
\end{lemma}

\begin{proof}[\pfref{lem:cov-concen}]
Fix a $\lambda>0$ to be specified later. Let $\til V_i=(\Sigma+\lambda \id)\isq V_i(\Sigma+\lambda \id)\isq$. Note that $\nrmopn{\til V_i}\leq \frac{B}{\lambda}$ and $\EE{\til V}\preceq \id$, and hence
\begin{align*}
    \EE{\til V_i}^2\preceq \EE{\frac{B}\lambda\til V_i}\preceq \frac{B}{\lambda}\id.
\end{align*}
Then by \cref{lem:Bernstein}, \whp,
\begin{align*}
    \nrmop{\sum_{i=1}^n \til V_i-\EE{\til V_i}}\leq \sqrt{\frac{2Bn}{\lambda}\log(2d/\delta)}+\frac{2B\log(2d/\delta)}{3\lambda}=:c_\lambda.
\end{align*}
Under this event, we have $v\tp\paren{\sum_{i=1}^n \til V_i-\EE{\til V_i}}v\leq c_\lambda\nrm{v}^2$ for all $v\in\Rd$, and hence
\begin{align*}
    n\Sigma-c_\lambda(\Sigma+\lambda\id)\preceq \sum_{i=1}^n V_i \preceq n\Sigma+c_\lambda(\Sigma+\lambda\id).
\end{align*}
Therefore, we choose $\lambda=\frac{4B\log(2d/\delta)}{nc^2}$, which ensures $c_\lambda\leq nc$ and
\begin{align*}
    (1-c)\Sigma-c\lambda\id \preceq \frac1n\sum_{i=1}^n V_i\preceq (1+c)\Sigma+c\lambda\id,
\end{align*}
\whp.
\end{proof}

\subsection{Matrix inequalities}

\begin{lemma}\label{lem:uxxu}
For any vector $\x\in\R^d$ and positive-definite matrix $U$, the matrix $\Phi=\usqx$ is bounded under the Frobenius norm: $\nrmF{\Phi}\leq \nrm{\x}$.
\end{lemma}

\begin{proof}
    By definition,
    \begin{align*}
        \nrmF{\Phi}^2=\tr(\Phi^2)=\tr\prn*{ \frac{\nrm{U\sq\x}^2}{\nrm{U\x}^2}U\sq \x\x\tp U\sq }
        =\frac{(\x\tp U\x)^2}{\nrm{U\x}^2}
        \leq \nrm{\x}^2.
    \end{align*}
\end{proof}

\begin{lemma}[Monotone matrix operation]\label{lem:matrix-monotone}
Suppose that $A, B\in\Rdd$ are PSD matrices, $B$ is invertible, and $A\preceq B$. Then it holds that $A\iv\succeq B\iv$ and $A\sq \preceq B\sq$. 

As a corollary, for PSD matrix $U$ and $V\succ 0$, we have $U\preceq \nrmop{UV\iv} V$.
\end{lemma}
\begin{proof}
Because $A\preceq B$, we have $(A\sq B\isq)\tp (A\sq B\isq)=B\isq A B\isq\preceq \id$, and hence $A\sq B\iv A\sq = (A\sq B\isq)(A\sq B\isq)\tp\preceq \id$, which implies $B\iv\preceq A\iv$.

Similarly, we have $\nrm{A\sq v}\leq \nrm{B\sq v}$ for all $v\in\R^d$, and hence $\nrmop{A\sq B\isq}\leq 1$. Let $D=B\sups{-1/4}A\sq B\sups{-1/4}$. Because $D$ is symmetric, all eigenvalues of $D$ are real numbers. For any eigenvalue $\lambda$ of $D$, we take an eigenvector $v$ of $\lambda$, i.e., $Dv=\lambda v$. Then, $A\sq B\isq (B\sups{1/4}v)=\lambda (B\sups{1/4}v)$, and hence $\abs{\lambda}\leq 1$. Therefore, we can conclude that all eigenvalues of $D$ are bounded by $1$, and hence $\nrmop{D}\leq 1$, which implies $A\sq\preceq B\sq$.

To prove the last claim, we note that $U^2\preceq \nrmop{UV\iv}^2 V^2$, and hence we have $U=(U^2)\sq\preceq \nrmop{UV\iv} V$.
\end{proof}

\begin{lemma}\label{lem:cov-k-converge}
Let $\bSigma$ be a positive semi-definite matrix with $\bSigma\preceq \id$, and $\eta\in[0,1]$. For any integer $k\geq 0$,  we have $\nrmop{\bSigma\sq (\id-\eta\bSigma)^k}\leq \sqrt{\frac{2e}{\eta k}}$.
\end{lemma}

\paragraph{Proof}
Let the $1\geq \lambda_1\geq \cdots\geq \lambda_d\geq 0$ be the spectrum of $\bSigma$. Then 
\begin{align*}
    \nrmop{\bSigma\sq (\id-\eta\bSigma)^k}=\max_{i\in[d]}\sqrt{\lambda_i}(1-\eta\lambda_i)^k
    \leq\sup_{\lambda\in[0,1]} \sqrt{\lambda}e^{-\eta k\lambda}\leq \sqrt{\frac{2e}{\eta k}}.
\end{align*}
\qed

\section{Proofs from \cref{ssec:OLS}}\label{appdx:OLS}
\subsection{\pfref{lem:1d-LDP}}\label{appdx:proof-1d-LDP}

The lower bound is shown by \citet{duchi2024right}, and it is also a immediate corollary of \cref{lem:LDP-lower-unify}.

\citet{duchi2024right} also provides an asympotic upper bound achieved by private SGD. Here, we provide a straightforward proof that \cref{alg:LDP-simple-linear-regression} achieves the desired upper bound for any $T\geq1$. Specifically, we note that as $(\x_t,y_t)\sim M$ are drawn i.i.d from a linear model with covariate distribution $p$, we have
\begin{align*}
    \EE[\wh\psi]=\EE[\til\psi_t]=\EE_{(\x,y)\sim M}[\sign(\x)y]=\ths\cdot \EE_{\x\sim p}\abs{\x}, \qquad
    \EE[\wh\Psi]=\EE[\til\Psi_t]=\EE_{\x\sim p}\abs{\x}.
\end{align*}
Therefore,
\begin{align*}
    \EE\abs{\wh \psi-\EE[\wh\psi]}^2=\frac{1}{T^2}\sum_{t=1}^T \EE\abs{\til \psi_t-\EE[\til \psi_t]}^2
    \leq \frac{1}{T}\prn*{ \frac{8}{\alpha^2}+ 1 },
\end{align*}
and similarly,
\begin{align*}
    \EE\abs{\wh \Psi-\EE[\wh\Psi]}^2=\frac{1}{T^2}\sum_{t=1}^T \EE\abs{\til \Psi_t-\EE[\til \Psi_t]}^2
    \leq \frac{1}{T}\prn*{ \frac{8}{\alpha^2}+ 1 }.
\end{align*}
Finally, by definition $\hth=\clip[1]{\prn{\wh \Psi}\iv\wh\psi}$, we have $\abs{\wh\Psi(\hth-\ths)}\leq \abs{\wh\Psi \ths-\wh\psi}$, and hence
\begin{align*}
    \EE_{\x\sim p}\abs{\x}\cdot \abs{\hth-\ths}
    \leq 3\abs{\wh \Psi-\EE[\wh\Psi]}+\abs{\wh \psi-\EE[\wh\psi]}.
\end{align*}
Taking expectation completes the proof.
\qed

\subsection{\pfref{prop:OLS-fail-local-gen}}\label{appdx:proof-OLS-fail-local}

Let $\theta=\frac12\min\crl*{1, \frac{1}{\EE_{\x\sim p}\abs{\x}^2}\cdot \frac{1}{\rho\sqrt{T}}}$. 
By the chain rule of KL divergence and \cref{lem:data-proc-smooth}, it holds that
\begin{align*}
    \KLd{\PP\sups{(p,\theta),\alg}}{\PP\sups{(p,0),\alg}}
    =&~ \EE\sups{(p,\theta),\alg}\brk*{\sum_{t=1}^T \KLd{\PP\sups{(p,\theta),\alg}(o_t=\cdot\mid o_{1:t-1})}{\PP\sups{(p,0),\alg}(o_t=\cdot\mid o_{1:t-1})}} \\
    \leq&~ 2T\cdot \prn*{\theta\rho \EE_{\x\sim p}\abs{\x}^2}^2 \leq \frac{1}{2}. 
\end{align*}
Therefore, by data-processing inequality, 
\begin{align*}
    \DTV{\PP\sups{(p,\theta),\alg}(\hth=\cdot), \PP\sups{(p,0),\alg}(\hth=\cdot)}\leq \sqrt{\frac12\KLd{\PP\sups{(p,\theta),\alg}}{\PP\sups{(p,0),\alg}}} \leq \frac12.
\end{align*}
Hence, using $\crl*{ \hth: \absn{\hth-\theta}\geq \frac12\theta} \cup \crl*{ \hth: \absn{\hth}\geq \frac12\theta }=\R$, we can conclude that 
\begin{align*}
    \PP\sups{(p,0),\alg}\prn*{ \absn{\hth-0}\geq \frac12\theta } + \PP\sups{(p,\theta),\alg}\prn*{ \absn{\hth-\theta}\geq \frac12\theta} \geq \frac12.
\end{align*}
This is the desired result.
\qed

\begin{lemma}\label{lem:data-proc-smooth}
Suppose that $\pr:[-1,1]^2\to\DO$ is a $\varrho$-channel. Then it holds that
\begin{align*}
    \chis{\pr\circ M_{p,\theta}}{\pr\circ M_{p,0}}\leq 2\prn*{\varrho\theta\EE_{\x\sim p}\abs{\x}^2}^2,
\end{align*}
where we denote $\pr\circ M_{p,\theta}$ to be the distribution of $o$ under $(\x,y)\sim M_{p,\theta}$ and $o\sim \pr(\cdot\mid\x,y)$.
\end{lemma}

\paragraph{\pfref{lem:data-proc-smooth}}
For notational simplicity, we fix a common base measure $\mu$ over $\cO$, and we denote $Q_{\x,y}\ldef \pr(\cdot\mid \x,y)$ and let $f_{\x,y}\ldef \frac{dQ_{\x,y}}{d\mu}$ be the density function of $Q_{\x,y}$.

For any $\theta\in[-1,1]$, we let $P_\theta$ be the density function of $(o,y)$ under $z=(\x,y)\sim M_{p,\theta}$ and $o\sim \pr(\cdot\mid z)$. Then, note that by the definition of $M_{p,\theta}$, we have %
\begin{align*}
    P_\theta(o,y)=\frac{1}{2}\EE_{\x\sim p}\brk*{(1+y\x\theta)f_{\x,y}(o)}=P_0(o,y)+\frac{y\theta}{2}\EE_{\x\sim p}\brk*{\x f_{\x,y}(o)},
\end{align*}  
where $P_0(o,y)=\frac{1}{2}\EE_{\x\sim p}\brk*{f_{\x,y}(o)}$.
Therefore,
\begin{align*}
    \chis{P_\theta}{P_0}=&~ \frac12\sum_{y\in\crl{-1,1}} \int_{\cO} \frac{(P_\theta(o,y)-P_0(o,y))^2}{ P_0(o,y) }d\mu(o) \\
    =&~ \frac{\theta^2}{4} \sum_{y\in\crl{-1,1}} \int_{\cO} \frac{\prn*{\EE_{\x\sim p}\brk*{\x f_{\x,y}(o)}}^2}{ \EE_{\x\sim p}\brk*{f_{\x,y}(o)} }d\mu(o) \\
    \leq&~ \frac{\theta^2}{4} \sum_{y\in\crl{-1,1}} \int_{\cO} \prn*{\EE_{\x\sim p}\brk*{\x f_{\x,y}(o)}}^2 \cdot \EE_{\x\sim p}\brk*{\frac1{f_{\x,y}(o)}} d\mu(o),
\end{align*}
where the last line follows from Cauchy inequality.
Using $\EE_p[\x]=0$, we have for any $\x''\in[-1,1]$ 
\begin{align*}
    \prn*{\EE_{\x\sim p}\brk*{\x f_{\x,y}(o)}}^2
    =\prn*{\EE_{\x\sim p}\brk*{\x (f_{\x,y}(o)-f_{\x'',y}(o))}}^2=\EE_{\x,\x'\sim p}\brk*{ \x\x' \cdot(f_{\x,y}(o)-f_{\x'',y}(o))(f_{\x',y}(o)-f_{\x'',y}(o)) }.
\end{align*}
Combining the results above, we have show that
\begin{align*}
    \chis{P_\theta}{P_0}
    \leq &~ \frac{C\theta^2}{4} \sum_{y\in\crl{-1,1}}\EE_{\x,\x',\x''\sim p}\brk*{ \x\x' \cdot \int_{\cO} \frac{(f_{\x,y}(o)-f_{\x'',y}(o))(f_{\x',y}(o)-f_{\x'',y}(o))}{f_{\x'',y}(o)} d\mu(o) } \\
    =&~ \frac{\theta^2}{4} \sum_{y\in\crl{-1,1}} \EE_{\x,\x'\sim p}\brk*{ \x\x' \cdot \EE_{o\sim Q_{\x'',y}}\brk*{\prn*{\frac{dQ_{\x,y}}{dQ_{\x'',y}}-1}\prn*{\frac{dQ_{\x',y}}{dQ_{\x'',y}}-1}} } \\
    \leq&~ \frac{\theta^2}{4} \sum_{y\in\crl{-1,1}} \EE_{\x,\x'\sim p}\brk*{ \abs{\x\x'} \cdot  \sqrt{\chis{Q_{\x,y}}{Q_{\x'',y}}\cdot \chis{Q_{\x,y}}{Q_{\x'',y}}} } \\
    \leq&~ \frac{\theta^2}{2} \EE_{\x,\x'\sim p}\brk*{ \abs{\x\x'} \cdot \varrho^2(\abs{\x}+\abs{\x''})(\abs{\x'}+\abs{\x''}) } \\
    =&~  \frac{\theta^2\varrho^2}{2} \brk*{ \prn*{\EE_{\x\sim p}\abs{\x}^2}^2 + 3\prn*{\EE_{\x\sim p}\abs{\x}^2}\cdot \prn*{\EE_{\x\sim p}\abs{\x}}^2 } \\
    \leq&~ 2\theta^2\varrho^2 \cdot \prn*{\EE_{\x\sim p}\abs{\x}^2}^2.
\end{align*}
This is the desired result.
\qed

\subsection{Proof of \cref{lem:OLS-fail}}\label{appdx:proof-OLS-fail}

\newcommand{\bfx}{{\boldsymbol{\x}}}
\newcommand{\bfy}{{\bf{y}}}

The proof follows the standard coupling argument~\citep{karwa2017finite}.

For notational simplicity, for the dataset $\cD=\crl*{(\x_t,y_t)}_{t\in[T]}$, we denote 
\begin{align*}
    \cD_y=(y_1,\cdots,y_T), \qquad \cD_{\x}=(\x_1,\cdots,x_T).
\end{align*}
Then, under $M_{p,\theta}$, the marginal distribution of $\cD$ is $p^{\otimes T}$. Therefore, fix any algorithm $\alg$ with output $\hth$, we have
\begin{align*}
    \DTV{\PP\sups{(p,\theta),\alg}(\hth=\cdot), \PP\sups{(p,0),\alg}(\hth=\cdot)}
    \leq &~ \EE_{\cD_{\x}\sim p^{\otimes T}} \DTV{ \PP\sups{(p,\theta),\alg}(\hth=\cdot\mid \cD_{\x}), \PP\sups{(p,0),\alg}(\hth=\cdot\mid \cD_{\x})}.
\end{align*}
Note that for any fixed sequence $\bfx=(\bfx_1,\cdots,\bfx_T)$, 
\begin{align*}
    \PP\sups{(p,\theta),\alg}(\hth=\cdot\mid \cD_{\x}=\bfx)=\EE_{y_t\sim \Rad{\theta\bfx_t}\text{ independently}} \brk*{ \Alg(\cdot\mid (\bfx_1,y_1),\cdots,(\bfx_T,y_T) ) }.
\end{align*}
In the following, for the fixed sequence $\bfx$, for each $t\in[T]$, we take the coupling $\gamma_t$ of $\Rad{\theta \bfx_t}$ and $\Rad{0}$, such that $\PP_{(y_t,y_t')\sim \gamma_t}(y_t\neq y_t')=\DTV{\Rad{\theta \bfx_t},\Rad{0}}=\abs{\bfx_t\theta}$.

Then, let $P_{\bfx}$ be the distribution of the random variables $\cD_{y}=(y_1,\cdots,y_T)$ and $\cD_{y}'=(y_1',\cdots,y_T')$ generated as $(y_t,y_t')\sim \gamma_t$ independently for $t\in[T]$. Then marginally, we have $\cD_y\sim \PP\sups{p,\theta}(\cdot\mid\cD_{\x}=\bfx)$ and $\cD_y'\sim \PP\sups{p,0}(\cdot\mid\cD_{\x}=\bfx)$. Therefore, we can bound
\begin{align*}
    &~ \DTV{ \PP\sups{(p,\theta),\alg}(\hth=\cdot\mid \cD_{\x}=\bfx), \PP\sups{(p,0),\alg}(\hth=\cdot\mid \cD_{\x}=\bfx)} \\
    =&~ \DTV{ \EE_{\cD_{y}\sim P_{\bfx}} \brk*{ \Alg(\cdot\mid \cD_{\x}=\bfx, \cD_y ) }, \EE_{\cD_{y}'\sim P_{\bfy}} \brk*{ \Alg(\cdot\mid \cD_{\x}=\bfx, \cD_y' ) } } \\
    \leq&~ \EE_{(\cD_{y}, \cD_{y}')\sim P_{\bfx}} \DTV{  \Alg(\cdot\mid \cD_{\x}=\bfx, \cD_y ) , \Alg(\cdot\mid \cD_{\x}=\bfx, \cD_y' ) } \\
    \leq&~ \EE_{(\cD_{y}, \cD_{y}')\sim P_{\bfx}}\brk*{ \sum_{t=1}^T \indic{y_t\neq y_t'}\cdot \varrho\abs{\bfx_t} } 
    \leq \varrho \abs{\theta} \sum_{t=1}^T \abs{\bfx_t}^2 
\end{align*}
where the first inequality uses convexity of TV distance, the second inequality follows from the condition on $\alg$, and last line follows from the coupling property.

Finally, choosing $\theta=\min\crl*{1, \frac{1}{2T\varrho\EE_{\x\sim p}\abs{\x}^2}}$, we can conclude that
\begin{align*}
    \DTV{\PP\sups{(p,\theta),\alg}(\hth=\cdot), \PP\sups{(p,0),\alg}(\hth=\cdot)}
    \leq&~ \EE_{\cD_{\x}\sim p^{\otimes T}} \DTV{ \PP\sups{(p,\theta),\alg}(\hth=\cdot\mid \cD_{\x}), \PP\sups{(p,0),\alg}(\hth=\cdot\mid \cD_{\x})} \\
    \leq&~ \EE_{\cD_{\x}\sim p^{\otimes T}}\brk*{\varrho \abs{\theta} \sum_{t=1}^T \abs{\x_t}^2 } 
    =T\varrho\abs{\theta}\EE_{\x\sim p}\abs{\x}^2\leq \frac12.
\end{align*}
Hence, similar to the proof of \cref{prop:OLS-fail-local-gen}, using $\crl*{ \hth: \absn{\hth-\theta}\geq \frac12\theta} \cup \crl*{ \hth: \absn{\hth}\geq \frac12\theta }=\R$, we derive
\begin{align*}
    \PP\sups{(p,0),\alg}\prn*{ \absn{\hth-0}\geq \frac12\theta } + \PP\sups{(p,\theta),\alg}\prn*{ \absn{\hth-\theta}\geq \frac12\theta} \geq \frac12.
\end{align*}
This is the desired result.
\qed

\section{Proofs of the Lower Bounds}\label{appdx:lower}
\newcommand{\negi}{\sups{(i)}}
\newcommand{\vhat}{\widehat{v}}

\subsection{Proofs of local DP lower bounds}\label{appdx:proof-LDP-lower}

We first prove the following lemma.

\begin{lemma}\label{lem:LDP-lower-unify}
For any covariate distribution $p$ supported on $\Ball(\Bd)$, $\theta\in\Rd\backslash\crl{0}$, PSD matrix $A$ and privacy parameter $\alpha\in(0,1]$, $\beta\in[0,\frac{1}{8T}]$. Then there is a parameter $s\geq 0$ such that $s\theta\in\Theta_p$, and for any $T$-round \aLDP~algorithm $\alg$ with output $\hth$, it holds that
\begin{align*}
    \sup_{\ths\in\set{0,s\theta}}\EE\sups{\ths,\alg} \nrmn{\hth-\ths}_A\geq \frac{\nrm{\theta}_A}{8}\min\crl*{\frac{1}{B\nrm{\theta}},\frac{1}{4\alpha\sqrt{T}\cdot \EE_{\x\sim p} \abs{\lr x, \theta\rr}}}.
\end{align*}
\end{lemma}

\paragraph{Proof of \cref{lem:LDP-lower-unify}}
We choose
\begin{align*}
    s=\min\crl*{\frac{1}{B\nrm{\theta}}, \frac{1}{4\alpha \sqrt{T}\cdot \EE_{\x\sim p} \abs{\lr x, \theta\rr}}}.
\end{align*}
Then by definition, $s\theta\in\Theta_p$, and $\EE_{\x\sim p} \abs{\lr x, s\theta\rr}\leq 
\frac{1}{4\alpha\sqrt{T}}$. Hence, by \cref{lem:LDP-chain-rule}, we have
\begin{align*}
    \DTV{\PP\sups{(p,s\theta),\alg},\PP\sups{(p,0),\alg}}\leq \frac{1}{2}.
\end{align*}
On the other hand, it holds that
\begin{align*}
    \crl*{\hth: \nrmn{\hth-s\theta}_A<\frac12s\nrm{\theta}_A}\cap \crl*{\hth: \nrmn{\hth}_A<\frac12s\nrm{\theta}_A}=\emptyset.
\end{align*}
Therefore,
\begin{align*}
    \PP\sups{(p,0),\alg}\prn*{ \nrmn{\hth-s\theta}_A\geq \frac12s\nrm{\theta}_A }+\PP\sups{(p,0),\alg}\prn*{ \nrmn{\hth}_A\geq \frac12s\nrm{\theta}_A }\geq 1,
\end{align*}
and hence by $\DTV{\PP\sups{(p,s\theta),\alg},\PP\sups{(p,0),\alg}}\leq \frac12$, we have
\begin{align*}
    \PP\sups{(p,s\theta),\alg}\prn*{ \nrmn{\hth-s\theta}_A\geq \frac12s\nrm{\theta}_A }+\PP\sups{(p,0),\alg}\prn*{ \nrmn{\hth-0}_A\geq \frac12s\nrm{\theta}_A }\geq \frac12.
\end{align*}
This gives the desired lower bound.
\qed

\paragraph{Proof of \cref{prop:lower-linear-cov} (a)}
Fix the decomposition $\covs=\sum_{i=1}^d \rho_i e_ie_i\tp$. We pick
\begin{align*}
    j=\argmax_{i\in[d]} \min\crl*{\rho_i, \frac{B^2}{\rho_i\alpha^2 T} }, \qquad \theta=e_j.
\end{align*}
Then, by definition, $\nrm{\theta}_{\cov}=\sqrt{\rho_j}$ and $ \EE_{\x\sim p} \abs{\lr \x, \theta\rr}
=\frac{\rho_j}{B^2}\cdot B=\frac{\rho_j}{B}$. Therefore, by \cref{lem:LDP-lower-unify}, there exists $s\geq0$ such that $\nrm{s\theta}\leq 1$ and for any \aLDP~algorithm $\alg$,
\begin{align*}
    \sup_{\ths\in\set{0,s\theta}}\EE\sups{(p,\ths),\alg} \nrmn{\hth-\ths}_{\cov}\geq \frac{\nrm{\theta}_{\cov}}{8}\min\crl*{\frac{1}{B},\frac{1}{4\alpha\sqrt{T}\EE_{\x\sim p} \abs{\lr \x, \theta\rr}}}
    =\frac{\sqrt{\rho_j}}{8}\min\crl*{\frac{1}{B},\frac{B}{4\rho_j \alpha\sqrt{T}}}.
\end{align*}
This is the desired lower bound.
\qed

\paragraph{Proof of \cref{prop:lower-perturbed-p} (a)}
We pick $\theta=e$, where $e$ is the unit vector defined in \cref{prop:lower-perturbed-p} satisfying $e\tp\cov e=\rho^2$. Then, we know $\nrm{e}_{\cov'}\geq \sqrt\rho$, and 
\begin{align*}
    \EE_{\x\sim p'} \abs{\lr \x, \theta\rr}
    =(1-\rho)\EE_{\x\sim p} \abs{\lr \x, \theta\rr}+\rho\abs{\lr e, \theta\rr}
    \leq \nrm{\theta}_{\cov}+\rho\abs{\lr e, \theta\rr}=2\rho.
\end{align*}
Therefore, applying \cref{lem:LDP-lower-unify} to the distribution $p'$ implies that there exists $s\geq0$ such that $\nrm{s\theta}\leq 1$ and for any \aLDP~algorithm $\alg$,
\begin{align*}
    \sup_{\ths\in\set{0,s\theta}}\EE\sups{(p,\ths),\alg} \nrmn{\hth-\ths}_{\cov}\geq \frac{\nrm{\theta}_{\cov}}{8}\min\crl*{\frac{1}{B},\frac{1}{4\alpha\sqrt{T}\EE_{\x\sim p} \abs{\lr \x, \theta\rr}}}
    \geq \frac{\sqrt{\rho}}{8}\min\crl*{\frac{1}{B},\frac{B}{8\rho \alpha\sqrt{T}}}.
\end{align*}
This is the desired lower bound.
\qed

\paragraph{Proof of \cref{lem:LDP-lower-bound-any} (a)}
Fix $\lambda>0$ and denote $U=\Ustar$. We pick a unit vector $e$ such that $A\sq Ue=\nrmop{A\sq U}e$, and 
we consider $\theta=Ue$. Note that by definition, $\lambda U\preceq \id$, and hence $\nrm{\theta}\leq \frac{1}{\lambda}$. Furthermore,
\begin{align*}
    \EE_{\x\sim p} \abs{\lr \x, \theta\rr}
    =\EE_{\x\sim p} \abs{\lr U\x, e\rr}
    \leq \sqrt{d},
\end{align*}
where the last inequality follows from \cref{lem:U-L1-covariance}. Hence, by \cref{lem:LDP-lower-unify}, there exists $s\geq0$ such that $\nrm{s\theta}\leq 1$ and for any \aLDP~algorithm $\alg$,
\begin{align*}
    \sup_{\ths\in\set{0,s\theta}}\EE\sups{(p,\ths),\alg} \nrmn{\hth-\ths}_A\geq \frac{\nrm{Ue}_A}{8}\min\crl*{\frac{\lambda}{B},\frac{1}{4\alpha\sqrt{dT}}}.
\end{align*}
The proof is completed by noticing $\nrm{Ue}_A=\nrm{A\sq Ue}=\nrmop{A\sq U}$ and $\lambda\geq \frac{B}{\alpha\sqrt{dT}}$.
\qed

\paragraph{\pfref{lem:LDP-lower-bound-any} (b)}
We combine the proof of \cref{lem:LDP-lower-bound-any} (a) with the standard Assouad method~\citep[][e.g.]{duchi2013local,chen2024beyond}. 

Fix $\lambda>0$ and denote $U=\Ustar$.
Fix an eigen-decomposition of $UAU$ as $UAU=\sum_{i=1}^d \rho_i e_ie_i$, where $e_1,\cdots,e_d$ are orthonormal. Let $v_i=Ue_i$, and then $v_j\tp A v_i\tp=0$ for any $i\neq j$. 
We let
\begin{align*}
    s=\min\crl*{\frac{1}{16\alpha\sqrt{dT}}, \frac{1}{2B\sqrt{d}\nrmop{U}} },
\end{align*}
and for any $v\in\crl{-1,1}^d$, we define $\theta_v\ldef s U \sum_{i=1}^d v[i]e_i$. We consider the set
\begin{align*}
    \Theta\ldef \crl*{ \theta_v: v\in\crl{-1,1}^d }.
\end{align*}
It is straightforward to verify $\nrm{\theta_v}\leq s\nrm{U}\cdot \sqrt{d}\leq \frac{1}{2B}$, and hence $\Theta\subset \Ball(\frac1{2B})$. Further, note that
\begin{align*}
    \nrmn{\hth-\theta_v}_{A}^2
    =&~ \nrmn{U\iv \hth-U\iv \theta_v}_{UAU}^2
    = \sum_{i=1}^d \rho_i \lr e_i, U\iv \hth-U\iv \theta_v\rr^2 \\ 
    = &~ \sum_{i=1}^d \rho_i (\lr e_i, U\iv \hth\rr-s v[i])^2
    = s^2\sum_{i=1}^d \rho_i \prn{\vhat_i-v[i]}^2,
\end{align*}
where we define $\vhat_i\ldef s\iv \lr e_i, U\iv \hth\rr$ for $i\in[d]$.

For any $v\in\crl{-1,1}^d$, we write $v\negi$ to be $v\negi[j]=v[j]$ for any $j\neq i$, and $v\negi[i]=-v[i]$. Further, for any $\theta=\theta_v\in\Theta$, we denote $\theta\negi\ldef \theta_{v\negi}$. Then, it is clear that $\theta_v-\theta_{v}\negi=2sUe_i$, and hence by \cref{lem:U-L1-covariance},
\begin{align*}
    \DTV{M_{p,\theta_v},M_{p,\theta_v\negi}} =\frac12 \EE_{\x\sim p} \abs{\lr \x, 2sUe_i \rr}
    \leq \sqrt{d}s\nrm{e_i}\leq \frac{1}{16\alpha\sqrt{T}}.
\end{align*}
Therefore, by \cref{lem:LDP-chain-rule}, it holds that
\begin{align*}
    \DTV{\PP\sups{\theta_v,\alg}, \PP\sups{\theta_v\negi,\alg}}\leq 4\alpha\sqrt{T}\DTV{M_{p,\theta_v},M_{p,\theta_v\negi}}+2\beta T\leq \frac12,
\end{align*}
and hence
\begin{align*}
    \PP\sups{\theta_v,\alg}\prn*{\vhat_i\cdot v[i]<0}+\PP\sups{\theta\negi_v,\alg}\prn{ \vhat_i\cdot v[i]\geq 0}\geq \frac12,
\end{align*}
which further implies
\begin{align*}
    \EE\sups{\theta_v,\alg}\prn*{\vhat_i-v[i]}^2+\EE\sups{\theta\negi_v,\alg}\prn*{\vhat_i-v\negi_i}^2\geq \frac{1}{2}.
\end{align*}
Taking $v\sim \Unif(\crl*{-1,1}^d)$ gives $\EE_{v\sim \Unif(\crl*{-1,1}^d)} \EE\sups{\theta_v,\alg}\prn*{\vhat_i-v[i]}^2\geq \frac{1}{4}$. 
Therefore,
\begin{align*}
    \EE_{v\sim \Unif(\crl*{-1,1}^d)} \EE\sups{\theta_v,\alg}\nrmn{\hth-\theta_v}_{A}^2=s^2\sum_{i=1}^d \rho_i \EE_{v\sim \Unif(\crl*{-1,1}^d)} \EE\sups{\theta_v,\alg}\prn*{\vhat_i-v[i]}^2 \geq \frac{s^2}4\sum_{i=1}^d \rho_i
    =\frac{s^2}{4}\tr(UAU).
\end{align*}
This gives the desired lower bound (using $\nrmop{U}\leq \frac{1}{\lambda}$).
\qed

\subsection{Proof of the DP lower bounds}\label{appdx:proof-W-lower}

\paragraph{Proof of \cref{lem:JDP-lower-bound-any} (a)}
Fix any parameter $\lambda>0, \gamma>0$ and we denote $W=\Wstar$. 
We pick an unit vector $e$ such that $A\sq W e=\nrmop{A\sq W} e$, and we let
\begin{align*}
    \theta=s W e, \qquad s=\min\crl*{\frac{1}{8\sqrt{T}}, \frac{1}{32\alpha\sqrt{d}\gamma T}, \frac{1}{2B\nrmop{W}} }.
\end{align*}
By definition, $\nrm{\theta}_A=s\nrmop{A\sq W}$ and $\nrm{\theta}\leq \frac{1}{2B}$.
Further, by \cref{lem:E-nrm-W-bound}, we can bound
\begin{align*}
    \Ex{\min\crl*{\lr \x, \theta \rr^2, \alpha \abs{\lr \x, \theta\rr} }}
    \leq \nrm{W\iv \theta}^2+\alpha\sqrt{d}\gamma\nrm{W\iv \theta}
    =s^2+s\alpha\sqrt{d}\gamma\leq \frac{1}{16T}.
\end{align*}
Therefore, applying \cref{lem:DP-chain-rule}, we have for any \aDP~algorithm with $\beta\leq \alpha$.
\begin{align*}
    \PP\sups{(p,\theta),\alg}\prn*{ \nrmn{\hth-\theta}_A\geq \frac{1}{2}\nrm{\theta}_A }+\PP\sups{0,\alg}\prn*{ \nrmn{\hth}_A\geq \frac{1}{2}\nrm{\theta}_A }\geq \frac12.
\end{align*}
This gives the desired lower bound, as $\nrm{\theta}_A=s\nrm{A\sq We}=s\nrmop{A\sq W}$.

\paragraph{Proof of \cref{lem:JDP-lower-bound-any} (b)}
Fix any parameter $\lambda>0, \gamma>0$ and we denote $W=\Wstar$. 
We consider the decompsition of $WA W=\sum_{i=1}^d \rho_i e_i e_i\tp$, where $\rho_1,\cdots,\rho_d\geq 0$ and $(e_1,\cdots,e_d)$ is an orthogonal basis of $\R^d$. We let
\begin{align*}
    s=\min\crl*{\frac{1}{8\sqrt{T}}, \frac{1}{32\alpha\sqrt{d}\gamma T}, \frac{1}{2B\sqrt{d}\nrmop{W}} },
\end{align*}
and for any $v\in\crl{-1,1}^d$, we define $\theta_v\ldef s W \sum_{i=1}^d v[i]e_i$. We consider the set
\begin{align*}
    \Theta\ldef \crl*{ \theta_v: v\in\crl{-1,1}^d }.
\end{align*}
It is straightforward to verify $\nrm{\theta_v}\leq s\nrm{W}\cdot \sqrt{d}\leq \frac{1}{2B}$, and hence $\Theta\subset \Ball(\frac1{2B})$. Further, note that
\begin{align*}
    \nrmn{\hth-\theta_v}_{A}^2
    =&~ \nrmn{W\iv \hth-W\iv \theta_v}_{WAW}^2
    = \sum_{i=1}^d \rho_i \lr e_i, W\iv \hth-W\iv \theta_v\rr^2 \\ 
    = &~ \sum_{i=1}^d \rho_i (\lr e_i, W\iv \hth\rr-s v[i])^2
    = s^2\sum_{i=1}^d \rho_i \prn{\vhat_i-v[i]}^2,
\end{align*}
where we define $\vhat_i\ldef s\iv \lr e_i, W\iv \hth\rr$ for $i\in[d]$.

For any $v\in\crl{-1,1}^d$, we write $v\negi$ to be $v\negi[j]=v[j]$ for any $j\neq i$, and $v\negi[i]=-v[i]$. Further, for any $\theta=\theta_v\in\Theta$, we denote $\theta\negi\ldef \theta_{v\negi}$. Then, it is clear that $\theta_v-\theta_{v}\negi=2sWe_i$. To apply \cref{lem:DP-chain-rule}, we note that by \cref{lem:E-nrm-W-bound},
\begin{align*}
    \EE_{\x\sim p} \min\crl*{ \abs{\lr \x, 2sWe_i \rr}^2, \alpha\abs{\lr \x, 2sWe_i\rr}}
    \leq 4s^2+2\alpha\sqrt{d}\gamma s\leq \frac{1}{16T}.
\end{align*}
Therefore, by \cref{lem:DP-chain-rule}, it holds that
\begin{align*}
    \PP\sups{\theta_v,\alg}\prn*{\vhat_i\cdot v[i]<0}+\PP\sups{\theta\negi_v,\alg}\prn{ \vhat_i\cdot v[i]\geq 0}\geq \frac12,
\end{align*}
and hence
\begin{align*}
    \EE\sups{\theta_v,\alg}\prn*{\vhat_i-v[i]}^2+\EE\sups{\theta\negi_v,\alg}\prn*{\vhat_i-v\negi_i}^2\geq \frac{1}{2}.
\end{align*}
Taking $v\sim \Unif(\crl*{-1,1}^d)$ gives $\EE_{v\sim \Unif(\crl*{-1,1}^d)} \EE\sups{\theta_v,\alg}\prn*{\vhat_i-v[i]}^2\geq \frac{1}{4}$. 
Therefore,
\begin{align*}
    \EE_{v\sim \Unif(\crl*{-1,1}^d)} \EE\sups{\theta_v,\alg}\nrmn{\hth-\theta_v}_{A}^2=s^2\sum_{i=1}^d \rho_i \EE_{v\sim \Unif(\crl*{-1,1}^d)} \EE\sups{\theta_v,\alg}\prn*{\vhat_i-v[i]}^2 \geq \frac{s^2}4\sum_{i=1}^d \rho_i
    =\frac{s^2}{4}\tr(WAW).
\end{align*}
This gives the desired lower bound.
\qed

\paragraph{Proof of \cref{prop:lower-perturbed-p} (b)}
Let $s=\min\crl*{1,\frac{1}{32\alpha\rho T}}$.
We have
\begin{align*}
    \EE_{\x\sim p'}\abs{\lr \x, e\rr}
    =(1-\rho)\EE_{\x\sim p'}\abs{\lr \x, e\rr}+\rho
    \leq \nrm{e}_{\cov}+\rho=2\rho,
\end{align*}
and hence
\begin{align*}
    \EE_{\x\sim p'} \min\crl*{ \abs{\lr \x, se \rr}^2, \alpha\abs{\lr \x, se\rr}}
    \leq 2\alpha s \rho\leq \frac{1}{16T}.
\end{align*}
Then, by \cref{lem:DP-chain-rule}, we know $\DTV{\PP\sups{p',se}(\hth=\cdot), \PP\sups{p',0}(\hth=\cdot)}\leq \frac12$. Therefore, similar to the proof of \cref{lem:JDP-lower-bound-any} (a), this implies 
\begin{align*}
    \PP\sups{p',se}\prn*{ \nrmn{\hth-se}_{\cov'}\geq \frac{s}{2}\nrm{e}_{\cov'} }+\PP\sups{p',0}\prn*{ \nrmn{\hth-0}_{\cov'}\geq \frac{s}{2}\nrm{e}_{\cov'} }\geq \frac12.
\end{align*}
This gives the desired lower bound, as $\nrm{e}_{\cov'}\geq \sqrt{\rho}$.
\qed

\subsubsection{Proof of \cref{prop:lower-linear-cov-JDP}}\label{appdx:proof-lower-cov}

We follow the proof strategy of \cref{lem:JDP-lower-bound-any} (b).

Without loss of generality, we assume $\covs=\diag(\rho_1,\cdots,\rho_d)$. 
We fix a sequence $(v_1,\cdots,v_d)$ such that
\begin{align}\label{eq:proof-lower-vi-seq}
    \sum_{i=1}^d v_i^2\leq 1, \qquad 0\leq v_i\leq \min\crl*{\frac1B, \frac{B}{16\rho_i \alpha T}},~~\forall i\in[d].
\end{align}
and we consider the following set:
\begin{align*}
    \Theta\defeq \crl{v_1,-v_1}\times\cdots\times\crl{v_d,-v_d}\subset \Theta_p.
\end{align*}
For any $\theta\in\Theta$, we write $\theta\negi\in\Theta$ to be $\theta\negi_j=\theta_j$ for any $j\neq i$, and $\theta\negi_i=-\theta_i$. Then, by definition,
\begin{align*}
    \EE_{\x\sim p}\abs{\lr \x, \theta-\theta\negi\rr}=p(Be_i)\cdot B v_i\leq \frac{1}{16\alpha T}.
\end{align*}
In particular, by \cref{lem:DP-chain-rule}, it holds that for any $\theta\in\Theta$,
\begin{align*}
    \PP\sups{p,\theta}\prn*{\hth_i\cdot \theta_i<0}+\PP\sups{p,\theta\negi}\prn{ \hth_i\cdot \theta_i\geq 0}\geq \frac12,
\end{align*}
and hence
\begin{align*}
    \EE\sups{p,\theta}\prn*{\hth_i-\theta_i}^2+\EE\sups{p,\theta\negi}\prn*{\hth_i-\theta\negi_i}^2\geq \frac{1}{2}v_i^2.
\end{align*}
Taking $\theta\sim \Unif(\Theta)$ gives $\EE_{\theta\sim \Unif(\Theta)} \EE\sups{p,\theta}\prn*{\hth_i-\theta_i}^2\geq \frac{1}{4}v_i^2$. Therefore,
\begin{align*}
    \EE_{\theta\sim \Unif(\Theta)} \EE\sups{p,\theta}\nrmn{\hth-\theta}_{\covs}^2=\sum_{i=1}^d \rho_i \EE_{\theta\sim \Unif(\Theta)} \EE\sups{p,\theta}\prn*{\hth_i-\theta_i}^2 \geq \frac14\sum_{i=1}^d \rho_iv_i^2.
\end{align*}
In particular, to derive the lower bound stated in \cref{prop:lower-linear-cov-JDP}, we only need to take $v_i=\frac{B}{16\rho_i\alpha T}$ for all $i\in[d^\star]$.
\qed

\subsection{Proof of \cref{lem:DP-chain-rule}}\label{appdx:proof-DP-chain-rule}

\newcommand{\dHam}{d_{\rm Ham}}

\newcommand{\gambar}{\overline{\gamma}}
\newcommand{\fbar}{\overline{f}}
We let $M$ be the distribution over $\R^d\times \crl{-1,1}$ defined as
\begin{align*}
    (\x,y)\sim M:\qquad \x\sim p, ~~y\mid x\sim \Rad{f(\x)}.
\end{align*}
where the function $f: \R^d\to [-\frac12,\frac12]$ is defined as
\begin{align*}
    f(\x)=\begin{cases}
        \lr \x, \theta\rr, & \text{if }\abs{\lr \x, \theta-\otheta\rr}\leq \alpha, \\
        \lr \x, \otheta\rr, & \text{otherwise}.
    \end{cases}
\end{align*}
By definition, it holds that $\DTV{M_{p,\theta}, M}\leq \frac{\eps}{2\alpha}$. Furthermore, using the fact that $\KLd{\Rad{x}}{\Rad{y}}\leq \frac{4}{3}(x-y)^2$ for all $x\in[-1,1], y\in[-\frac12,\frac12]$, we also have $\KLd{M}{M_{p,\otheta}}\leq \frac43\eps$.

Fix coupling $\gamma\in\Delta(\RR^d\times \crl{-1,1}\times\crl{-1,1})$ such that under $(x,y,y')\sim \gamma$, the marginal distribution of $z=(x,y)$ is $M_\theta$, the marginal distribution of $z'=(x,y')$ is $M$, and $\PP_\gamma(y\neq y')=\DTV{M_\theta,M}$. Then, consider the following random variables:
\begin{align*}
    (x_1,y_1,y_1'),\cdots,(x_T,y_T,y_T')\sim \gamma \text{ independently},
\end{align*}
and let $\cD=((x_1,y_1),\cdots,(x_T,y_T)), \cD'=((x_1,y_1'),\cdots,(x_T,y_T'))$. 
Note that
\begin{align*}
    \dHam(\cD,\cD')=\sum_{i=1}^T \indic{y_i\neq y_i'}
\end{align*}
is a sum of i.i.d. binary variables, we know
\begin{align*}
    \EE\brk*{\exp\prn*{\alpha \dHam(\cD,\cD')}}=\prn*{1+\PP_\gamma(y\neq y')\cdot (e^\alpha-1)}^T
    \leq \exp\prn*{T\eps},
\end{align*}
where we use the fact that $e^\alpha-1\leq (e-1)\alpha\leq 2\alpha$ for $\alpha\in[0,1]$.
Further, because $\alg$ is \yDP, we know
\begin{align*}
    \PP\sups{\alg}(\pi\in E|\cD')\leq e^{\alpha \dHam(\cD,\cD')}\PP\sups{\alg}(\pi\in E|\cD)+ \frac{e^{\alpha \dHam(\cD,\cD')}-1}{e^\alpha-1}\beta.
\end{align*}
In particular, using $\PP\sups{\alg}(\pi\in E|\cD')\leq 1$ and $\min\crl{1,e^x-1}\leq 2x$, we have
\begin{align*}
    \PP\sups{\alg}(\pi\in E|\cD')-\PP\sups{\alg}(\pi\in E|\cD)\leq&~  \min\crl*{ 1, e^{\alpha \dHam(\cD,\cD')}-1}+ \frac{e^{\alpha \dHam(\cD,\cD')}-1}{e^\alpha-1}\beta \\
    \leq&~ 2\alpha \dHam(\cD,\cD') + \prn*{e^{\alpha \dHam(\cD,\cD')}-1}\cdot \frac{\beta}{\alpha}.
\end{align*}
Therefore, taking expectation over $(\cD,\cD')$, we have
\begin{align*}
    \PP\sups{M,\alg}(\pi\in E)-\PP\sups{(p,\theta),\alg}(\pi\in E)=\EE_{\cD,\cD'}\brk*{ \PP\sups{\alg}(\pi\in E|\cD')-\PP\sups{\alg}(\pi\in E|\cD)}\leq 2T\eps + \prn*{e^{T\eps}-1}\frac{\beta}{\alpha}.
\end{align*}
Finally, by the chain rule of KL divergence, we have
\begin{align*}
    \DTV{\PP\sups{M,\alg},\PP\sups{(p,\otheta),\alg}}^2\leq \frac12\KLd{\PP\sups{M,\alg}}{\PP\sups{(p,\otheta),\alg}}=\frac{T}{2}\KLd{M}{M_{p,\otheta}}\leq T\eps,
\end{align*}
and hence
\begin{align*}
    \abs{\PP\sups{M,\alg}(\pi\in E)-\PP\sups{(p,\otheta),\alg}(\pi\in E)}\leq \sqrt{T\eps}.
\end{align*}
Combining the inequalities above and taking supremum over $E\subseteq \Pi$ complete the proof.
\qed

\section{Proofs from \cref{sec:LDP-linear}}\label{appdx:LDP-linear}

\subsection{Proof of \cref{prop:alg-U-LDP}}\label{appdx:proof-U-LDP}

The following lemma is a standard concentration result (following from taking union bounds with \cref{lem:Gaussian-concen} and \cref{lem:vec-concen}). 

\begin{lemma}\label{lem:spectral-concen-LDP}
In \cref{alg:U-LDP}, \whp, the following inequality holds for all $k=0,\cdots,K-1$:
\begin{align}
    \nrmop{H\kk-\Ep{\usqx[U\kk]}}\leq C\siga\sqrt\frac{d+\log(K/\delta)}{N}=:\epsN,
\end{align}
where $C$ is a large absolute constant. In the following, we denote this event as $\cE$.
\end{lemma}

Therefore, we can simplify the iterations in \cref{alg:U-LDP} as follows: $U\kz=\id$, and for $k=0,1,\cdots,K$:
\begin{align}
    H\kk=&~\Ep{ \usqx[U\kk] }+E\kk, \label{eq:spec-update-H}\\
    \covF\kk=&~U\kk\sq H\kk U\kk\sq+\lambda\kk U\kk, \label{eq:spec-update-cov}\\
    U\kp=&~\sym(\covF\kk\isq U\kk), \label{eq:spec-update-U}
\end{align}
where $E\kk$ is a symmetric matrix.
We note that \cref{alg:U-LDP} does not actually compute $(H\kc, \covF\kc)$, and they only appear in our analysis (where we can regard $E\kc=0$).

\begin{proposition}\label{prop:spec-converge}
Suppose that the sequence of matrices $\sset{ (U\kk, H\kk, \covF\kk) }$ is defined recursively by \eqref{eq:spec-update-H} - \cref{eq:spec-update-U}, with $\nrmop{E\kk}\leq \eps$. Suppose that $\lambda\kk=(2k+1)\eps$, and $\eps\in[0,\frac12]$. Then, for any $k\geq 1$, it holds that
\begin{align*}
    \lmin(\covF\kk) \geq \exp\paren{ -\frac{\log(1/\eps)}{2^{k}} }, \qquad \lmax(\covF\kk)\leq \exp\paren{ \frac{8}{k} }.
\end{align*}
In particular, $K\geq \max\crl{\log\log(1/\eps),12}$, we have
\begin{align*}
    \frac12\id\preceq \Ep{ \uxxu[U\kc] }+\lambda\kc U\kc \preceq 2\id.
\end{align*}
\end{proposition}

\paragraph{Proof of \cref{prop:spec-converge}}\label{appdx:proof-spec-converge}
We further rewrite \eqref{eq:spec-update-cov} and \eqref{eq:spec-update-H} as
\begin{align}\label{eq:spec-update-cov-H}
    \covF\kk=\FLDP[\lambda\kk](U\kk)+U\kk\sq E\kk U\kk\sq.
\end{align}
By definition \cref{eq:spec-update-U}, for each $k\geq 1$, there exists an orthogonal matrix $V\kp$ such that $U\kp=V\kp\covF\kk\isq U\kk$. Therefore, $U\kp U\kk\iv =V\kp \covF\kk$, and hence by \cref{lem:F-monotone}, we have
\begin{align*}
    \covF\kp\succeq \FLDP[\lambda\kk+\eps](U\kp)\succeq &~ \frac{1}{\nrmop{U\kp U\kk\iv}} U\kp U\kk\iv \FLDP[\lambda\kk+\eps](U\kk) U\kk\iv U\kp \\
    =&~ \frac{1}{\nrmop{\covF\kk\isq}} V\kp \covF\kk\isq \prn*{ \FLDP[\lambda\kk](U\kk)+\eps U\kk  } \covF\kk\isq V\kp \\
    =&~\frac{1}{\nrmop{\covF\kk\isq}} V\kp \covF\kk\isq \prn*{ \covF\kk+\eps U\kk - U\kk\sq E\kk U\kk\sq  } \covF\kk\isq V\kp \\
    \succeq&~ \frac{1}{\nrmop{\covF\kk\isq}}\id=\sqrt{\lmin(\covF\kk)}\cdot \id,
\end{align*}
where the first line uses $E\kp\succeq -\eps\id$, the third line uses \eqref{eq:spec-update-cov-H}, and the last line uses the fact that $\EE\kk\preceq \eps\id$. 
Hence, we can conclude that $\lmin(\covF\kp)\geq \sqrt{\lmin(\covF\kk)}$ for $k\geq 0$. Using the fact that $\lmin(\covF\kz)\geq \lambda\kz-\eps\geq \eps$, we have
\begin{align*}
    \lmin(\covF\kk)\succeq \exp\prn*{-\frac{\log(1/\epsilon)}{2^k}}\id, \qquad \forall k\geq 0.
\end{align*}

Similarly, \cref{lem:F-monotone} also implies that
\begin{align*}
    \covF\kp\preceq \FLDP[\lambda\kp+\eps](U\kp)\preceq &~ \nrmop{U\kk U\kp\iv} U\kp U\kk\iv \FLDP[\lambda\kp+\eps](U\kk) U\kk\iv U\kp \\
    =&~\nrmop{\covF\kk\sq} V\kp \covF\kk\isq \prn*{ \FLDP[\lambda\kk](U\kk)+3\eps U\kk  } \covF\kk\isq V\kp \\
    =&~\nrmop{\covF\kk\sq} V\kp \covF\kk\isq \prn*{ \covF\kk+3\eps U\kk - U\kk\sq E\kk U\kk\sq  } \covF\kk\isq V\kp \\
    \preceq&~\nrmop{\covF\kk\sq}\prn*{\id+ 4\eps V\kp \covF\kk\isq U\kk  \covF\kk\isq V\kp} \\
    \preceq&~ \sqrt{\lmax(\covF\kk)}\prn*{1+\frac{4\eps}{\lambda\kk-\eps}}\cdot \id,
\end{align*}
where the last inequality follows from $\covF\kk\succeq \lambda\kk U\kk+ U\kk\sq E\kk U\kk\sq \succeq (\lambda\kk-\eps)U\kk$. 

Therefore, it holds that
\begin{align*}
    \log\lmax(\covF\kp)\leq&~ \frac{4\eps}{\lambda\kp-\eps}+\frac{1}{2}\log\lmax(\covF\kk)\\
    \leq&~ \sum_{j=0}^{k} \frac{1}{2^j}\cdot \frac{4\eps}{\lambda_{k+1-j}-\eps}+\frac{\log\lmax(\covF\kz)}{2^{k+1}}.
\end{align*}
Note that $\lmax(\covF\kz)\leq 1+\eps+\lambda\kz$, and we also have
\begin{align*}
    \frac{1}{2^{k+1}}+\sum_{j=0}^{k} \frac{1}{2^j}\cdot \frac{1}{k+1-j}\leq \frac{1}{2^{k+1}}+\frac{2}{k+1}+\sum_{j=0}^{k} \frac{1}{2^j}\cdot \frac{j}{(k+1)(k+1-j)} \leq \frac{4}{k+1}.
\end{align*}
Therefore, we have proven that as long as $\eps\in[0,\frac12]$,
\begin{align*}
    \log\lmax(\covF\kp)\leq&~ \frac{8}{k+1}.
\end{align*}
This is the desired result.
\qed

\paragraph{Proof of \cref{lem:spectral-concen-LDP}}

In \cref{alg:U-LDP}, the matrix $H\kk$ is given by 
\begin{align*}
    H\kk=\frac1N\sumkn V_t+Z\kk, \qquad V_t=\usqx[U\kk][t],
\end{align*}
where conditional on $U\kk$, $Z\kk$ has i.i.d entries $Z_{ij}=Z_{ji}\sim \normal{0,\frac{4\siga^2}{N}}$, and $\Phi_{kN+1},\cdots, \Phi_{(k+1)N}$ are independent. Further, by \cref{lem:uxxu}. %
Therefore, using \cref{lem:vec-concen}, we have \whp,
\begin{align*}
    \nrmF{\frac1N\sumkn V_t-\Ep{\usqx[U\kk][]}}\leq \frac{1+\sqrt{2\log(1/\delta)}}{\sqrt{N}}
\end{align*}
By \cref{lem:Gaussian-concen}, we also have $\nrmop{Z\kk}\leq C\frac{\siga\sqrt{d+\log(1/\delta)}}{N}$ \whp. Taking the union bound over $k=0,1,\cdots,K-1$ and rescaling $\delta\leftarrow \frac{\delta}{2K}$ completes the proof.
\qed

\paragraph{Proof of \cref{prop:alg-U-LDP}}
By \cref{lem:spectral-concen-LDP}, under the event $\cE$, the matrix $E\kk=H\kk -\Ep{ \usqx[U\kk] }$ is bounded as $\nrmop{E\kk}\leq \epsN$. 
Therefore, under $\cE$, we can apply \cref{prop:spec-converge}, which gives the desired results.
\qed

\subsection{Proof of \cref{lem:L-LDP-solution}}\label{appdx:proof-L-LDP-solution}

We first prove (a). Denote $R\defeq \wh \Psi-\Ex{\frac{U\x \x\tp}{\nrm{U\x}}}$, and then $\Psi U=RU+\Ex{\frac{U\x \x\tp U}{\nrm{U\x}}}$. Then, for any vector $v\in\R^d$,
\begin{align*}
    \nrm{(\Psi+\lambda\id)Uv}\geq \nrm{\prn*{\Ex{\frac{U\x \x\tp U}{\nrm{U\x}}}+\lambda U}v}-\nrm{RUv}\geq \frac{1}{2}\nrm{v}-\nrmop{R}\nrm{Uv},
\end{align*}
where we use $\FLDP(U)=\Ex{\frac{U\x \x\tp U}{\nrm{U\x}}}+\lambda U \succeq \frac12\id$. Then, using $U\preceq \frac{1}{\lambda}\FLDP(U)\preceq  \frac{2}{\lambda}\id$, we know $\nrm{Uv}\leq \frac{2}{\lambda}\nrm{v}$, and hence we have shown that (under the condition $\nrmop{R}\leq \frac{\lambda}{8}$)
\begin{align*}
    \nrm{(\Psi+\lambda\id)Uv}\geq \frac14\nrm{v}, \qquad\forall v\in\R^d.
\end{align*}
Similarly, we can show that
\begin{align*}
    \nrm{(\Psi+\lambda\id)Uv}\leq \frac94\nrm{v}, \qquad\forall v\in\R^d.
\end{align*}
In particular, all the singular values of $(\Psi+\lambda\id)U$ belongs to $[\frac14,\frac94]$, and hence $\Psi+\lambda\id$ is invertible.

Next, we prove (b). We first note that
\begin{align*}
    \EE\brk*{\frac{U\x}{\nrm{U\x}}y}-\EE\brk*{\frac{U\x\x\tp}{\nrm{U\x}}}\cdot \ths=\EE\brk*{\frac{U\x}{\nrm{U\x}}\prn*{\EE[y\mid \x]-\lr \x,\ths\rr}},
\end{align*}
and hence using the conditions,
\begin{align*}
    \nrm{(\wh \Psi+\lambda)\ths-\wh\psi}\leq &~ \nrm{\wh\psi- \EE\brk*{\frac{U\x}{\nrm{U\x}}y}} + \nrm{ \EE\brk*{\frac{U\x\x\tp}{\nrm{U\x}}} - \wh\Psi} +\lambda+\EE\brk*{\abs{\EE[y\mid \x]-\lr \x,\ths\rr}}\leq 2\lambda+\epsapx.
\end{align*}
Now, we define the quadratic function $J(\theta)\ldef \nrmn{\prn{\wh \Psi+\lambda\id}\theta- \wh \psi}^2$. Using $\hth=\argmin_{\theta\in\Bone} J(\theta)$, we have $\lr \nabla J(\hth), \ths-\hth\rr\geq 0$, and hence
\begin{align*}
    J(\ths)-J(\hth)-\nrm{\ths-\hth}_{\nabla^2 J}^2=\lr \nabla J(\hth), \ths-\hth\rr\geq 0.
\end{align*}
Therefore, by definition,
\begin{align*}
    \nrmn{(\wh \Psi+\lambda)(\hth-\ths)}^2\leq J(\ths)-J(\hth)\leq J(\ths)\le \prn*{2\lambda+\epsapx}^2.
\end{align*}
Finally, using $\nrmn{(\wh \Psi+\lambda)v}\geq \frac14\nrm{U\iv v}\forall v\in\R^d$ (from (a)) completes the proof.
\qed

\subsection{Proof of \cref{thm:LDP-linear-regression-full}}\label{appdx:proof-LDP-linear-regression}

We instantiate the subroutine \cref{alg:U-LDP} with $K=\max\crl{\log\log T, 12}$, and
\begin{align*}
    \lambda=C_1K\siga\sqrt\frac{K(d+\log(K/\delta))}{T},
\end{align*}
where $C_1>4C$ is a sufficiently large absolute constant, $C$ is defined in \cref{lem:spectral-concen-LDP}. Then by \cref{prop:alg-U-LDP}, it holds that \whp, \cref{alg:U-LDP} returns a matrix $U$ such that $\frac12\id\preceq \FLDP(U)\preceq 2\id$.

We condition on $U$ and this success event.
By definition, 
\begin{align*}
    \wh \psi=z_\psi+\frac{1}{T}\sum_{t=T+1}^{2T} \frac{U\x_t}{\nrm{U\x_t}}\cdot y_t, \qquad 
    \wh \Psi=Z_\Psi+\frac{1}{T}\sum_{t=T+1}^{2T} \frac{U\x_t \x_t\tp}{\nrm{U\x_t}},
\end{align*}
where $z_\psi\in\R^{d}$, $Z_\Psi\in\R^{d\times d}$ has i.i.d. entries from $\normal{0,\frac{16\siga^2}{T}}$. 

Then, by \cref{lem:Gaussian-concen}, \cref{lem:vec-concen} and \cref{lem:cov-concen}, the following inequalities hold simultaneously \whp:

(1) $\nrm{z_\psi}\leq \frac{\lambda }{4}, \nrmop{Z_\Psi}\leq \frac{\lambda}{16}$.

(2)
\begin{align*}
    \nrm{\frac1T\sum_{t=T+1}^{2T} \frac{U\x_t}{\nrm{U\x_t}}y_t-\EE\brk*{\frac{U\x}{\nrm{U\x}}y}}\leq \frac{\lambda}{4}, \qquad
    \nrmF{\frac{1}{T}\sum_{t=T+1}^{2T} \frac{U\x_t \x_t\tp}{\nrm{U\x_t}}-\Ex{\frac{U\x \x\tp}{\nrm{U\x}}}}\leq \frac{\lambda}{16},
\end{align*}

Therefore, under the success event of \cref{alg:U-LDP}, (1) and (2), the condition of \cref{lem:L-LDP-solution} holds, and the desired results follow.
\qed

\subsection{Proof of \cref{cor:LDP-minimax}}\label{appdx:proof-LDP-minimax-optimal}

\newcommand{\privp}{\mathsf{Priv}'_{\alpha}}

In this section, we derive tighter upper bounds for the setting where the covariate distribution $p$ is known. More generally, the method applies as long as $\Ustar$ can be accurately approximated, e.g., in the setting where the learner has sample access to $p$. 

We first recall that the $(\alpha,0)$-DP channel proposed by \citet{duchi2013local} with variance bound similar to the Gaussian channel.
\begin{lemma}
For any $\alpha\in(0,1]$, $d\geq 1$, there exists a channel $\privp(\cdot)$ such that

(a) For any $v\in\Bone$, it holds that $\EE_{\til v \sim \privp(v)}\brk{\til v}=v$, and
\begin{align*}
    \EE_{\til v \sim \privp(v)}\prn*{\til v-v}\prn*{\til v-v}\tp \preceq \frac{C}{\alpha^2}\id,
\end{align*} 
where $C>0$ is an absolute constant.

(b) For any $F: \cZ\to \Bone$, $z\mapsto \privp(F(z))$ is $(\alpha,0)$-DP.
\end{lemma}

\newcommand{\errlevel}{\min\crl*{\lambda\cdot \lmin(U),1}}

We propose \cref{alg:LDP-linear-regression-fixed-p}, which simplifies \cref{alg:LDP-linear-regression} with knowledge of $p$ and also utilizes the pure private channel $\privp$. Note that for any fixed $\lambda>0$, a matrix $U$ satisfying $\nrmop{\FLDP(U)-\id}\leq \errlevel$ can be computed in polynomial time.\footnote{For example, using the non-private version of \cref{alg:U-LDP}, we can ensure $\nrmop{\FLDP(U)-\id}\leq \errlevel$ with $\poly(B,\lambda^{-1})$ queries from $p$. } Therefore, in the setting where $p$ is known, \cref{alg:LDP-linear-regression-fixed-p} can be implemented in a computation-efficient way. 

\begin{algorithm}
    \caption{Distribution-specific \METHOD (LDP)
    }\label{alg:LDP-linear-regression-fixed-p}
    \begin{algorithmic}[1]
    \REQUIRE Dataset $\cD=\crl*{(\x_t,y_t)}_{t\in[T]}$.
    \STATE Set $\lambda=\frac{1}{\alpha\sqrt{T}}$.
    \STATE Compute matrix $U$ such that $\nrmop{\FLDP(U)-\id}\leq \errlevel$.
    \FOR{$t=1,\cdots,T$}
    \STATE Privatize $\til \psi_t\sim \privp\prn*{\frac{U\x_t}{\nrm{U\x_t}}\cdot y_t}$.
    \ENDFOR
    \STATE Compute
    \begin{align*}
        \wh \psi=\frac{1}{T}\sum_{t=1}^{T} \til \psi_t, \qquad 
        \hth=U \wh\psi.
    \end{align*}
    \ENSURE Estimator $\hth$.
    \end{algorithmic}
\end{algorithm}

\begin{proposition}\label{prop:LDP-linear-regression-upper-fixed-p}
\cref{alg:LDP-linear-regression-fixed-p} preserves $(\alpha,0)$-LDP. Further,
for any PSD matrix $A$, it holds that
\begin{align*}
    \EE\nrmn{\hth-\ths}_A^2\leq  \frac{C_1}{\alpha^2T} \cdot \trd{A}{\Ustar},
\end{align*}
where $\lambda=\frac{1}{\alpha\sqrt{T}}$ and $C_1>0$ is an absolute constant.
\end{proposition}

\paragraph{\pfref{prop:LDP-linear-regression-upper-fixed-p}}
We denote $\psi_t=\frac{U\x_t}{\nrm{U\x_t}}\cdot y_t$ and $z_t=\til \psi_t-\psi_t$. Then, $(\x_t,y_t,z_t)$ are i.i.d, and
\begin{align*}
    \EE[z_t\mid \x_t, y_t]=\psi_t, \qquad \EE[z_t z_t\tp\mid \x_t, y_t]\preceq \frac{C}{\alpha^2}\id.
\end{align*}
Therefore, for any PSD matrix $\Lambda$, it holds that
\begin{align*}
    \EE\nrm{\sum_{t=1}^T \prn*{ \til \psi_t-\EE[\psi_t]} }^2_\Lambda
    =\sum_{t=1}^T \EE\nrm{z_t}^2_\Lambda+\EE\nrm{\psi_t-\EE[\psi_t]}^2_\Lambda
    \leq \frac{CT}{\alpha^2}\tr(\Lambda)+T\nrmop{\Lambda}.
\end{align*}
We also have $\EE[\psi_1]=\prn*{\FLDP(U)U\iv -\lambda \id }\ths$, and hence
\begin{align*}
    \nrm{U\iv \ths-\EE[\psi_1]}=\nrm{\lambda\ths+\prn*{\FLDP(U)-\id}U\iv\ths}\leq \lambda+\nrmop{\FLDP(U)-\id}\nrmop{U\iv}\leq 2\lambda.
\end{align*}
Therefore, setting $\Lambda=UAU$, we have shown that 
\begin{align*}
    \EE\nrmn{ U\wh{\psi}-\ths }^2_A\leq&~ 2\nrm{U\iv\ths-\EE[\psi_1]}^2_\Lambda+\frac{2}{T^2}\EE\nrm{\sum_{t=1}^T \prn*{ \til \psi_t-\EE[\psi_t]} }^2_\Lambda \\
    \leq&~ 8\lambda^2\nrmop{\Lambda}+\frac{2}{T}\prn*{\frac{C}{\alpha^2}\tr(\Lambda)+\nrmop{\Lambda}}.
\end{align*}
This is the desired upper bound, as $U^2\preceq 4\Ustar^2$ by \cref{lem:U-L1-covariance}.
\qed

\section{Proofs from \cref{sec:DP-linear}}\label{appdx:DP-linear}

In this section, we present the proofs from \cref{sec:DP-linear} while considering a slightly more general setting: linear models with misspecification. We formulate the setting with the following assumption.
\begin{assumption}\label{asmp:linear-model-mis}
    The dataset $\cD=\crl*{(\x_t,y_t)}_{t\in[T]}$ is generated as $(\x_t,y_t)\sim M$ i.i.d for $t\in[T]$, such that the following holds:

    (1) Under $(\x,y)\sim M$, it holds that $\nrm{\x}\leq B$ and $\abs{y}\leq 1$ almost surely.
    
    (2) It holds that
    \begin{align*}
        \EE_{\x\sim M}\prn*{ \EE[y|\x]-\lr \ths, \x\rr}^2\leq \epsapx^2. 
    \end{align*}
\end{assumption}
When $\epsapx=0$, this assumption is consistent with our definition of (well-specified) linear models (\cref{def:linear-model}).

\subsection{Privacy analysis of \cref{alg:W-JDP}}\label{appdx:JDP-verify}

We first note that in \cref{alg:W-JDP}, the dataset $\cD$ is split equally as $\cD=\cD\kz\sqcup \cdots \sqcup \cD\epk{K-1}$, and iteration at the $k$th epoch can be regarded as a random function $(U\kk;\cD\kk)\mapsto U\kp$. Therefore, using the composition property of joint DP (\cref{lem:DP-composition}), we only need to verify that $(U\kk;\cD\kk)\mapsto U\kp$ preserves \aJDP~(with respect to $\cD\kk$).

For the data $(\x_t,y_t)$ in the $k$th epoch, the quantity $\Phi_t=\usqx[U\kk][t]$ can be bounded as $\nrmF{\Phi_t}\leq 1$ (\cref{lem:uxxu}). Thus, $H_{(k)}$ defined in \eqref{eq:spectral-approx-JDP} has sensitivity $\Delta=1/N$ under Frobenius norm. Therefore, by the privacy guarantee of Gaussian channels (\cref{def:Guassian-channel}), the mechanism $(U\kk;\dataset\kk)\mapsto \til H\kk$ preserves \aJDP~with respect to $\dataset\kk$. Consequently, by the post-processing property, $(U\kk;\dataset\kk)\mapsto U\kk$ also preserves \aJDP. Therefore, by applying the composition property (\cref{lem:DP-composition}) inductively, we have shown that \cref{alg:JDP-L1-regression} also preserves \aJDP.
\qed

\begin{lemma}[Iterative composition of DP algorithms] \label{lem:DP-composition}
Suppose the algorithm $\alg: \cZ^{N_1+N_2}\to \DPi$ outputs $\pi\sim \alg(z_1,\cdots,z_{N_1+N_2})$ generated as
\begin{align*}
    \pi_1\sim \alg_1(z_1,\cdots,z_{N_1}), \quad
    \pi\sim \alg_2(\pi_1; z_{N_1+1},\cdots,z_{N_1+N_2}),
\end{align*}
where the algorithm $\alg_1:\cZ^{N_1}\to \Delta(\Pi_1)$ preserves \aDP, and for any fixed $\pi_1\in\Pi_1$, $\alg_2(\pi_1;\cdot):\cZ^{N_2}\to \DPi$ preserves \aDP. Then $\alg$ preserves \aDP.
\end{lemma}

\paragraph{Proof of \cref{lem:DP-composition}}

For ease of presentation, we only consider the case both $\Pi$ and $\Pi_1$ are discrete.
For two neighbored dataset $\cD=\{z_i\}_{i=1}^{N_1+N_2}$ and $\cD'=\{z'_i\}_{i=1}^{N_1+N_2}$, denote 
\begin{align*}
    \cD_1=&~  \{z_1,\cdots, z_{N_1}\}, \qquad 
    \cD_2=\{z_{N_1+1},\cdots, z_{N_1+N_2}\}, \\
    \cD_1'=&~  \{z_1',\cdots, z_{N_1}'\}, \qquad 
    \cD_2'=\{z_{N_1+1}',\cdots, z_{N_1+N_2}'\}.
\end{align*}
Assume that $\cD$ and $\cD'$ differ at index $t\in[N_1+N_2]$, i.e., $z_j=z'_j $ for $j\neq t$. We consider two cases.

\paragraph{Case 1: $t\leq N_1$}
Because $\alg_1$ preserves \aDP, we have
\begin{align*}
    \alg_1(E_1|\cD_1)\leq \ea \alg_1(E_1|\cD_1)+\beta, \qquad \forall E_1\subseteq \Pi_1,
\end{align*}
and hence, equivalently, it holds that
\begin{align*}
    \sum_{\pi_1\in\Pi_1} \brac{ \alg_1(\pi_1|\cD_1)-\ea \alg_1(\pi_1|\cD_1') }_+ \leq \beta.
\end{align*}
Note that for any $E\subseteq \Pi$,
\begin{align*}
    \alg(E|\cD)=\sum_{\pi_1\in\Pi_1} \alg_1(\pi_1|\cD_1) \cdot \alg_2(E|\pi_1;\cD_2),
\end{align*}
and therefore,
\begin{align*}
    \alg(E|\cD)-\ea\alg(E|\cD')=&~\sum_{\pi_1\in\Pi_1} \brac{ \alg_1(\pi_1|\cD_1) - \ea \alg_1(\pi_1|\cD_1') }\cdot \alg_2(E|\pi_1;\cD_2) \\
    \leq&~ \sum_{\pi_1\in\Pi_1} \brac{ \alg_1(\pi_1|\cD_1) - \ea \alg_1(\pi_1|\cD_1') }_+\leq \beta,
\end{align*}
where we use the fact that $\alg_2(E|\pi_1;\cD_2)\in[0,1]$ for any $\pi_1\in\Pi_1, E\subseteq \Pi$.

\paragraph{Case 2: $t>N_1$}
In this case, because $\alg_2(\pi_1;\cdot)$ preserves \aDP~for any $\pi_1\in\Pi_1$, for any $E\subseteq \Pi$, we have
\begin{align*}
    \alg(E|\cD)=&~\sum_{\pi_1\in\Pi_1} \alg_1(\pi_1|\cD_1) \cdot \alg_2(E|\pi_1;\cD_2) \\
    \leq&~\sum_{\pi_1\in\Pi_1} \alg_1(\pi_1|\cD_1) \cdot \brac{ \ea\alg_2(E|\pi_1;\cD_2')+\beta }\\
    =&~ \ea \alg(E|\cD')+\beta.
\end{align*}

Combining the inequalities above from both cases completes the proof.
\qed

\subsection{Proof of \cref{prop:JDP-W}}\label{appdx:proof-JDP-W}

We first state the following concentration result, which follows immediately from \cref{lem:Gaussian-concen} and \cref{lem:cov-concen}.

\newcommand{\Hs}{H^\star}
\begin{lemma}\label{lem:W-JDP-concen-Bern}
Suppose that \cref{asmp:linear-model-mis} holds, and the sequence $\set{(W\kk, \til H\kk)}$ is generated by \cref{alg:W-JDP}. 
For each $k=0,1,\cdots,K-1$, we define
\begin{align*}
    \Hs\kk\defeq \Ep{\wsqx[W\kk]}.
\end{align*}
Then, for any fixed parameter $c>1$, \whp, the following inequality holds for all $k=0,\cdots,K-1$:
\begin{align}\label{eq:U-JDP-concen-Bern}
    c\iv \cdot \Hs\kk-\epsN\id\preceq \til H\kk\preceq c\cdot \Hs\kk+\epsN\id,
\end{align}
where $\epsN=C_0\paren{\frac{B\log(dK/\delta)}{(c-1)\gamma N}+\frac{\siga B\sqrt{d+\log(K/\delta)}}{\gamma N}}$, and $C_0$ is a large absolute constant. In the following, we denote this event as $\cE$ and condition on $\cE$.
\end{lemma}

Therefore, following the analysis from \cref{prop:spec-converge}, we prove the following result.

\begin{proposition}\label{prop:spec-converge-JDP}
Fix $\lambda,\gamma,\delta\in(0,1)$. Let $c>1$ be a constant. Suppose that \cref{asmp:linear-model-mis} holds, and \cref{alg:W-JDP} is instantiated with $\lambda\geq \frac{c^2+1}{c^2-1}\epsN$, where $\epsN$ is defined in \cref{lem:W-JDP-concen-Bern}. Then, under the success event $\cE$ of \cref{lem:W-JDP-concen-Bern}, for any $k\geq 1$, it holds that
\begin{align*}
    \lmin(\covF\kk) \geq c^{-4}\exp\paren{ -\frac{\log(1/\lambda)}{2^{k}} }, \qquad \lmax(\covF\kk)\leq c^4 \exp\paren{ \frac{\lambda}{2^{k}} }.
\end{align*}
In particular, for $c=1.1$, $\epsN\in[0,1]$, $K\geq \max\sset{\log\log(1/\lambda),4}$, we have
\begin{align*}
    \frac12\id\preceq \FJDP(W\kc)=\Ep{ \wxxw[W\kc] }+\lambda W\kc \preceq 2\id.
\end{align*}
\end{proposition}

\paragraph{Proof of \cref{prop:spec-converge-JDP}}
In the following proof, we abbreviate $\eps\defeq \epsN$ and we generalize the argument in \cref{appdx:proof-spec-converge} to allow multiplicative error.

Recall that we can simplify the iterations in \cref{alg:W-JDP} as follows: $W\kz=\id$, and for $k=0,1,\cdots,K-1$:
\begin{align*}
    \covF\kk=&~W\kk\sq \til H\kk W\kk\sq+\lambda W\kk, \qquad
    W\kp=\sym(\covF\kk\isq W\kk),
\end{align*}
and we regard $\til H\kc=\Hs\kc$. Then, we have the following facts:

(1) By \cref{lem:W-JDP-concen-Bern}, for each $k\geq 0$, we have
\begin{align*}
    \covF\kk\preceq cW\kk\sq \Hs\kk W\kk\sq+(c\eps+\lambda) W\kk\preceq c\cdot \FJDP(W\kk),
\end{align*}
where we use $c\eps+\lambda\leq c\lambda$.
Similarly,
\begin{align*}
    \covF\kk\succeq c\iv W\kk\sq \Hs\kk W\kk\sq+(-\eps+\lambda) W\kk\succeq c\iv\cdot \FJDP(W\kk),
\end{align*}
where we use $\lambda-\eps\geq c\iv \lambda$. Therefore, it holds that
\begin{align}\label{eq:proof-spec-W-concen}
    c\iv\cdot \FJDP(W\kk) \preceq \covF\kk\preceq c\cdot \FJDP(W\kk), \qquad \forall k\geq 0.
\end{align}

(2)
For each $k\geq 0$, there exists an orthogonal matrix $V\kp$ such that $W\kp=V\kp\covF\kk\isq W\kk$.

By \cref{lem:F-monotone}, we have
\begin{align*}
    \FJDP(W\kp)\succeq &~ \frac{1}{\max\crl{ \nrmop{W\kp W\kk\iv}, 1} } W\kp W\kk\iv \FJDP(W\kk) W\kk\iv W\kp \\
    =&~ \frac{1}{\max\crl{ \nrmop{\covF\kk\isq}, 1} } V\kp \covF\kk\isq \FJDP(W\kk) \covF\kk\isq V\kp\tp  \\
    \succeq &~\frac{1}{\max\crl{ \nrmop{\covF\kk\isq}, 1}} V\kp \covF\kk\isq \cdot c\iv \covF\kk\cdot  \covF\kk\isq V\kp \\
    =&~ c\min\crl*{ \sqrt{\lmin(\covF\kk)}, 1} V\kp V\kp\tp  \\
    =&~ c\min\crl*{ \sqrt{\lmin(\covF\kk)}, 1}\cdot \id,
\end{align*}
where the third line uses \eqref{eq:proof-spec-W-concen}, and the last line follows from the orthogonality of $V\kp$.
Similarly, \cref{lem:F-monotone} also implies that
\begin{align*}
    \FJDP(W\kp)\preceq &~ \nrmop{W\kk W\kp\iv} W\kp W\kk\iv \FJDP(W\kk) W\kk\iv W\kp \\
    =&~\max\crl{ \nrmop{\covF\kk\sq}, 1} V\kp \covF\kk\isq \FJDP(W\kk) \covF\kk\isq V\kp\tp \\
    \preceq&~\max\crl{ \nrmop{\covF\kk\sq}, 1}V\kp \covF\kk\isq \cdot c\covF\kk \cdot \covF\kk\isq V\kp\tp\\
    =&~ c\max\crl*{ \nrmop{\covF\kk\sq},1} V\kp V\kp\tp \\
    =&~ c\max\crl*{ \sqrt{\lmax(\covF\kk)},1 }\cdot \id,
\end{align*}
where the third line follows from \cref{eq:proof-spec-W-concen}.

Therefore, using \cref{eq:proof-spec-W-concen} again, we can conclude that for $k\geq 0$, it holds that
\begin{align*}
    c^{-2} \min\crl*{\sqrt{ \lmin(\covF\kk) },1}\leq \lmin(\covF\kp)\leq \lmax(\covF\kp)\leq c^2 \max\crl*{\sqrt{ \lmax(\covF\kk) },1}.
\end{align*}
Using this inequality recursively, we then have
\begin{align*}
    \min\crl*{\lmin(\covF\kk),1}\geq c^{-4}\min\crl*{c^4\lmin(\covF\kz),1}^{2^{-k}}, \qquad
    \max\crl*{\lmax(\covF\kk),1}\leq c^{4}\max\crl*{c^{-4}\lmax(\covF\kz),1}^{2^{-k}}.
\end{align*}
The desired conclusion follows by recalling that we regard $\til H\kc=\Hs\kc$ and hence
\begin{align*}
    \covF\kc=\Ep{ \wxxw[W\kc] }+\lambda W\kc,
\end{align*}
and we also have $\lmin(\covF\kz)\geq \lambda\geq \eps$ and $\lmax(\covF\kz)\leq c+\lambda+\eps\leq c(1+\lambda)$. 
\qed

\subsection{Proof of \cref{thm:JDP-linear-regression-full}}\label{appdx:proof-JDP-linear-regression}

In this section, we analyze \cref{alg:JDP-linear-regression} under the following parameter regime:
\begin{align}\label{eq:JDP-linear-regression-parameters}
K=\max\crl{\log\log T, 4}, \qquad \lambda\gamma \geq \frac{CK\prn{\siga \sqrt{d+\log(K/\delta)}+\log(K/\delta)}}{T},
\end{align}
where $C$ is a sufficiently large absolute constant. Then by \cref{prop:JDP-W}, it holds that \whp, \cref{alg:W-JDP} returns a matrix $W$ such that $\frac12\id\preceq \FJDP(W)\preceq 2\id$.

We summarize the property of $\hthNW=\hth$ in the following lemma. Here, (3) makes use of the sub-Gaussian property of the residual $\hth-\ths$, which is useful for deriving tight regret bound for linear contextual bandits (\cref{sec:cb}, and specifically \cref{appdx:proof-regret-upper-JDP-better}).
\begin{lemma}\label{lem:hth-JDP-property}
Suppose that \cref{asmp:linear-model-mis} holds. Then, condition on the event $\{$ $(W,\lambda)$ returned by the subroutine $\JDPLU$ satisfies $\frac12\id\preceq \FJDP(W)\preceq 2\id$ $\}$, the following holds \whp:

(1) $ \nrmn{W\iv(\hth-\ths)}\leq 8\lambda+8\sqrt{d}\cdot \epsapx+\bigO{\sqrt{\frac{d\log(1/\delta)}{T}}}$.

(2) For any fixed PSD matrix $A$, it holds that \whp,
\begin{align*}
    \nrmn{\hth-\ths}_A
    \leq 8\prn*{\lambda+\sqrt{d}\cdot \epsapx}\nrmop{A\sq W}+\bigO{\sqrt{\frac{\log(1/\delta)}{T}}}\cdot \nrmF{A\sq W}.
\end{align*}

(3) Suppose that $Q$ is a distribution of projection matrix $\Prjm\in\PSD$ (i.e., $\Prjm^2=\Prjm$). Then \whp, it holds that
\begin{align*}
    \PP_{\Prjm\sim Q}\prn*{\nrmn{\Prjm W\iv(\hth-\ths)}\geq 8\prn{\lambda+\sqrt{d}\cdot \epsapx}+16\sqrt{\frac{\rank(\Prjm)\log(2/\delta)}{T}}}\leq \delta.
\end{align*}
\end{lemma}

\paragraph{Proof of \cref{lem:hth-JDP-property} (1)}
We condition on $\cD_0$ and the event that $\frac12\id\preceq \FJDP(W)\preceq 2\id$, and we denote $N=\frac{T}{2}$. In the following proof, we only consider the randomness of the dataset $\cD_1=\set{(\x_t,y_t)}_{t=N+1,\cdots,T}$.
By definition, 
\begin{align*}
\til \psi=z_\psi+\frac{1}{N}\sum_{(\x,y)\in \cD_1}\frac{W\x}{1+\gamma\nrm{W\x}}\cdot y, \qquad
    \til \Psi=Z_\Psi + \frac{1}{N}\sum_{(\x,y)\in \cD_1} \frac{W\x\x\tp}{1+\gamma\nrm{W\x}},
\end{align*}
where $z_\psi\in\R^{d}$, $Z_\Psi\in\R^{d\times d}$ has i.i.d. entries from $\normal{0,\frac{4\siga^2}{(\gamma N)^2}}$. 

First, applying \cref{lem:cov-concen} to the random matrices $\crl*{\frac{W\sq\x_t\x_t\tp W\sq}{1+\gamma\nrm{W\x_t}}}_{t=N+1,\cdots,T}$, we have:

(a) \Whp[\frac{\delta}{4}], it holds that
\begin{align*}
    \frac12\Ex{\frac{W\sq\x\x\tp W\sq}{1+\gamma\nrm{W\x}}}-\frac{\lambda}{16}\id\preceq \frac{1}{N}\sum_{t=N+1}^T \frac{W\sq\x_t\x_t\tp W\sq}{1+\gamma\nrm{W\x_t}}\preceq 2\Ex{\frac{W\sq \x\x\tp W\sq}{1+\gamma\nrm{W\x}}}+\frac{\lambda}{16}\id.
\end{align*}

Next, by \cref{lem:Gaussian-concen}, we have:

(b) \Whp[\frac\delta4], $\nrm{z_\psi}\leq \frac{\lambda }{16}, \nrmop{Z_\Psi}\leq \frac{\lambda}{16}$.

Denote $\Phi=\frac{1}{N}\sum_{t=N+1}^T \frac{W \x\x\tp W }{1+\gamma\nrm{W\x}}$. Under (a), we know \begin{align*}
    \Phi\succeq \frac12\Ex{\frac{W\x\x\tp W}{1+\gamma\nrm{W\x}}}-\frac{\lambda}{2}W\succeq \frac12\id-\lambda W.
\end{align*}
Further, we know $(\til \Psi +\lambda\id)W=\Phi+\lambda W+Z_\Psi W$, and hence for any $v\in\R^d$,
\begin{align*}
    \nrm{(\til \Psi +\lambda\id)Wv}\geq \nrm{\prn{\Phi+\lambda W}v}-\nrmop{Z_\Psi}\nrm{Wv} \geq \frac{1}{2}\nrm{v}-\frac{\lambda}{8}\nrm{Wv}\geq \frac{1}{4}\nrm{v},
\end{align*}
where we use $W\preceq \frac{2}{\lambda}\id$.
Therefore, $\til \Psi +\lambda\id$ is invertible, and we also have $\nrmop{W\iv (\til \Psi +\lambda\id)\iv}\leq 4$.
Furthermore, by definition,
\begin{align*}
    (\til \Psi+\lambda)(\hth-\ths)=\til \psi-(\til \Psi+\lambda)\ths
    =&~ z_\psi-Z_\Psi\ths-\lambda\ths+\frac{1}{N}\sum_{t=N+1}^T \frac{W\x_t}{1+\gamma\nrm{W\x_t}}\cdot \prn*{y_t-\lr \x_t, \ths\rr}.
\end{align*}
To simplify presentation, we denote $w(\x)\defeq \frac{W\x}{1+\gamma\nrm{W\x}}$, and
\begin{align*}
    \Delta_0\defeq \frac{1}{N}\sum_{t=N+1}^T w(\x_t)\cdot \prn*{y_t-\EE[y_t|\x_t]}, \qquad
    \Delta_1\defeq \frac{1}{N}\sum_{t=N+1}^T w(\x_t)\cdot \prn*{\EE[y_t|\x_t]-\lr \x_t, \ths\rr}.
\end{align*}
Then it holds that %
\begin{align*}
    \nrm{W\iv(\hth-\ths)}\leq \frac32\lambda\nrmop{W\iv (\til \Psi +\lambda\id)\iv}+\nrm{W\iv (\til \Psi +\lambda\id)\iv (\Delta_0+\Delta_1)}
    \leq 6\lambda + 4\nrm{\Delta_0}+4\nrm{\Delta_1}.
\end{align*}

In the following, we upper bound $\nrm{\Delta_0}$ and $\nrm{\Delta_1}$.
 Note that $\nrm{w(\x)}\leq \frac{1}{\gamma}$ and
\begin{align*}
    \Ex{\nrm{w(\x)}^2}=\Ex{\frac{\nrm{W\x}^2}{\prn*{1+\gamma\nrm{W\x}}^2}}
    \leq \Ex{\frac{\nrm{W\x}^2}{1+\gamma\nrm{W\x}}}\leq \tr\prn*{\FJDP(W)}\leq 2d.
\end{align*}
Hence, applying \cref{lem:vec-concen} to the zero-mean random vectors $\crl*{w(\x_t)\cdot (y_t-\EE[y_t|\x_t])}$, we have the following bound.

(c) \Whp[\frac{\delta}{4}], it holds that
\begin{align*}
    \nrm{\Delta_0}=\nrm{\frac{1}{N}\sum_{t=N+1}^T w(\x_t)\cdot \prn*{y_t-\EE[y_t|\x_t]}}\leq \sqrt{\frac{4d\log(8/\delta)}{N}}+\frac{2\log(8/\delta)}{\gamma N}.
\end{align*}

Furthermore, we can bound $\nrm{\Delta_1}$ as
\begin{align*}
    \nrm{\Delta_1}\leq \frac{1}{N}\sum_{t=N+1}^T \nrm{w(\x_t)}\cdot \abs{\EE[y_t|\x_t]-\lr \x_t, \ths\rr}.
\end{align*}
Hence, using the fact that $\nrm{w(\x_t)}\leq \frac{1}{\gamma}$ and $\abs{\EE[y_t|\x_t]-\lr \x_t, \ths\rr}\leq B+1\leq 2B$ almost surely, we can apply \cref{lem:cov-concen} to get the following bound.

(d) \Whp[\frac{\delta}{4}], it holds that
\begin{align*}
    \nrm{\Delta_1}\leq \frac{1}{N}\sum_{t=N+1}^T \nrm{w(\x_t)}\cdot \abs{\EE[y_t|\x_t]-\lr \x_t, \ths\rr} \leq 2\EE\brk*{ \nrm{w(\x)}\cdot \abs{\EE[y|\x]-\lr \x, \ths\rr} } + \frac{4B\log(8/\delta)}{\gamma N}.
\end{align*}

Note that $\frac{4B\log(8/\delta)}{\gamma N}\leq \frac{\lambda}{4}$ by our choice of $\lambda$, and
\begin{align*}
    \EE\brk*{ \nrm{w(\x)}\cdot \abs{\EE[y|\x]-\lr \x, \ths\rr} }
    \leq&~ \sqrt{\EE\brk*{ \nrm{w(\x)}^2 } \cdot \EE \abs{\EE[y|\x]-\lr \x, \ths\rr} ^2}
    \leq \sqrt{2d}\cdot \epsapx.
\end{align*}
Therefore, under (d), we have $\nrm{\Delta_1}\leq \sqrt{2d}\cdot \epsapx+\frac{\lambda}{4}$.
Combining the inequalities above gives the desired upper bound.
\qed

\paragraph{Proof of \cref{lem:hth-JDP-property} (2)}
We follow the proof above and condition on the success event (a), (b) and (d). We first note that under the success event (a), (b) and (d), it holds that $\nrmop{W\iv (\til \Psi +\lambda\id)\iv}\leq 4$, and
\begin{align*}
    \nrm{(\til \Psi+\lambda)(\hth-\ths)-\Delta_0}\leq \frac{3}{2}\lambda+\sqrt{2d}\cdot \epsapx.
\end{align*}
Therefore, it holds that
\begin{align*}
    \nrm{\hth-\ths}_A
    \leq \prn*{6\lambda+8\sqrt{d}\cdot \epsapx}\nrmop{A\sq W}+\nrm{A\sq (\til \Psi+\lambda)\iv \Delta_0}.
\end{align*}
In the following, we denote $\Lambda\defeq A\sq (\til \Psi +\lambda\id)\iv$, and we condition on $\cD_\x=\set{\x_t}_{t=N+1,\cdots,T}$, i.e., we only account the randomness of $\cD_y=\set{y_t}_{t=N+1,\cdots,T}$. Note that conditional on $\cD_\x$, the random variables $\set{y_t-\EE[y_t|\x_t]}_{t=N+1,\cdots,T}$ are independent and zero-mean. Therefore, to apply Bernstein's inequality for random vectors, we consider the following sequence of independent random vectors:
\begin{align*}
    X_t=\Lambda\cdot \frac{W\x_t}{1+\gamma\nrm{W\x_t}}\cdot \prn*{y_t-\EE[y_t|\x_t]}, \qquad t=N+1,\cdots,T.
\end{align*}
By definition, $\nrm{X_t}\leq \frac{2\nrmop{\Lambda}}{\gamma}$ almost surely, and
\begin{align*}
    \sum_{t=N+1}^T \EE\nrm{X_t}^2
    \leq \sum_{t=N+1}^T \frac{\nrm{\Lambda W\x_t}^2}{(1+\gamma\nrm{W\x_t})^2}
    \leq \tr\prn*{ \Lambda \sum_{t=N+1}^T \frac{W\x_t \x_t \tp W}{1+\gamma\nrm{W\x_t}} \Lambda\tp  }
    \leq 4N\tr\prn*{\Lambda \Lambda\tp}=4N\nrmF{\Lambda}^2,
\end{align*}
where we use (a): $\frac1N\sum_{t=N+1}^T \frac{W\x_t \x_t \tp W}{1+\gamma\nrm{W\x_t}}\preceq 2\FJDP(W)\preceq 4\id$.
\qed

\paragraph{\pfref{lem:hth-JDP-property} (3)}

We extend the argument of the proof of \cref{lem:hth-JDP-property} (b).
Under the success event (a), (b) and (d), we know that for any projection matrix $\Prjm$, it holds that %
\begin{align*}
    \nrm{\Prjm W\iv(\hth-\ths)}\leq  6\lambda + 8\sqrt{d}\epsapx +\nrm{\Prjm W\iv (\til \Psi +\lambda\id)\iv \Delta_0}.
\end{align*}

In the following, we denote $\Lambda\defeq \Prjm W\iv (\til \Psi +\lambda\id)\iv$, and we again condition on $W$ and $\cD_\x=\set{\x_t}_{t=N+1,\cdots,T}$ and only account the randomness of $\cD_y=\set{y_t}_{t=N+1,\cdots,T}$. Note that conditional on $\cD_\x$, the random variables $\set{y_t-\EE[y_t|\x_t]}_{t=N+1,\cdots,T}$ are independent and zero-mean. 
By the proof of \cref{lem:hth-JDP-property} (b), for any fixed projection matrix $\Prjm$, it holds that
\begin{align*}
    \PP_{\cD_y\mid W, \cD_x}\prn*{ \nrm{\Lambda_{\Prjm} \Delta} \geq \sqrt{\frac{16\nrmF{\Lambda_{\Prjm}}^2\log(2/\delta)}{N}}+\frac{4\nrmop{\Lambda_{\Prjm}} \log(2/\delta)}{\gamma N} }\leq \delta^2,
\end{align*}
where $\PP_{\cD_y}$ is the conditional probability of $\cD_y$ on $W$ and $\cD_x$, and $\Lambda_{\Prjm}\defeq \Prjm W\iv (\til \Psi +\lambda\id)\iv$. Note that $\nrmop{\Lambda_{\Prjm}}\leq \nrmop{W\iv(\til \Psi +\lambda\id)\iv}\le 4$ and hence $\nrmF{\Lambda_{\Prjm}}^2\leq 16\rank(\Lambda_{\Prjm})\leq 16\rank(\Prjm)$. 
Further, by our choice of $\lambda$, we can ensure that $\frac{\lambda}{4}\geq \frac{16\log(2/\delta)}{\gamma N}$, and hence we have shown that for any fixed projection matrix $\Prjm$, it holds that
\begin{align*}
    \PP_{\cD_y}\prn*{ \nrm{\Prjm  W\iv (\til \Psi +\lambda\id)\iv \Delta_0} \geq 16\sqrt{\frac{\rank(\Prjm)\log(2/\delta)}{T}}+ \frac{\lambda}{4} }\leq \delta^2.
\end{align*}
Now, taking expectation over $\Prjm\sim Q$, we have
\begin{align*}
    \EE_{\cD_y} \PP_{\Prjm\sim Q}\prn*{ \nrm{\Prjm  W\iv (\til \Psi +\lambda\id)\iv \Delta_0} \geq 16\sqrt{\frac{\rank(\Prjm)\log(2/\delta)}{T}}+ \frac{\lambda}{4} } \leq \delta^2.
\end{align*}
Applying Markov inequality completes the proof.
\qed

\subsection{Proof of \cref{prop:JDP-minimax-optimal}}\label{appdx:proof-JDP-minimax-optimal}

Similar to the local model (\cref{appdx:proof-LDP-minimax-optimal}), we present an variant of our \method~method (\cref{alg:JDP-linear-regression-fixed-p}) that better utilizes distribution information in the setting where covariate distribution $p$ is given. We note that similar to the LDP setting, \cref{prop:JDP-linear-regression-fixed-p} can also be implemented in polynomial time.

\newcommand{\errlevelJDP}{\min\crl*{\lambda\cdot\lmin(W),1}}

\begin{algorithm}
    \caption{Distribution-specific \Method~(DP)
    }\label{alg:JDP-linear-regression-fixed-p}
    \begin{algorithmic}[1]
    \REQUIRE Dataset $\cD=\crl*{(\x_t,y_t)}_{t\in[T]}$, covariate distribution $p$.
    \STATE Set $\lambda=\frac{1}{\sqrt{T}}$, $\gamma=\ogammaval$.
    \STATE Compute matrix $W$ such that $\nrmop{\FJDP(W)-\id}\leq \errlevelJDP$.
    \STATE Compute
    \begin{align*}
        \psi=\frac{1}{T}\sum_{t=1}^{T} \frac{W\x_t}{1+\gamma\nrm{W\x_t}}y_t.
    \end{align*}
    \STATE Privatize $\til \psi\sim \priv[1/(\gamma T)]{\psi}$.
    \ENSURE Estimator $\hth=W\til \psi$.
    \end{algorithmic}
\end{algorithm}

\begin{proposition}\label{prop:JDP-linear-regression-fixed-p}
\cref{alg:JDP-linear-regression-fixed-p} is \aDP, and for any PSD matrix $A$, it holds that 
\begin{align*}
    \EE\nrmn{\hth-\ths}_A^2
    \leq \frac{C}{T}\trd{A}{\Wstar},
\end{align*}
where $C>0$ is an absolute constant, and $\lambda=\frac{1}{\sqrt{T}}$, $\gamma=\ogammaval$ are chosen in \cref{alg:JDP-linear-regression-fixed-p}.
\end{proposition}

\paragraph{\pfref{prop:JDP-linear-regression-fixed-p}}
By definition, $\til\psi=\psi+Z$, where $Z\sim \normal{0,\frac{\siga^2}{(\gamma T)^2}}$ is independent of the dataset $\cD$. Therefore, it holds that
\begin{align*}
    \EE\nrm{W\til \psi-W\EE[\psi]}_A^2
    =\EE\nrm{WZ}_A^2+\EE\nrm{W\psi-W\EE[\psi]}_A^2
    =\frac{\siga^2}{(\gamma T)^2}\tr(WAW)+\EE\nrm{\psi-\EE[\psi]}_{WAW}^2
\end{align*}
In the following, we denote $\Lambda=WAW$.
Using the fact that $\cD=\crl*{(\x_t,y_t)}_{t\in[T]}$ are generated i.i.d from a linear model $M$, 
we have %
\begin{align*}
    \EE\nrm{\psi-\EE[\psi]}_{\Lambda}^2
    =&~ \frac1{T^2}\sum_{t=1}^T \EE\nrm{\frac{W\x_t}{1+\gamma\nrm{W\x_t}}y_t-\EE[\psi]}_{\Lambda}^2 \\
    \leq&~ \frac1T \EE_{(\x,y)\sim M}\nrm{\frac{W\x}{1+\gamma\nrm{W\x}}y}_{\Lambda}^2 
    \leq \frac1T \tr\prn*{\Lambda \EE_{(\x,y)\sim M}\frac{W\x\x\tp W}{(1+\gamma\nrm{W\x})^2} } \\
    \leq&~ \frac1T \tr(\Lambda\cdot (\FJDP(W)-\lambda W)) \leq \frac2T\tr(\Lambda),
\end{align*}
where the last line uses the fact that the marginal distribution of  $\x\sim M$ is $p$, and $\FJDP(W)\preceq 2\id$.
Further, using $\EE_{M}[y\mid \x]=\lr \x,\ths\rr$ almost surely under $\x\sim p$, we know
\begin{align*}
    \EE[\psi]=\EE_{\x\sim p}\brk*{\frac{W\x}{1+\gamma\nrm{W\x}}\lr \x,\ths\rr}
    =(\FJDP(W)-\lambda W)W\iv\ths=\FJDP(W)W\iv\ths-\lambda\ths,
\end{align*}
and hence
\begin{align*}
    \nrm{\EE[\psi]-W\iv\ths}
    \leq \nrm{(\FJDP(W)-\id)W\iv\ths}+\lambda\nrm{\ths}
    \leq \nrmop{\FJDP(W)-\id}\nrmop{W\iv}+\lambda \leq 2\lambda.
\end{align*}
In particular, $\nrm{\EE[\psi]-W\iv\ths}_\Lambda\leq 2\lambda\nrmop{\Lambda}$, and we can conclude that
\begin{align*}
    \EE\nrm{\hth-\ths}_A^2
    \leq 2\EE\nrm{W\til \psi-W\EE[\psi]}_A^2+2\nrm{W\EE[\psi]-\ths}^2
    \leq 2\prn*{\frac{\siga^2}{(\gamma T)^2}+\frac{2}{T}}\tr(\Lambda)+8\lambda^2\nrmop{\Lambda}.
\end{align*}
Combining this inequality with $W^2\preceq 4\Wstar^2$ (\cref{lem:Wstar-properties}) gives the desired upper bound.
\qed

\paragraph{\pfref{prop:JDP-minimax-optimal}}
\newcommand{\Rbar}{\wh{R}}
For notational simplicity, we denote
\begin{align*}
    R(T)\ldef T\cdot \inf_{\DPtag}~~\sup_{\ths\in\Ball(1/B)}\EE\sups{\ths,\alg} \nrmn{\hth-\ths}_A^2, \qquad \forall T\geq 1,
\end{align*}
and $\Rbar(\gamma,\lambda)\ldef \trd{A}{\Wstar}$.
Then, by \cref{appdx:proof-JDP-minimax-optimal} and \cref{lem:JDP-lower-bound-any}, there are absolute constants $C>c>0$ such that
\begin{align*}
    c\cdot \Rbar(\ugamma(T),\ulambda(T))\leq R(T) \leq C\cdot \Rbar(\ogamma(T),\olambda(T)),
\end{align*}
where
\begin{align*}
    \ugamma(T)=\ogammaval, \quad \ulambda(T)=\frac{1}{\sqrt{T}}, \qquad
    \ogamma(T)=\frac{1}{\alpha\sqrt{dT}}, \qquad \olambda(T)=\frac{B\sqrt{d}}{\sqrt{T}}.
\end{align*}
In particular, for any fixed $T\geq 1$, we consider $D_T=[\ogamma(T),\ugamma(T)]\times [\ulambda(T),\olambda(T)]$, and the following sets
\begin{align*}
    D_T^+\ldef \crl*{(\lambda,\gamma)\in D_T: \Rbar(\gamma,\lambda)\geq \frac{1}{C}R(T)}, \qquad
    D_T^-\ldef \crl*{(\lambda,\gamma)\in D_T: \Rbar(\gamma,\lambda)\leq \frac{1}{c}R(T)}.
\end{align*}
Then, by definition, $D_T^+\cup D_T^-=D_T$, and both of $D_T^+$ and $D_T^-$ are non-empty.
Therefore, by the continuity of $(\gamma,\lambda)\mapsto \Wstar$, we have $D_T^+\cap D_T^-\neq \emptyset$ is non-empty.
Hence, there exists function $\gammas(T), \lambdas(T)$ such that $\ugamma(T)\leq \gammas(T)\leq \ogamma(T)$ and $\ulambda(T)\leq \lambdas(T)\leq \olambda(T)$, and 
\begin{align*}
    c\cdot \Rbar(\gammas(T),\lambdas(T))\leq R(T) \leq C\cdot \Rbar(\gammas(T),\lambdas(T)).
\end{align*}
This is the desired result.
\qed

\section{Private Generalized Linear Regression}\label{appdx:GLM}
\newcommand{\wstar}{w^\star}
\newcommand{\hwst}{\widehat{w}^\star}

In this section, we study private regression in the generalized linear models (GLMs), generalizing the ideas of our \method~method. For simplicity, we consider the well-specified generalized linear models.

\begin{definition}\label{def:GLM}
    A generalized linear model (GLM) with covariate distribution $p$, parameter $\ths$ and link function $\nu$ is a distribution $M\in\Delta(\R^d\times \R)$ such that
    \begin{align*}
        (\x,y)\sim \Mstar: \quad \x\sim \pph, ~~
        \EE[y|\x]=\link(\lr \x,\ths\rr).
    \end{align*}
    We assume the link function $\nu:\R\to \R$ is known and satisfies $\nu'(x)\geq \mug>0$ for all $x\in\R$.
\end{definition}

For generalized linear models with link function $\link$, the following loss objective is typically considered:
\begin{align*}%
\textstyle
    \Lgl(\theta)\defeq \Exy{ \ellg(\lr \x,\theta\rr, y) },
\end{align*}
where the \emph{integral loss} $\ellg$ associated with $\link$ is defined as $\ellg(t,y)\defeq -yt+\int_{0}^t \link(s)ds$. The basic property of $\Lgl$ is that $\nabla \Lgl(\theta)=\Exy{ \paren{\link(\lr \x,\theta \rr) -y }\cdot \x }$,
and hence $\nabla \Lgl(\ths)=0$, i.e., $\ths$ is a global minimizer of $\Lgl$. %

\newcommand{\constrth}{\nrm{\theta}\leq 1}
\newcommand{\constrw}{\nrm{Uw}\leq 1}
\newcommand{\cLbarDP}{\widebar{\cL}_{W,\gamma,\lambda}}
\newcommand{\cLbarLDP}{\widebar{\cL}_{U,\lambda}}
\paragraph{Reweighted objective for LDP setting}
Under the DP setting, 
given a matrix $U\in\PSD$ and parameter $\lambda>0$ such that $\frac12\id \preceq \FLDP(U)\preceq 2\id$, we introduce
\begin{align}\label{def:cL-LDP-GLM}
    \cLLDPg(\theta)\defeq \EE\brk*{ \frac{\ellg(\lr \x,\theta\rr, y)}{\nrm{U\x}} }+\frac{\mug\lambda}{2} \nrm{\theta}^2_{U\iv},
\end{align}
For any $\nrm{\theta}\leq 1$, it holds that $\nabla^2 \cLLDPg(\theta)\succeq \frac12\mug\cdot  U\iq$.
Therefore, under the change-of-variable $\theta=Uw$, we define
\begin{align*}
    \cLbarLDP(w)\defeq \Exy{ \frac{\ellg(\lr U\x,w\rr, y)}{\nrm{U\x}} }+\frac{\mug\lambda}{2}\nrm{w}^2_{U},
\end{align*}
and then $\cLbarLDP$ is $(2\Lipg)$-smooth and $(\mug/2)$-strongly convex over the domain $\cW_U\defeq \set{w\in\R^d: \nrm{Uw}\leq 1}$. Further, the gradient of $\cLbarLDP$ can be derived as
\begin{align*}
    \nabla \cLbarLDP(w)=\Exy{ \frac{U\x}{\nrm{U\x}}\paren{ \link(\lr U\x, w\rr)-y }} +\lambda \mug \cdot Uw.
\end{align*}
Hence, $\nabla \cLbarLDP(w)$ can be privately approximated given samples from the GLM (details in \cref{appdx:LDP-l1-regression}). 

Finally, we argue that any approximate minimizer of $\cLbarLDP$ must be close to $U\iv \ths$.
\begin{lemma}\label{lem:LDP-l1-regression-approx}
    Suppose that $\frac12\id \preceq \FLDP(U)\preceq 2\id$ holds. We let $\hwst_U\defeq \argmin_{w: \nrm{Uw}\leq 1} \cLbarLDP(w)$.
    Then it holds that $\nrm{\hwst_U-U\iv \ths}\leq 4\lambda$. 
\end{lemma}

\paragraph{\Iw~objective for DP regression}
Similarly, under DP setting, 
given a matrix $W\in\PSD$ and parameter $\lambda,\gamma>0$ such that $\frac12\id \preceq \FJDP(W)\preceq 2\id$, we introduce
\begin{align}\label{def:cL-JDP-GLM}
    \cLJDPg(\theta)\defeq \EE\brk*{ \frac{\ellg(\lr \x,\theta\rr, y)}{1+\gamma\nrm{W\x}} }+\frac{\mug\lambda}{2} \nrm{\theta}^2_{W\iv},
\end{align}
and under the change-of-variable $\theta=Uw$, we consider 
\begin{align*}
    \cLbarDP(w)\defeq \EE\brk*{ \frac{\ellg(\lr W\x,w\rr, y)}{1+\gamma\nrm{W\x}} }+\frac{\mug\lambda}{2} \nrm{w}^2_{W},
\end{align*}
Then, $\cLbarDP$ is $(\mug/2)$-strongly convex over the domain $\cW_W\defeq \set{w\in\R^d: \nrm{Ww}\leq 1}$. In \cref{appdx:JDP-l1-regression}, we utilize this idea by solving the empirical version of $\cLbarDP$ privately.

\paragraph{Proof of \cref{lem:LDP-l1-regression-approx}}
We first state the following basic lemma.
\begin{lemma}\label{lem:SC-func}
Suppose that $\cW\subseteq \R^d$ is a convex and compact domain, $F$ is a $\mu$-strongly-convex function over $\cW$, and $\wstar=\argmin_{w\in\cW} F(w)$. Then it holds that 
\begin{align*}
    \frac{\mu}{2}\nrm{w-\wstar}^2 \leq \lr \nabla F(w), w-\wstar\rr, \qquad \forall w\in\cW,
\end{align*}
and particularly, $\nrm{w-\wstar}\leq \frac{2}{\mu}\nrm{\nabla F(w)}$.
\end{lemma}
To see this, we note that by strong convexity,
\begin{align*}
    F(\wstar)\geq F(w)+\lr \nabla F(w), \wstar-w\rr + \frac{\mu}{2}\nrm{w-\wstar}^2, \qquad w\in\cW.
\end{align*}
Then, using $F(\wstar)\leq F(w)$, we have the desired result.

Now, we prove \cref{lem:LDP-l1-regression-approx} by noting that $\nabla \cLbarLDP(U\iv \ths)=\lambda \mug \ths$, and hence $\nrm{\hwst_U-\wstar}\leq 4\lambda$ holds.
\qed

\subsection{Local DP Generalized Linear Regression}\label{appdx:LDP-l1-regression}

Similar to the JDP setting, to privately optimize the objective $\cLbarLDP$, we use the following LDP batched SGD subroutine (\cref{alg:LDP-SGD}).

\newcommand{\gradell}[2]{\prn*{ \nu(#1)-#2} }
\begin{algorithm}[H]
\caption{Subroutine $\AlgLDPGD$}\label{alg:LDP-SGD}
\begin{algorithmic}[1]
\REQUIRE Dataset $\dataset=\crl*{(\x_t,y_t)}_{t\in[N]}$, weights $(U,\lambda)$.
\REQUIRE Stepsize schedule $(\eta_t)_{t\in[N]}$.
\STATE Initialize $w_1=0$.
\FOR{$t=1,\cdots,N$}
\STATE Compute gradient estimator
\begin{align*}
    g_t=\frac{U\x_t}{\|U\x_t\|} \cdot \gradell{ \lr U\x_t, w_t \rr }{ y_t }+\lambda\mug \cdot  U w\kk.
\end{align*}
\STATE Privatize $\til g_t\sim \priv[2]{g_t}$ and update
\begin{align*}
    w_{t+1}=\Proj_{\cW}\prn*{w_t-\eta_t \til g_t}.
\end{align*}
\ENDFOR
\ENSURE $\hth=W w_N$.
\end{algorithmic}
\end{algorithm}

\begin{proposition}\label{prop:LDP-GD}
The subroutine $\AlgLDPGD$ (\cref{alg:LDP-SGD}) preserves \aLDP. When the input $(U,\lambda)$ to \cref{alg:LDP-SGD} satisfies $\Uapp$ and $\lambda\in(0,1]$, it holds that
\begin{align*}
    \EE\nrm{w_N-\hwst_U}^2\leq C\frac{d\siga^2}{\mug^2 N},
\end{align*} 
where $C>0$ is an absolute constant. 
\end{proposition}

Based on \cref{alg:LDP-SGD}, we present $\AlgLDPRegression$ (\cref{alg:LDP-L1-regression}) for generalized linear regression. Putting all the results together (\cref{prop:alg-U-LDP}, \cref{prop:LDP-GD} and \cref{lem:LDP-l1-regression-approx}), we have the following guarantee.

\begin{algorithm}
\caption{$\AlgLDPRegression$ %
}\label{alg:LDP-L1-regression}
\begin{algorithmic}[1]
\REQUIRE Dataset $\dataset=\crl*{(\x_t,y_t)}_{t\in[T]}$, parameter $\delta\in(0,1)$.
\STATE Set $N=\frac{T}{2}$.
\STATE Set $(U,\lambda) \leftarrow \LDPLU(\crl*{(\x_t,y_t)}_{t\in[N]},\delta)$.
\STATE Set $\hth \leftarrow \AlgLDPGD(\crl*{(\x_t,y_t)}_{t\in[N+1,T]},(U,\lambda))$.
\ENSURE Estimator $\hth$, $(U,\lambda)$.
\end{algorithmic}
\end{algorithm}

\begin{theorem}[LDP generalized linear regression]\label{thm:LDP-L1-regression}
Let $T\geq 1, \delta\in(0,1)$.
Suppose that in \cref{alg:LDP-L1-regression}, the subroutine $\LDPLU$ is instantiated as in \cref{prop:alg-U-LDP}. Then \whp, the returned matrix $U$ satisfies $\Uapp$, and it holds that 
\begin{align*}
    \EE\brk*{ \nrmn{U\iv(\hth-\ths)}^2 \mid U }\leqsim \lambda^2+ \frac{d\siga^2}{\mug^2 T},
\end{align*}
where the parameter $\lambda=\tbO{B\siga\sqrt{\frac{d+\log(1/\delta)}{T}}}$ is defined in \cref{prop:alg-U-LDP}.

In particular, choosing $\delta=\frac{1}{\poly(T)}$, the upper bound above implies that 
\begin{align*}
    \EE\brk*{ \nrmn{\Ustar\iv(\hth-\ths)}^2 }\leq \frac{d\siga^2}{T}\cdot \prn*{B^2+\mug^{-2}}\poly\log(T).
\end{align*}
\end{theorem}

As a remark, we note that high-probability upper bounds can also be derived (e.g., by applying the results of \citet{rakhlin2011making} to the SGD iterates in \cref{alg:LDP-SGD}).

\newcommand{\iind}[1]{\epk{#1}}
\newcommand{\us}{w^\star}

\subsubsection{\pfref{prop:LDP-GD}}
We invoke the following lemma, which is a well-known result \citep{nemirovski2009robust,rakhlin2011making}.
\begin{lemma}
Suppose that $\cW\subseteq \R^d$ is a convex domain, and $F$ is a $\mu$-strongly-convex function over $\cW$. Consider the following SGD iterates: $w\iind{1}=0$,
    \begin{align}\label{eq:SGD-general}
        w\iind{k+1}=\Proj_\cW\paren{w\iind{k}-\eta\iind{k} g\iind{k} },
    \end{align}
    where $\EE[g\iind{k}\mid w\iind{1},\cdots,w\iind{k}]=\nabla F(w\iind{k})$ and $\EE\nrm{g\iind{k}}^2\leq G^2$ for all $k\in[K]$. Then with the stepsize schedule $\eta\iind{k}=\frac{1}{\mu k}$, it holds that for all $k\in[K]$,
    \begin{align*}
        \EE\nrm{w\iind{k}-\us}^2\leq \frac{16G^2}{\mu^2 k},
    \end{align*}
    where $\us\defeq \argmin_{w\in\cW} F(w)$.
\end{lemma}
Then, we recall that $\cLbarLDP$ is $(\mug/2)$-strongly-convex and $(2\Lipg)$-smooth, and it is clear that in \cref{alg:LDP-SGD}, it holds that
\begin{align*}
    \EE[\til g_t\mid w_1,\cdots,w_t]=\EE[g_t\mid w_1,\cdots,w_t]=\nabla \cLbarLDP(w_t).
\end{align*}
We also have $\nrm{g_t}\leq 2+\lambda \mug\leq 3$, and $\EE\nrm{\til g_t-g_t}^2=\EE_{Z\sim \normal{0, 16\siga^2\id}}\nrm{Z}^2=16\siga^2d$. Combining these conditions gives the desired upper bound.
\qed

\subsection{DP Generalized Linear Regression}\label{appdx:JDP-l1-regression}

\newcommand{\whcL}{\wh{\cL}}

We propose \cref{alg:JDP-L1-regression}, a generalization of \cref{alg:JDP-linear-regression} for GLM. The subroutine $\AlgJDPGD$ (\cref{alg:JDP-SGD}) is based on privately sovling the ERM problem induced by $\cLJDPg$.

\begin{algorithm}
    \caption{$\AlgJDPRegression$ %
    }\label{alg:JDP-L1-regression}
    \begin{algorithmic}[1]
    \REQUIRE Dataset $\dataset=\sset{(\x_t,y_t)}_{t\in[T]}$, \errpara~$\delta\in(0,1)$.
    \STATE Split the dataset $\dataset=\dataset_0 \cup \cD_1$ equally.
    \STATE Set $(W,\gamma,\lambda) \leftarrow \JDPLU(\dataset_0,\delta)$.
    \STATE Set $\hth\leftarrow \AlgJDPGD(\cD_1, (W,\gamma,\lambda))$.
    \ENSURE Estimator $\hth$, $(W,\lambda,\gamma)$.
    \end{algorithmic}
\end{algorithm}

\begin{algorithm}
\caption{Subroutine $\AlgJDPGD$}\label{alg:JDP-SGD}
\begin{algorithmic}[1]
\REQUIRE Dataset $\dataset=\sset{(\x_t,y_t)}_{t\in[N]}$, weights $(W,\gamma,\lambda)$.
\STATE Compute
\begin{align*}
    H\ldef \frac1N\sum_{t=1}^N \frac{W\sq \x_t\x_t\tp W\sq}{1+\gamma\nrm{W\x_t}}+\lambda\id.
\end{align*}
\STATE Privatize $\til H\sim \sympriv[2B/(\gamma N)]{H}$.
\IF{$\til H \succeq \frac14 W\iv$ is not true} \label{line:H-verify}
    \STATE \textbf{Output: }$\hth=0$.
\ENDIF
\STATE Construct
\begin{align*}
    \whcL(w)\ldef \frac1N\sum_{t=1}^N \frac{\ellg(\lr W\x_t, w\rr, y_t)}{1+\gamma\nrm{W\x_t}}+\frac{\mug}{2}\nrm{W\sq w}_{\lambda \id + (\til H-H)}^2.
\end{align*}
\STATE Set $\Delta_N:=\frac{32(\mug\iv+B)}{\gamma N}$. 
\STATE Compute $w$ (by gradient descent on $\whcL$) such that $\nrm{w-\wh{w}}\leq \frac{\Delta_N}{4}$, where $\wh{w}=\argmin_{w\in\cW_W} \whcL(w)$. \label{line:GD-approx}
\STATE Privatize $\til w\sim \priv[2\Delta_N]{w}$
\ENSURE Estimator $\hth= W \til w$.
\end{algorithmic}
\end{algorithm}

We note that in Line \ref{line:GD-approx} of \cref{alg:JDP-SGD}, $\wh{w}$ is well-defined as $\nabla \whcL(w)\succeq \mug W\sq \til H W \succeq \frac{\mug}{4}\id$ ensured by Line~\ref{line:H-verify}. Further, in this step, polynomial number of gradient steps are sufficient thanks to the strong convexity of $\whcL$. We summarize the guarantees of \cref{alg:JDP-SGD} and \cref{alg:JDP-L1-regression} in the following theorem.

\begin{theorem}[DP generalized linear regression]\label{thm:JDP-L1-regression}
    Let $T\geq 1, \delta\in(0,\frac1T]$, $\alpha,\beta\in(0,1)$. Then the subroutine $\AlgJDPGD$ (and hence \cref{alg:JDP-L1-regression}) preserves \aDP. Further, suppose that the parameter $(\gamma,\lambda)$ and $K$ satisfy
    \begin{align*}
        \lambda\gamma\geq C(B+\mug\iv)\frac{\siga\sqrt{d+\log(K/\delta)}}{T}, \qquad K\geq \max\crl*{\log\log(1/\lambda),4},
    \end{align*}
    where $C>0$ is a sufficiently large absolute constant.
    Then, \cref{alg:JDP-L1-regression} ensures that \whp[2\delta], it holds that $\Wapp$, and the returned estimator $\hth$ satisfies
    \begin{align*}
        \nrm{W\iv\prn*{\hth-\ths}}\leq \lambda+\bigO{\sqrt{\frac{d\log(1/\delta)}{\mug^2 T}}}.
    \end{align*}
\end{theorem}

In particular, we can choose $\lambda=\sqrt{\frac{d}{\mug^2 T}}$ and $\gamma=\tbO{\max\crl*{\mug B, 1}\frac{\siga}{\sqrt{T}}}$, so that with high probability,
\begin{align*}
    \nrmn{\Wstar\iv\prn*{\hth-\ths}}\leq \bigO{\sqrt{\frac{d}{\mug^2 T}}}, \qquad\Rightarrow\qquad \nrmn{\hth-\ths}_A^2\leq \bigO{\frac{d}{\mug^2 T}}\cdot \trd{A}{\Wstar},~~\forall A\in\PSD.
\end{align*}
This is analoguous to \cref{thm:JDP-linear-regression-full} (b).

\subsubsection{\pfref{thm:JDP-L1-regression}}

\paragraph{Privacy guarantee}
We first prove that \cref{alg:JDP-SGD} preserves \aDP. To make the argument clearer, we define
\begin{align*}
    H(\cD)\ldef&~ \frac1N\sum_{t=1}^N \frac{W\sq \x_t\x_t\tp W\sq}{1+\gamma\nrm{W\x_t}}+\lambda\id, \\
    \whcL_{\cD;\til H}(w)\ldef&~ \frac1N\sum_{t=1}^N \frac{\ellg(\lr W\x_t, w\rr, y_t)}{1+\gamma\nrm{W\x_t}}+\frac{\mug}{2}\nrm{W\sq w}_{\lambda \id + (\til H-H(\cD))}^2, \\
    \wh{w}(\cD;\til H)=&~\argmin_{w\in\cW_W} \whcL_{\cD;\til H}(w),
\end{align*}
and let $w(\cD;\til H)$ be the approximator of $\wh{w}(\cD;\til H)$ produced in Line~\ref{line:GD-approx}, such that $\nrmn{w(\cD;\til H)-\wh{w}(\cD;\til H)}\leq \frac{\Delta_N}{4}$. Then, \cref{alg:JDP-SGD} can be regarded as the adaptive composition of the following two mechanisms:
\begin{enumerate}
    \item[(a)] $\til H\sim \sympriv[2B/(\gamma N)]{H(\cD)}$.
    \item[(b)] If $\til H\succeq \frac14 W\iv$, output $\til w\sim \priv[\Delta_N/2]{w(\cD;\til H)}$; Otherwise, output $\til w=0$.
\end{enumerate}
It is clear that mechanism (a) is $(\frac{\alpha}{2},\frac{\beta}{2})$-DP. Therefore, it remains to verify that (b) is $(\frac{\alpha}{2},\frac{\beta}{2})$-DP for \emph{any} fixed $\til H$. Then, we only need to consider the case $\til H\succeq \frac14 W\iv$ (as otherwise the mechanism $\til w=0$ is trivially private).

In the following, we show that for any $\til H\succeq \frac14 W\iv$, any neighboring dataset $\cD, \cD'$, it holds that
\begin{align*}
    \nrmn{w(\cD;\til H)-w(\cD;\til H)}\leq \Delta_N,
\end{align*}
and hence guaranteeing that (b) is $(\frac{\alpha}{2},\frac{\beta}{2})$-DP.
Note that for any $w\in\cW_W$,
\begin{align*}
    \whcL_{\cD;\til H}(w)=\frac1N\sum_{t=1}^N \frac{W\x_t }{1+\gamma\nrm{W\x_t}}\partial_1 \ellg(\lr W\x_t, w\rr, y_t)+\mug W\sq\prn*{\lambda \id + (\til H-H(\cD))}W\sq w.
\end{align*}
we have $\nrm{\whcL_{\cD;\til H}(w)-\whcL_{\cD';\til H}(w)}\leq \frac{1+\mug B}{\gamma N}=:\eps$. Now, using the optimality of $\wh{w}(\cD';\til H)$ and the fact that $\whcL_{\cD';\til H}$ is $\frac{\mug}{4}$-strongly-convex (\cref{lem:SC-func}), we have $\forall w\in\cW_W$,
\begin{align*}
    \frac{\mug}{8}\nrm{w-\wh{w}(\cD';\til H)}^2\leq \lr \nabla \whcL_{\cD';\til H}(w), w-\wh{w}(\cD';\til H)\rr
    \leq \eps\nrm{w-\wh{w}(\cD')}+\lr \nabla \whcL_{\cD;\til H}(w), w-\wh{w}(\cD';\til H)\rr.
\end{align*}
Further, by the optimality of $\wh{w}(\cD;\til H)$, we also have
\begin{align*}
    \lr \nabla \whcL_{\cD;\til H}(\wh{w}(\cD;\til H)), w-\wh{w}(\cD;\til H)\rr\geq 0, \qquad \forall w\in\cW_W.
\end{align*}
Combining the inequalities above, we can conclude that
\begin{align*}
    \nrmn{\wh{w}(\cD;\til H)-\wh{w}(\cD;\til H)}\leq \frac{8(\mug\iv+B)}{\gamma N}=\frac{\Delta_N}{2},
\end{align*}
and hence the desired statement follows from the definition that $\nrmn{w(\cD;\til H)-\wh{w}(\cD;\til H)}\leq \frac{\Delta_N}{4}$.

\paragraph{Utility guarantee} 
We follow the proof of \cref{lem:hth-JDP-property}, and condition on $\Wapp$. 

Note that $\til H=H(\cD)+Z$ with $Z_{ij}\sim \normal{0, \frac{16B^2\siga^2}{\gamma^2N^2}}$ independently, and when $\til H\succeq \frac14\id$, $\til w=w(\cD;\til H)+z$ with $z\sim \normal{0, 16\Delta_N^2\siga^2\id}$. Therefore, by \cref{lem:Gaussian-concen}, \cref{lem:cov-concen}, and \cref{lem:Bernstein}, the following inequalities hold \whp:

(a) 
\begin{align*}
    H(\cD)= \frac{1}{T}\sum_{t=1}^T \frac{W\sq\x_t\x_t\tp W\sq}{1+\gamma\nrm{W\x_t}}\succeq \frac12\Ex{\frac{W\sq\x\x\tp W\sq}{1+\gamma\nrm{W\x}}}-\frac{\lambda}{4}\id.
\end{align*}

(b) $\nrm{z}\leq \frac{\lambda }{4}, \nrmop{Z}\leq \frac{\lambda}{4}$.

(c) 
\begin{align*}
    \nrm{\frac1N\sum_{t=1}^N \frac{W\x_t }{1+\gamma\nrm{W\x_t}}\gradell{ \lr \x_t, \ths \rr }{ y_t }}\leq \sqrt{\frac{4d\log(8/\delta)}{N}}+\frac{2\log(8/\delta)}{\gamma N}.
\end{align*}

Note that (a) and (b) imply that $\til H\succeq \frac14\id$, and (b) and (c) further imply that
\begin{align*}
    \nrmn{\nabla \whcL_{\cD;\til H}(W\iv \ths)}\leq \bigO{\sqrt{\frac{d\log(1/\delta)}{N}}}+\frac{\mug\lambda}{8}.
\end{align*}
Therefore, using the $(\mug/4)$-strong-convexity of $\cL_{\cD;\til H}$ and \cref{lem:SC-func} yields $\nrm{W\iv \ths-\wh{w}}\leq \frac{8}{\mug}\nrmn{\nabla \whcL_{\cD;\til H}(W\iv \ths)}$. Combining these inequalities with $\nrmn{w-\wh{w}}\leq \frac{\Delta_N}{4}$ completes the proof.
\qed

\section{Proofs from \cref{sec:cb}}\label{appdx:CB}

\subsection{Spanner}\label{appdx:spanner}

In $\AlgPlan$ (\cref{alg:plan}), we utilize the following notion of the \emph{spanner} of a set of actions.
\begin{definition}\label{def:spanner}
Given a context $x\in\cX$ and a set $\cA_1$ of actions, a subset $\cA'\subseteq \cA_1$ is a \emph{$c$-spanner} of $\cA_1$ if for any $a\in\cA$, there exists weights $(\gamma_{a'}\in[-c,c])_{a'\in\cA'}$ such that
\begin{align*}
    \phxa=\sum_{a'\in\cA'} \gamma_{a'} \phxa[x,a'].
\end{align*}
\end{definition}

It is well-known that a spanner with size bounded by the dimension exists, known as the \emph{barycentric spanner}~\citep{awerbuch2008online}.
\begin{lemma}[Barycentric spanner]\label{lem:bary-spanner}
For any context $x\in\cX$ and set $\cA_1\subseteq \cA$, there exists a $1$-spanner of size
\begin{align*}
    \dim(\cA_1,x)\defeq \dim(\sset{a\in\cA_1: \phxa}).
\end{align*}
\end{lemma}
Further, an approximate barycentric spanner (which is a $2$-spanner) can be computed in time $\poly(d,|\cA_1|)$. Further, given a linear optimization oracle over the set $\sset{a\in\cA_1: \phxa}$, the time complexity can further be reduced to $\poly(d)$~\citep{hazan2016volumetric,perchet2016batched}. Therefore, in the implementation of \cref{alg:plan}, we always consider $2$-spanner.

\subsection{Meta Batch Elimination Algorithm and Guarantees}\label{appdx:regret-meta}

\newcommandx{\Epi}[2][1=\pi]{\EE^{#1}\brac{#2}}
\newcommand{\Px}[1]{\PP_{x\sim \mu}\paren{#1}}

To provide a unified analysis framework for \cref{alg:batch-cb-JDP} with different private regression subroutines, we first present a general action elimiation algorithm (\cref{alg:batch-cb-meta}) that additionally takes an $\Lone$-regression subroutine $\AlgCIEst$ as input.

\begin{algorithm}
\begin{algorithmic}
\REQUIRE Round $T\geq 1$, epoch schedule $1=T_0<T_1<T_2<\cdots<T_{J}=T$.
\REQUIRE Subroutine $\AlgCIEst$.
\STATE Initialize $\hft[0]\equiv 0, \CIt[0]\equiv 1$.
\FOR{$j=0,1,\cdots,J-1$}
\STATE Set $\pit[j]\leftarrow\AlgPlan(\set{ (\hft,\CIt)  }_{\tau<j})$.
\STATE Initialize the subroutine $\AlgCIEst\epj$ with round $N\epj=T\epj[j+1]-T\epj$.
\FOR{$t=T_j,\cdots,T_{j+1}-1$}
\STATE Receive context $x_t$, take action $a_t\sim \pi\epj(x_t)$, and receive reward $r_t$.
\STATE Feed $(\phxa[x_t,a_t],r_t)$ into $\AlgCIEst\epj$.
\ENDFOR
\STATE Receive estimation $(\hft[j],\CIt[j])$ from $\AlgCIEst\epj$.
\ENDFOR
\end{algorithmic}
\caption{Meta Batch Elimination Algorithm}\label{alg:batch-cb-meta}
\end{algorithm}

Similar to our argument in \cref{appdx:JDP-verify}, we can show that \cref{alg:batch-cb-meta} preserves \aJDP~(\aLDP) if the subroutine $\AlgCIEst$ preserves \aJDP~(\aLDP).

Furthermore, we state the regret guarantee of \cref{alg:batch-cb-meta} under the following assumption on the subroutine $\AlgCIEst$, which only requires that the confidence bound $\abs{\lr \x, \hth-\ths\rr}\leq \CI(\x)$ holds true in a \emph{distributional} sense (cf. \cref{lem:hth-JDP-property}).

\begin{assumption}\label{asmp:EstCI}
For each $j$, the subroutine $\AlgCIEst$ returns $(\hft[j],\CIt[j])$ such that the following holds \whp.

(1) The function $\CIt[j](x,a)$ provides a valid confidence bound:
\begin{align*}
    \Px{ \forall a\in\cA, \abs{ \hft(x,a)-\fs(x,a) }\leq \CIt(x,a) }\geq 1-\delta.
\end{align*}

(2) The function $\CIt[j](x,a)=b\epj\paren{\phxa}$ is given by a norm function $b\epj$ over $\Rd$.
\end{assumption}

\newcommand{\success}{\mathsf{success}}

\begin{theorem}[Meta regret guarantee]\label{thm:regret-upper-meta}
If the subroutine $\AlgCIEst$ preserves \aJDP~(or correspondingly \aLDP), then \cref{alg:batch-cb-meta} preserves \aJDP~(or correspondingly \aLDP).

Under \cref{asmp:EstCI}, \cref{alg:batch-cb-meta} ensures that
\begin{align*}
    \Reg\leq 4\dA\cdot \EE\brac{ \indic{\success}\sum_{j=0}^{J-2} N\epj[j+1] \Epi[{\pit[j]}]{ \CIt[j](x,a) } } + 2N\epj[0] + 4TJ\delta,
\end{align*}
where $N\epj=T\epj[j+1]-T\epj$ is the batch size of the $j$th epoch, and $\indic{\success}$ is the success event that \cref{asmp:EstCI} holds for each $j\geq 0$.

Additionally, under \cref{asmp:gap}, we have
\begin{align*}
    \Reg\leq 8\dA\cdot \EE\brac{ \indic{\success}\sum_{j=0}^{J-2} N\epj[j+1] \Epi[{\pit[j]}]{ \til{b}\epj(x,a) } } + 2N\epj[0] + 4TJ\delta,
\end{align*}
where $\til{b}\epj(x,a)=\CIt[j](x,a)\indic{\CIt[j](x,a)\geq \frac{\gap}{8\dA}}$ is the clipped confidence bound.
\end{theorem}

\paragraph{Proof of \cref{thm:regret-upper-meta}}
For any policy $\pi:\cX\to\Delta(\cA)$, we define its sub-optimality as
\begin{align*}
    \reg(\pi)=\EE_{x\sim P, a\sim \pi(x)}\brac{ \fs(x,\pis(x))- \fs(x,a) },
\end{align*}
where we recall that $\pis(x)\defeq \argmax_{a\in\cA} \fs(x,a)$ is the optimal policy under $\fs$.
Then, by definition, for \cref{alg:batch-cb-meta},
\begin{align*}
    \Reg=\EE\brac{ \sum_{j=0}^{J-1} \sum_{t=T\epj}^{T\epj[j+1]-1} \reg(\pi\epj)}
    =\EE\brac{ \sum_{j=0}^{J-1} N\epj \cdot \reg(\pi\epj)}.
\end{align*}

In the following, we work with the following quantity:
\begin{align*}
    \Reg^+\defeq \sum_{j=0}^{J-1} N\epj \cdot \reg(\pi\epj),
\end{align*}
which is a random variable measuring the cumulative sub-optimality of the algorithm.

We assume the success event of \cref{asmp:EstCI} and define
\begin{align*}
    \cX\epj\defeq \sset{ x\in\cX: \forall \tau\leq j, a\in\cA,  \abs{ \hft(x,a)-\fs(x,a) }\leq \CIt(x,a) }.
\end{align*}
Then, for any $x\in\cX\epj$ and $\tau\leq j$, we have
\begin{align*}
    \hft(x,\pis(x))+\CIt(x,\pis(x))\geq \fs(x,\pis(x))
    =\max_{a\in\cA} \fs(x,a)
    \geq \max_{a\in\cA} \hft(x,a)-\CIt(x,a),
\end{align*}
and hence
$\pis(x)\in\cAxt[j]$. Further, for any $a\in\cAxt[j]$,
\begin{align*}
    \hft[j](x,a)+\CIt[j](x,a)\geq \hft[j](x,\pis(x))-\CIt[j](x,\pis(x)),
\end{align*}
and hence
\begin{align*}
    \fs(x,\pis(x))-\fs(x,a)
    \leq&~ \hft[j](x,\pis(x))+\CIt[j](x,\pis(x))-\hft[j](x,a)+\CIt[j](x,a)\\
    \leq&~ 2\CIt[j](x,\pis(x))+2\CIt[j](x,a) \\
    \leq&~ 4\max_{a'\in\cAxt[j]}\CIt[j](x,a'), \qquad \forall x\in\cX\epj, a\in\cAxt[j].
\end{align*}
Note that $\pi\epj[j+1](x)$ is supported on $\cAxt[j]$, and hence
\begin{align*}
    \reg(\pi\epj[j+1])=\EE_{x\sim P, a\sim \pi\epj[j+1](x)}\brac{ \fs(x,\pis(x))- \fs(x,a) }
    \leq 4\EE_{x\sim P} \max_{a\in\cAxt[j]}\CIt[j](x,a).
\end{align*}
Further, $\cAsp[j]$ is a barycentric spanner of $\cAxt[j-1]$, and hence for any $a\in\cAxt[j]\subseteq \cAxt[j-1]$, there exists parameters $(\gamma_{x,a,a'}\in[-1,1])_{a'\in\cAsp[j]}$, such that
\begin{align*}
    \phxa=\sum_{a'\in\cAsp[j]} \gamma_{x,a,a'} \phxa[x,a'].
\end{align*}
Hence, by \cref{asmp:EstCI} (2), for any $x\in\cX\epj$, $a\in\cAxt[j]$,
\begin{align*}
    \CIt[j](x,a)=&~ b\epj\paren{ \sum_{a'\in\cAsp[j]} \gamma_{x,a,a'} \phxa[x,a'] } 
    \leq \sum_{a'\in\cAsp[j]} \abs{\gamma_{x,a,a'}} b\epj(\phxa[x,a']) \\
    \leq&~ \sum_{a'\in\cAsp[j]} b\epj(\phxa[x,a'])
    = |\cAsp[j]|\cdot \EE_{a'\sim \pi\epj(x)} \CIt[j](x,a').
\end{align*}
Therefore, for any $x\in\cX\epj$,
\begin{align*}
    \max_{a\in\cAxt[j]}\CIt[j](x,a)\leq \dA \cdot \EE_{a'\sim \pi\epj(x)} \CIt[j](x,a'),
\end{align*}
and thus,
\begin{align*}
    \reg(\pit[j+1])
    \leq 4\dA\cdot \EE_{x\sim P, a\sim \pi\epj(x) }{ \CIt[j](x,a) } + 2P(x\not\in\cX\epj).
\end{align*}
By \cref{asmp:EstCI}, $P(x\not\in\cX\epj)\leq J\delta$, and hence taking summation over $j=0,1,\cdots,J-2$ gives
\begin{align*}
    \sum_{j=1}^{J-1} N\epj \cdot \reg(\pi\epj)
    \leq 4\dA\sum_{j=0}^{J-2} N\epj[j+1] \cdot \EE^{\pi\epj[j]}\brac{ \CIt[j](x,a) } +2TJ\delta.
\end{align*}
Note that the above inequality holds under the success event of \cref{asmp:EstCI}, which holds with probability at least $1-J\delta$. Then taking expectation gives the desired upper bound on $\Reg$.

Next, we show the regret bound under additional \cref{asmp:gap}. We have shown that
\begin{align*}
    \max_{a\in\cAxt[j]}\fs(x,\pis(x))-\fs(x,a)
    \leq&~ 4\max_{a'\in\cAxt[j]}\CIt[j](x,a')\leq 4\sum_{a'\in\cAsp[j]} b\epj(\phxa[x,a']), \qquad \forall x\in\cX\epj.
\end{align*}

Therefore, with \cref{asmp:gap}, if $\sum_{a'\in\cAsp[j]} b\epj(\phxa[x,a'])<\frac{\gap}{4}$, we have $\fs(x,a)=\fs(x,\pis(x))$ for all $a\in\cAxt[j]$. Hence, the following inequality holds
\begin{align*}
    \max_{a\in\cAxt[j]}\fs(x,\pis(x))-\fs(x,a)
    \leq&~ 8\sum_{a'\in\cAsp[j]}\CIt[j](x,a')\indic{\CIt[j](x,a')\geq\frac{\gap}{8\dA}}, \qquad \forall x\in\cX\epj.
\end{align*}
In particular, we can bound
\begin{align*}
    \reg(\pit[j+1])
    \leq 8\dA\cdot \EE_{x\sim P, a\sim \pi\epj(x) }\brk*{ \CIt[j](x,a)\indic{\CIt[j](x,a)\geq \frac{\gap}{8\dA}} } + 2P(x\not\in\cX\epj),
\end{align*}
which gives the desired regret bound through our previous argument.
\qed

We also remark that a high-probability upper bound on the regret follows similarly (with an extra step of applying martingale concentration).

\subsection{Proof of \cref{thm:regret-upper-JDP-better} and \cref{thm:JDP-cb-gap}}\label{appdx:proof-regret-upper-JDP-better}

We begin by applying \cref{thm:regret-upper-meta}, the regret upper bound of the meta-algorithm \cref{alg:batch-cb-meta}.
To see how \cref{alg:batch-cb-meta} recovers \cref{alg:batch-cb-JDP}, we consider instantiate it with the subroutine $\AlgCIEst$ be specified by $\JDPLinearRegression$ (\cref{alg:JDP-linear-regression}), with parameter $(\gamma\epj,\lambda\epj)$ chosen 
according to \cref{lem:JDP-cb-logA}, and for the output $\hth\epj, (W\epj,\gamma\epj,\lambda\epj)$ of the $j$th instance $\JDPLinearRegression\epj$, we set
\begin{align*}
    \hft[j](x,a)=\nu(\lr \phxa, \hth\epj \rr), \qquad
    \CIt[j](x,a)=10\lambda \epj \nrm{W\epj \phxa}, \qquad \forall (x,a)\in\cX\times\cA.
\end{align*}
Then, it is clear that under these specifications, \cref{alg:batch-cb-meta} agrees with \cref{alg:batch-cb-JDP}. 

We now provide the guarantee of the subroutine $\JDPLinearRegression$, according to \cref{appdx:proof-JDP-linear-regression}. In the following, to simplify presentation, we assume $\delta\leq \frac1T$.
\begin{lemma}\label{lem:JDP-cb-logA}
For each epoch $j\geq 0$, the instance $\JDPLinearRegression\epj$ can be instantiated so that
\begin{align*}
    \lambda\epj=\bigO{\sqrt{\frac{\dA\log(1/\delta)}{N\epj}}}, \qquad \gamma\epj=\bigO{\frac{\siga\sqrt{d+\log(1/\delta)} +\log(1/\delta)}{\lambda\epj N\epj} },
\end{align*}
so that \whp, the following holds: 

(1) For the distribution $p\epj$ of $\x=\phxa$ under $x\sim P, a\sim \pi\epj(x)$, it holds that
\begin{align*}
    \EE_{\x\sim p\epj} \nrm{W\epj \x}\leq 2(\sqrt{d}+d\gamma\epj),
\end{align*}
and
\begin{align*}
    \EE_{\x\sim p\epj} \brk*{\nrm{W\epj \x}\indic{\nrm{W\epj\x}\geq G} }\leq 2d(G^{-1}+\gamma\epj), \qquad \forall G>0.
\end{align*}

(2) It holds that
\begin{align*}
    \PP_{x\sim P}\paren{ \exists a\in\cA, \abs{\lr \phxa, \hth\epj-\ths\rr}\geq \CI\epj(\phxa) } \leq \delta.
\end{align*}
\end{lemma}

Here, (1) follows from $\FJDP[\gamma\epj,\lambda\epj](W\epj)\preceq 2\id$ under the distribution $p=p\epj$ and \cref{lem:E-nrm-W-bound}, and (2) follows from applying \cref{lem:hth-JDP-property} (3) with $\Prjm$ to be projection matrix into the subspace spanned by the set $\crl*{\phxa: a\in\cA}$ with $x\sim P$.

\newcommand{\delz}{\delta_0}
\newcommand{\lamz}{\lambda_0}

Therefore, \cref{asmp:EstCI} holds, and
\begin{align*}
    \EE^{\pi\epj}\brac{\CIt[j](x,a)}=&~8\lambda\epj \EE^{\pi\epj}\brac{\nrm{W\epj \phxa}}
    =8\lambda\epj \EE_{\x\sim p\epj}\brac{\nrm{W\epj \x}}  \\
    \leq&~ 16\lambda\epj \prn*{\sqrt{d}+d\gamma\epj} 
    =\bigO{\sqrt{\frac{\dA d\log(1/\delta)}{N\epj}}}+ \tbO{\frac{\siga d\sqrt{d+\log(1/\delta)}+d\log(1/\delta)}{N\epj}},
\end{align*}
where $\tbO{\cdot}$ hides polynomial factors of $\log\log(T)$.
Then, \cref{thm:regret-upper-meta} yields
\begin{align*}
    \Reg\leqsim &~ \dA\sqrt{\dA d\log(1/\delta)}\sum_{j=0}^{J-2} \frac{N\epj[j+1]}{\sqrt{N\epj}} + \tbO{\siga \dA d\sqrt{d+\log(1/\delta)}+\dA d\log(1/\delta)} \sum_{j=0}^{J-2} \frac{N\epj[j+1]}{N\epj} 
    +N\epj[0]+TJ\delta.
\end{align*}
In particular, with the choice $T\epj=2^{j+1}-1$ ($N\epj=2^j$) and $\delta=\frac{1}{T}$, we have
\begin{align*}
    \Reg\leqsim \sqrt{\dA^3 dT\log T}+\tbO{\siga \dA d\sqrt{d+\log T}\cdot\log T+\dA d \log^2(T)}.
\end{align*}
This gives the regret bound of \cref{thm:regret-upper-JDP-better}.

Similarly, for the gap-dependent bound, we apply the second inequality of \cref{thm:regret-upper-meta}. By \cref{lem:JDP-cb-logA} (1), we have 
\begin{align*}
    \EE^{\pi\epj}\brac{\CIt[j](x,a)\indic{\CIt[j](x,a)\geq\frac{\gap}{8\dA}}}
    =&~8\lambda\epj \EE_{\x\sim p\epj}\brac{\nrm{W\epj \x}\indic{\nrm{W\epj \x}\geq \frac{\gap}{64\dA \lambda\epj}}} \\
    \leq&~ 16\lambda\epj d\prn*{\frac{64\dA \lambda\epj}{\gap}+\gamma\epj} \\
    \leq&~ \bigO{\frac{\dA^2 d\log(1/\delta)}{\gap N\epj}}+ \tbO{\frac{\siga d\sqrt{d+\log(1/\delta)}+d\log(1/\delta)}{N\epj}}
\end{align*}
Therefore, applying \cref{thm:regret-upper-meta} with the choice $T\epj=2^{j+1}-1$ ($N\epj=2^j$) and $\delta=\frac{1}{T}$, we have
\begin{align*}
    \Reg\leqsim \frac{\dA^3d\log^2(T)}{\gap}+\tbO{\siga \dA d\sqrt{d+\log T}\cdot\log T},
\end{align*}
where $\tbO{\cdot}$ again hides polynomial factors of $\log\log(T)$. This gives the desired upper bound of \cref{thm:JDP-cb-gap}.
\qed

\begin{remark}
We also note that by slightly modifying the analysis of \cref{lem:hth-JDP-property}, we can replace a factor of $\dA$ by $\log|\cA|$, leading to the following regret bound:
\begin{align*}
    \Reg\leq \tbO{ \dA\sqrt{dT\log|\cA|} +\dA(\siga d^{3/2}+d\log|\cA|)},
\end{align*}
where $\tbO{\cdot}$ also hides $\poly(\log T)$ factors.
\end{remark}

\subsection{JDP regret bounds for Generalized Linear Contextual Bandits}\label{appdx:proof-JDP-gen-linear}

For generalized linear contextual bandits, suppose that we suitably instantiate \cref{alg:batch-cb-meta}, with the subroutine $\AlgCIEst$ specified by $\AlgJDPRegression$ (\cref{alg:JDP-L1-regression}, instantiated as in \cref{thm:JDP-L1-regression}). 

In the $j$th epoch, for the output $\hth\epj, (W\epj,\gamma\epj,\lambda\epj)$ of the instance $\AlgJDPRegression\epj$, we set
\begin{align*}
    \hft[j](x,a)=\nu(\lr \phxa, \hth\epj \rr), \qquad
    \CIt[j](x,a)=8\Lipg \lambda \epj \nrm{W\epj \phxa}, \qquad \forall (x,a)\in\cX\times\cA.
\end{align*}
Following \cref{thm:regret-upper-JDP-better}, we let the epoch schedule be $T\epj=2^j$ for $j=0,1,\cdots$ and $\delta=\frac{1}{T}$. We then have the following regret guarantee for generalized linear contextual bandits.

\begin{theorem}[Regret upper bound under JDP]\label{thm:regret-upper-JDP}
Under the above specifications,
\cref{alg:batch-cb-meta} preserves \aJDP, and
\begin{align*}
\textstyle
    \Reg\leq \tbO{\dA d \sqrt{\kpg^3 T} + \siga \dA d^{3/2}\kpg^2},
\end{align*}
where $\tbO{\cdot}$ hides polynomial factors of $\log(T)$.
\end{theorem}

\paragraph{Proof of \cref{thm:regret-upper-JDP}}
For each epoch $j$, Under the success event of \cref{thm:JDP-L1-regression}, we have
\begin{align*}
    \abs{ \hft[j](x,a)-\fs(x,a) }\leq \CIt[j](x,a), \qquad \forall x\in\cX, a\in\cA,
\end{align*}
and we also have $\Epi[{\pit[j]}]{ \CIt[j](x,a) }\leq 16\Lipg\lambda\epj(\sqrt{d}+d\gamma\epj)$ by \cref{lem:E-nrm-W-bound}.

Then, \cref{thm:regret-upper-meta} yields
\begin{align*}
    \Reg\leqsim&~ \dA \Lipg \sum_{j=0}^{J-2} N\epj[j+1] \lambda\epj(\sqrt{d}+d\gamma\epj) +N\epj[0] \\
    \leq&~ \poly(\log T)\cdot  \dA d\sum_{j=0}^{J-2} \paren{ \kpg^{3/2} \frac{N\epj[j+1]}{\sqrt{N\epj}}+\Lipg \kpg \sqrt{d} \frac{N\epj[j+1]}{N\epj} }.
\end{align*}
In particular, under the choice $T\epj=2^{j}$, we have 
\begin{align*}
    \Reg\leq \tbO{\dA d \kpg^{3/2} \sqrt{T} + \siga \kpg^2 \dA d^{3/2}}.
\end{align*}
This is the desired upper bound.
\qed

\subsection{Proof of \cref{thm:regret-upper-LDP}}\label{appdx:regret-upper-LDP}

To derive the algorithm for LDP setting, in \cref{alg:batch-cb-meta}, we instantiate the subroutine $\AlgCIEst$ with $\AlgLDPRegression$ (\cref{alg:LDP-L1-regression}). 

In the $j$th epoch, we let $\hth\epj, (U\epj,\lambda\epj)$ be the output of the $j$th subroutine $\AlgLDPRegression\epj$, and we let
\begin{align*}
    \hft[j](x,a)=\lr \phxa, \hth\epj \rr, \qquad
    \CIt[j](x,a)=8\lambda \epj \cdot \nrmn{U\epj \phxa}, \qquad \forall (x,a)\in\cX\times\cA,
\end{align*}
following \cref{alg:batch-cb-JDP}. Recall that the parameter $\lambda\epj$ is chosen as
\begin{align*}
    \lambda\epj=\tbO{\siga \sqrt{\frac{d+\log(1/\delta)}{T}}},
\end{align*}
which is defined in \cref{thm:LDP-linear-regression-full} and $\tbO{\cdot}$ hides polynomial factors of $\log\log(T)$. 
Note that under the success event of \cref{thm:LDP-linear-regression-full}, \cref{asmp:EstCI} holds with $\delta_0=0$ under the above specifications, and we also have $\EE^{\pi\epj}\brac{\CIt[j](x,a)}\leq 16d\lambda\epj$. Therefore, applying \cref{thm:regret-upper-meta} yields
\begin{align*}
    \Reg\leqsim&~ \dA d\sum_{j=0}^{J-2} N\epj[j+1] \cdot \lambda\epj + N\epj[0]+TJ\delta.
\end{align*}
In particular, with the epoch schedule $T\epj=2^{j}$ and $\delta=\frac1T$, we have
\begin{align*}
    \Reg\leq \tbO{ \dA\sqrt{d^3T\log T} }.
\end{align*}
This is the desired upper bound.
\qed

\begin{remark}
By replacing $\LDPLinearRegression$ with $\AlgLDPRegression$ (\cref{alg:LDP-L1-regression}), we can also get a regret guarantee for generalized linear contextual bandits. We omit the proof for succintness. 
\end{remark}

\section{Proofs from \cref{sec:unbounded}}\label{appdx:unbouned}

\subsection{Proof of \cref{thm:improper-JDP}}\label{appdx:improper-JDP}

We first present the algorithm (\cref{alg:JDP-improper-GD}).

\renewcommand{\BR}{\Bone}
\newcommand{\ConstR}{4}

\begin{algorithm}
\caption{DP Stochastic Gradient Descent}\label{alg:JDP-improper-GD}
\begin{algorithmic}[1]
\REQUIRE Dataset $\dataset=\sset{(\x_t,y_t)}_{t\in[T]}$.
\REQUIRE Stepsize $\eta\in(0,\frac12]$. %
\STATE Initialize $\theta_1=\bz$.
\FOR{$t=1,\cdots,T$}
\STATE Set
\begin{align*}
    \theta_{t+1}=\Proj_{\Bone}\paren{\theta_t-\eta \x_t\paren{ \lr \x_t, \theta_t\rr-y_t }}.
\end{align*}
\ENDFOR
\STATE Aggregate $\Bar{\theta}=\frac1T\sum_{t=1}^{T}\theta_t$ and privatize $\hth\sim \priv[2\eta]{\Bar{\theta}}$.
\ENSURE Estimator $\hth$.
\end{algorithmic}
\end{algorithm}

\newcommand{\cDsub}{\cD_{\mathsf{sub}}}

\paragraph{Privacy guarantee}
To make the presentation clear, we re-write the iteration of \cref{alg:JDP-improper-GD} as follows. Denote $g(\theta;\x,y)=\x\paren{ \lr \x, \theta\rr-y }$, and
\begin{align*}
    F(\theta;z)=\Proj_{\Bone}\paren{\theta-\eta g(\theta;z)}.
\end{align*}
Then, the iteration of \cref{alg:JDP-improper-GD} can be re-written as follows:
$\theta\kz=\bz$, and for $k=0,1,\cdots,K-1$, 
\begin{align*}
    \theta_{t+1}(\cD)\defeq F(\theta_t(\cD);z_t).
\end{align*}
Note that for any $z=(\x,y)$ with $\nrm{\x}\leq 1$, $g(\theta;z)=\nabla \paren{ \frac{1}{2} (\lr \x,\theta\rr-y)^2 }$
is the gradient of a 1-Lipschitz convex function, and hence $\theta\mapsto F(\theta;z)$ is a contraction under $\nrm{\cdot}=\nrm{\cdot}_2$, i.e.,
\begin{align*}
    \nrm{\theta-\theta'}\geq \nrm{F(\theta;z)-F(\theta';z)}, \qquad \forall \theta,\theta'.
\end{align*}
Therefore, for neighbored dataset $\cD=\crl*{z_1,\cdots,z_T}$ and $\cD'=\crl*{z_1',\cdots,z_T'}$ with $z_s\neq z_s'$, we have
\begin{align*}
    \nrm{\theta_{t+1}(\cD)-\theta_{t+1}(\cD')}\leq \nrm{\theta_{t}(\cD)-\theta_{t}(\cD')}, \qquad \forall t\neq s,
\end{align*}
and $\nrm{\theta_{s+1}(\cD)-\theta_{s+1}(\cD')}\leq \nrm{\theta_{s}(\cD)-\theta_{s}(\cD')}+4\eta$. Therefore, it holds that
\begin{align*}
    \nrm{\theta_t(\cD)-\theta_t(\cD')}\leq 4\eta, \qquad \forall t\in[T].
\end{align*}
This immediately show that \cref{alg:JDP-improper-GD} preserves \aDP~by the privacy of Gaussian channels (\cref{def:Guassian-channel}).
\qed

\newcommand{\og}{\Bar{g}}
\newcommand{\gerr}[1]{\mathsf{err}\epk{#1}}
\paragraph{Utility guarantee}
The upper bound follows immediately from the guarantee of online gradient descent. For completeness, we provide the proof below.

We denote $g_t=\x_t\paren{ \lr \x_t, \theta_t\rr-y_t }$ and rewrite the iteration as 
\begin{align*}
    \theta_{t+1}=\Proj_{\Bone}\paren{\theta_t-\eta g_t}.
\end{align*}
Then, we have $\nrm{\theta_{t+1}-\ths}\leq \nrm{\theta_t-\eta g_t - \ths}$, and hence
\begin{align*}
    \lr g_t, \theta_t-\ths\rr \leq \frac{\nrm{\theta_t-\ths}^2-\nrm{\theta_{t+1}-\ths}^2}{2\eta}+\frac{\eta\nrm{g_t}^2}{2}.
\end{align*}
Taking summation over $t\in[T]$ and using $\nrm{g_t}\leq 2$, we know
\begin{align*}
    \sum_{t=1}^T \lr g_t, \theta_t-\ths\rr  \leq \frac{1}{2\eta}+T\eta.
\end{align*}
Let $f(\theta)=\frac12\EE_{(\x,y)\sim M}\prn*{\lr \x,\theta\rr-y}^2$. Then, we know $\EE_{t-1}[g_t]=\nabla f(\theta_t)$, where $\EE_{t-1}[\cdot]$ is the conditional expectation of $z_t\mid z_1,\cdots,z_{t-1}$. Therefore, by martingale concentration inequality, we have \whp~that
\begin{align*}
    \sum_{t=1}^T \lr \nabla f(\theta_t) -  g_t, \theta_t-\ths\rr\leq 8\sqrt{T\log(1/\delta)}.
\end{align*}
Further, by the convexity of $f$,
\begin{align*}
    \sum_{t=1}^T \lr \nabla f(\theta_t) , \theta_t-\ths\rr \geq \sum_{t=1}^T \brk*{f(\theta_t)-f(\ths)}
    \geq T\cdot \brk*{f(\Bar{\theta})-f(\ths)}=\frac{T}{2}\nrmn{\Bar{\theta}-\ths}_{\cov}^2.
\end{align*}
Combining the inequalities above, we have shown that \whp,
\begin{align*}
    \nrmn{\Bar{\theta}-\ths}_{\cov}^2\leq \frac{1}{2T\eta}+\eta+8\sqrt{\frac{\log(1/\delta)}{T}}.
\end{align*}

Finally, we know $\hth=\Bar{\theta}+Z$, where $Z\sim \normal{0,4\eta^2\siga^2\id}$ is a Gaussian random vector. Therefore, \whp, $\nrm{Z}_{\bSigma}\leq C\eta\siga\sqrt{\log(1/\delta)}$ (\cref{lem:Gaussian-concen}). Taking the union bound, we know \whp[2\delta], it holds that
\begin{align*}
    \nrmn{\hth-\ths}_{\cov}^2\leq 2\nrmn{\Bar{\theta}-\ths}_{\cov}^2+2\nrm{Z}_{\bSigma}^2
    \leqsim \sqrt{\frac{\log(1/\delta)}{T}}+\frac{1}{T\eta}+\eta+\eta^2\siga^2\log(1/\delta).
\end{align*}
Therefore, we choose
\begin{align*}
    \eta=\min\crl*{ \frac{1}{2\sqrt{T}}, \prn*{\frac{1}{\siga^2 T\log(1/\delta)}}^{1/3} },
\end{align*}
and the desired upper bound follows:
\begin{align*}
    \nrmn{\hth-\ths}_{\cov}^2\leq 2\nrmn{\Bar{\theta}-\ths}_{\cov}^2+2\nrm{Z}_{\bSigma}^2
    \leqsim \sqrt{\frac{\log(1/\delta)}{T}}+\prn*{\frac{\siga\sqrt{\log(1/\delta)}}{T}}^{2/3}.
\end{align*}
\qed

\subsection{Proof of \cref{thm:improper-LDP}}

\newcommand{\sigsp}{\sigma_{N}}

We prove the following convergence rate of \cref{alg:LDP-improper-GD}, under general choice of $(\eta,R)$.
\begin{proposition}\label{prop:unbounded-converge}
Let $K,N\geq 2, \delta\in(0,1)$. We denote $B_\delta\defeq 6(R+1)\sqrt{\frac{K\log(K/\delta)}{N}}$ and $\epsN=(R+1)\sqrt{\frac{K\log(K/\delta)}{N}}$. Suppose that the parameters $(\eta, R)$ are chosen so that
\begin{align}\label{eq:unbounded-cond}
    R\geq 1+\eta\cdot \paren{ B_\delta+4\epsN }.
\end{align}
Then it holds that \whp
\begin{align*}
    \nrm{\theta\kc-\ths}_{\bSigma}\leqsim \frac{1}{\sqrt{\eta K}}+R\eta\siga \sqrt{\frac{K\log(K/\delta)}{N}}.
\end{align*}
\end{proposition}

\paragraph{Proof of \cref{thm:improper-LDP}}
For \cref{alg:LDP-improper-GD}, we choose $\eta=1$, $R=2$, and
\begin{align*}
    K=c\paren{\frac{T}{\siga^2 \log (T/\delta)}}^{1/3}\vee 1, \qquad
    N=\frac{T}{K}, 
\end{align*}
where $c>0$ is an absolute constant so that \eqref{eq:unbounded-cond} holds. Then, by \cref{prop:unbounded-converge}, \cref{alg:JDP-improper-GD} achieves
\begin{align*}
    \nrm{\theta\kc-\ths}_{\bSigma}\leqsim \paren{\frac{\siga^2\log(T/\delta)}{T}}^{1/6}.
\end{align*}
This is the desired upper bound.

\subsubsection{Proof of \cref{prop:unbounded-converge}}

In \cref{alg:LDP-improper-GD}, for epoch $k=0,1,\cdots,K-1$, we have
\begin{align*}
    \til g\kk=\zeta\kk+\avgtk g_t, \qquad
    \theta\kp=\theta\kk-\eta \til g\kk,
\end{align*}
where $\set{\zeta\kz,\cdots,\zeta\epk{K-1}}$ are i.i.d samples from $\normal{0,\sigsp^2}$ with $\sigsp=\frac{(R+1)\siga}{\sqrt{N}}$ and independent of the dataset $\sset{(\x_t,y_t)}_{t\in[T]}$. 

\newcommand{\err}[1]{\mathsf{err}^{(#1)}}
To begin with, we denote $\bSigma\defeq \Ep{\x\x}$ (the covariance matrix), 
\begin{align*}
    \ogd{k}\defeq&~ \EE_{(x,y)\sim p}\brac{ x\paren{ \clip{\lr \pa{k},x \rr}-y }},
\end{align*}
and define the error vectors
\begin{align*}
    \err{k}_0\defeq \nabla \Lsq(\pa{k})-\ogd{k} \qquad
    \err{k}_1\defeq \ogd{k}-\avgtk g^t, \qquad
    \err{k}_2=-\zeta\kk.
\end{align*}
Then, we can decompose the error of the estimator $\gd{k}$ as
\begin{align*}
    \err{k}\defeq \nabla \Lsq(\pa{k}) - \gd{k} = \err{k}_0 + \err{k}_1 +\err{k}_2.
\end{align*}

Notice that by definition, we have
\begin{align*}
    \Lsq(\theta\kk)=\frac12\nrm{\theta\kk-\ths}_{\bSigma}^2, \qquad
    \nabla \Lsq(\pa{k})=\bSigma(\pa{k}-\ths).
\end{align*}
Therefore, recursively using $\gd{k}=\nabla \Lsq(\pa{k})-\err{k}$ and $\pa{k+1}=\pa{k}-\eta \gd{k}$, we have
\begin{align}\label{eqn:proof-unbounded-decomp}
    \pa{k}-\ths=(\id-\eta\bSigma)^k(\pa{0}-\ths)+\eta\sum_{i=0}^{k-1} (\id-\eta\bSigma)^{k-i-1}\err{i}.
\end{align}
We bound the three types of error separately: For each $j\in\set{0,1,2}$, denote
\newcommandx{\Ek}[1][1=k]{E^{(#1)}}
\begin{align*}
    \Ek_j\defeq \sum_{i=0}^{k-1} (\id-\eta\bSigma)^{k-i-1}\err{i}_j.
\end{align*}

\newcommand{\epsNi}[1]{\eps_{N,#1}}

\begin{lemma}\label{lem:unbounded-E1}
\Whp, the following holds:
For all $k\in[K]$, it holds that
\begin{align*}
    \nrm{ \Ek_1 }\leq 6(R+1)\sqrt{\frac{K\log(K/\delta)}{N}}=:B_\delta.
\end{align*}

\end{lemma}

The proof of \cref{lem:unbounded-E1} follows immediately from \cref{lem:vec-concen}, and hence we omit it for succintness.

\begin{lemma}\label{lem:unbounded-E2}
Denote $\epsN\defeq \sigsp\sqrt{K\log(2K/\delta)}$. \Whp, for all $k=0,1,\cdots,K-1$, it holds that
\begin{align*}
    \Pp{ \absn{\lr \Ek_2, \x\rr}> 3\epsN}\leq \frac{1}{K^6}, \qquad
    \Epp\lr \Ek_2, \x\rr^2\leq 4\epsN^2.
\end{align*}
where $C_2$ is an absolute constant. We denote this success event at $\cE_2$.
\end{lemma}

\paragraph{Proof of \cref{lem:unbounded-E2}}
Fix a $k\in[K]$. Then by definition
\begin{align*}
    \Ek_2=\sum_{i=0}^{k} (\id-\eta\bSigma)^{k-i}\zeta\epk{i},
\end{align*}
where $\zeta\epk{i}\sim \normal{0,\sigsp^2\id}$. Therefore, because $\zeta=(\zeta\kz,\cdots,\zeta\kc)$ is a sequence of independent Gaussian random variables, we have
\begin{align*}
    \Ek_2\sim \normal{0, C_k}, \qquad C_k=\sigsp^2\sum_{i=0}^{k}(\id-\eta\bSigma)^{2(k-i)}\preceq K\sigsp^2 \id.
\end{align*}
Therefore, for any fixed $\x\in\Bone$, $\lr \Ek_2,\x\rr\sim \normal{0,\nrm{\x}^2_{C_k}}$ is a zero-mean Gaussian random variable with variance $\nrm{\x}^2_{C_k}\leq K\sigsp^2$. This immediately implies that
\begin{align*}
    \forall \x\in\Bone, \qquad 
    \EE_\zeta\brac{ \exp\paren{ \frac{\absn{\lr \Ek_2,\x\rr}^2 }{4K\sigsp^2} } }\leq 2,
\end{align*}
where the expectation is taken over the sequence $\zeta=(\zeta\kz,\cdots,\zeta\kc)$ of independent Gaussian random vectors. Therefore, we have
\begin{align*}
    \EE_\zeta\brac{ \Epp{\exp\paren{ \frac{\absn{\lr \Ek_2,\x\rr}^2 }{4K\sigsp^2} }} } \leq 2, \quad\forall k.
\end{align*}
Therefore, by Markov's inequality and taking the union bound, we have
\begin{align*}
    \PP_\zeta\paren{ \forall k\in[K]: \Epp{\exp\paren{ \frac{\absn{\lr \Ek_2,\x\rr}^2 }{4K\sigsp^2} } }\leq \frac{2K}{\delta} }\geq 1-\delta.
\end{align*}
Let this event be $\cE_2$. Then, under $\cE_2$, 
using Markov inequality's again, we have
\begin{align*}
    \Pp{\absn{\lr \Ek_2,\x\rr}\geq 3\sigsp\sqrt{K\log(K/\delta)}}\leq \frac{1}{K^6}, \qquad \forall k.
\end{align*}
Similarly, under $\cE_2$, using Jensen's inequality, we also have
\begin{align*}
    \Epp{\absn{\lr \Ek_2,\x\rr}^2}\leq  4\sigsp^2 K\log(2K/\delta).
\end{align*}
This is the desired result.
\qed

\begin{lemma}\label{lem:unbounded-E0}
Under the success event of \cref{lem:unbounded-E2} and assuming that $K\geq 2$ and
\begin{align*}
    R\geq 1+\eta\cdot \paren{ B_\delta+4\epsN }.
\end{align*}
Then, we have for all $k=0,1,\cdots,K-1$:
\begin{align}\label{lem:unbouned-err-0}
    \nrm{\err{k}_0}\leq \frac{2\eta\epsN}{K^3}.
\end{align}
In particular, it holds that $\nrm{\Ek_0}\leq \frac{2\eta\epsN}{K^2}$.
\end{lemma}

\paragraph{Proof of \cref{lem:unbounded-E2}}
We prove by induction. The base case $k=0$ is trivial because $\err{0}_0=0$.

Now, suppose that \eqref{lem:unbouned-err-0} holds for all $k'\leq k$. In the following, we proceed to prove \eqref{lem:unbouned-err-0} for $k+1$.

By \eqref{eqn:proof-unbounded-decomp}, we have
\begin{align*}
    \pa{k+1}=\paren{\id-(\id-\eta\bSigma)^k}\ths+\eta\paren{\Ek_0+\Ek_1+\Ek_2},
\end{align*}
and under $\cE_1\cap \cE_2$, we have $\nrm{\cE_1}\leq C_1\epsN$, 
\begin{align*}
    \Pp{ \absn{\lr \Ek_2, \x\rr}> 3\epsN }\leq \frac{1}{K^6}.
\end{align*}
By induction hypothesis, 
\begin{align*}
    \nrm{\Ek_0}=\nrm{\sum_{i=0}^{k} (\id-\eta\bSigma)^{k-i}\err{i}_0}
    \leq \sum_{i=0}^{k} \nrm{\err{i}_0}\leq \frac{2\epsN}{K^2}.
\end{align*}
Combining all the equations above, we have
\begin{align*}
    \abs{\lr\x, \pa{k+1}\rr}
    \leq 1+\eta\paren{\frac{2\eta\epsN}{K^2}+B_\delta+\abs{\lr \Ek_2, \x\rr}}.
\end{align*}
By our assumption on $R$, we know that $\abs{\lr\x, \pa{k+1}\rr}-R\leq \eta\paren{\absn{\lr \Ek_2, \x\rr}-3\epsN}$, and hence
\begin{align*}
    \Pp{ \absn{\lr\x, \pa{k+1}\rr} \geq R}
    \leq \Pp{ \absn{\lr \Ek_2, \x\rr}> 3\epsN } \leq \frac{1}{K^6}.
\end{align*}
Finally, by the definition of $\err{k+1}_0$, we have
\begin{align*}
    \err{k+1}_0=\nabla \Lsq(\pa{k+1})-\ogd{k+1}
    =\Ep{ \x\paren{ \lr \pa{k+1},\x \rr-\clip{\lr \pa{k+1},\x \rr} }}.
\end{align*}
Hence,
\begin{align*}
    \nrm{\err{k+1}_0}\leq&~ \Ep{\indic{ \absn{\lr \pa{k+1},\x \rr}>R }\cdot \paren{ \absn{\lr \pa{k+1},\x \rr} - R } } \\
    \leq&~ \sqrt{\Pp{ \absn{\lr\x, \pa{k+1}\rr} \geq R} \cdot \Epp{\paren{\absn{\lr \pa{k+1},\x \rr}-R}_+^2  }} \\
    \leq&~ \sqrt{\frac{\eta^2}{K^6}\Epp{\lr \Ek_2,\x \rr^2 } }
    \leq \frac{2\eta\epsN }{K^3}.
\end{align*}
This completes the proof of the step $k+1$.
\qed

Now, we prove \cref{prop:unbounded-converge} by combining the lemmas above. By definition,
\begin{align*}
    \nrm{\theta\kc-\ths}_{\bSigma}\leq \nrm{(\id-\eta\bSigma)^K\ths}+\eta\nrm{\Ek[K]_0}_{\bSigma}+\eta\nrm{\Ek[K]_1}_{\bSigma}+\eta\nrm{\Ek[K]_2}_{\bSigma}.
\end{align*}
Under the success event of \cref{lem:unbounded-E1} and \cref{lem:unbounded-E2}, we have 
\begin{align*}
    \nrm{\Ek[K]_1}_{\bSigma}\leq \nrm{\Ek[K]_1}\leq B_\delta\leqsim \epsN, \qquad
    \nrm{\Ek[K]_2}_{\bSigma}\leq 2\epsN,
\end{align*}
and $\nrm{\Ek[K]_0}\leq \frac{2\eta\epsN}{K^2}$. Further, by \cref{lem:cov-k-converge}, we also have 
\begin{align*}
    \nrm{(\id-\eta\bSigma)^K\ths}\leq \sqrt{\frac{2e}{\eta K}}.
\end{align*}
Combining the inequalities above gives
\begin{align*}
    \nrm{\theta\kc-\ths}_{\bSigma}\leqsim \frac{1}{\sqrt{\eta K}}+\eta \epsN.
\end{align*}
This is the desired upper bound.
\qed

\subsection{Algorithm SquareCB}\label{appdx:square-cb}

\newcommand{\AlgRegression}{\mathsf{Regression}}
\begin{algorithm}
\begin{algorithmic}
\REQUIRE Round $T\geq 1$, epoch schedule $1=T_0<T_1<T_2<\cdots<T_{J}=T$.
\REQUIRE Oracle $\AlgRegression$, parameter $\delta\in(0,1)$.
\STATE Initialize $\hft[0]\equiv 0$.
\FOR{$j=0,1,\cdots,J-1$}
\STATE Initialize the subroutine $\AlgRegression\epj$ with round $N\epj=T\epj[j+1]-T\epj$ and confidence $\delta'=\frac{\delta}{2J^2}$.
\FOR{$t=T_j,\cdots,T_{j+1}-1$}
\STATE Receive context $x_t$.
\STATE Let $\hat{a}_t\defeq \argmax_{a\in\cA} \hft[j](x_t,a)$, and set
\begin{align*}
    p_t(a)\defeq \begin{cases}
        \frac{1}{|\cA|+\gamma\epj\paren{\hft[j](x_t,a_t)-\hft[j](x_t,a)}}, & a\neq \hat{a}_t, \\
        1-\sum_{a\neq \hat{a}_t} p_t(a), & a=\hat{a}_t.
    \end{cases}
\end{align*}
\STATE Take action $a_t\sim p_t$, and receive reward $r_t$.
\STATE Feed $(\phxa[x_t,a_t],r_t)$ into $\AlgRegression\epj$.
\ENDFOR
\STATE Receive estimation $\hft[j+1]$ from $\AlgRegression\epj$.
\ENDFOR
\end{algorithmic}
\caption{$\AlgSQCB$~\citep{foster2020beyond,simchi2020bypassing}}\label{alg:square-cb}
\end{algorithm}

\begin{assumption}\label{asmp:L2-regression}
For any policy $\pi:\cX\to \Delta(\cA)$ and any linear reward function $\fs$, given $N$ independent samples $\sset{(x_t,a_t,r_t)}$ generated as
\begin{align*}
    x_t\sim P,\quad
    a_t\sim \pi(x_t), \quad
    \EE[r_t|x_t,a_t]=\fs(x_t,a_t),
\end{align*}
the regression oracle $\AlgRegression$ (initialized with round $N$ and confidence $\delta$) returns an estimated mean function $\hf:\cX\times\cA\to\R$ such that \whp,
\begin{align*}
    \EE_{x\sim P, a\sim \pi(x)} \paren{\hf(x,a)-\fs(x,a)}^2\leq \cE_{\delta}(N)^2,
\end{align*}
where $\cE_\delta$ is a non-increasing function of $N$. 
\end{assumption}
\begin{theorem}[Guarantee of $\AlgSQCB$]\label{thm:square-cb}
Suppose that \cref{asmp:L2-regression} holds. Then, with parameters
\begin{align*}
    \gamma\epj[0]=1, \qquad \gamma\epj\defeq \frac{\sqrt{|\cA|}}{\cE_{\delta'}(N\epj[j-1])}, \quad j=1,\cdots,J-1,
\end{align*}
$\AlgSQCB$ (\cref{alg:square-cb}) achieves
\begin{align*}
    \Reg\leq C\sqrt{|\cA|}\sum_{j=0}^{J-1} \cE_{\delta'}(N\epj)\cdot N\epj[j+1]+N\epj[0]+T\delta.
\end{align*}
\end{theorem}

Furthermore, the privacy guarantee of \cref{alg:square-cb} can be implied by the privacy guarantee of the regression oracle (similar to \cref{alg:batch-cb-meta}).
\begin{lemma}
If the oracle $\AlgRegression$ preserves \aDP~(or correspondingly \aLDP), then \cref{alg:square-cb} preserves \aJDP~(or correspondingly \aLDP).
\end{lemma}

\paragraph{Proof of \cref{thm:regret-dim-free} (1)}
For JDP learning, we consider instantiating the regression oracle $\AlgRegression$ with the algorithm $\AlgJDPIGD$ (\cref{alg:JDP-improper-GD}). For the output $\hth$ of $\AlgJDPIGD$ given a dataset of size $N$, we consider the estimated mean function $\hf(x,a)\defeq \lr \hth,\phxa\rr$ for all $(x,a)\in\cX\times\cA$. Then, by \cref{thm:improper-JDP}, \cref{asmp:L2-regression} holds with 
\begin{align*}
    \cE_\delta(N)\leqsim \paren{\frac{\log N \log(1/\delta)}{N}}^{1/4}+\paren{\frac{\siga\log(1/\delta)}{N}}^{1/3}.
\end{align*}
Therefore, we choose $T\epj=2^j$ and $\delta=\frac{1}{T}$, and \cref{thm:square-cb} provides the following regret bound:
\begin{align*}
    \Reg\leq \sqrt{|\cA|}\cdot \tbO{T^{3/4}+\siga^{1/3}T^{2/3}}.
\end{align*}
This is the desired result.
\qed

\paragraph{Proof of \cref{thm:regret-dim-free} (2)}
Similarly, for LDP learning, we consider instantiating the regression oracle $\AlgRegression$ with the algorithm $\AlgLDPIGD$ (\cref{alg:LDP-improper-GD}). Then, by \cref{thm:improper-LDP}, \cref{asmp:L2-regression} holds with 
\begin{align*}
    \cE_\delta(N)\leqsim \paren{\frac{\siga^2\log(N/\delta)}{N}}^{1/6}.
\end{align*}
Therefore, we choose $T\epj=2^j$ and $\delta=\frac{1}{T}$, and \cref{thm:square-cb} provides the following regret bound:
\begin{align*}
    \Reg\leq \sqrt{|\cA|}\cdot \tbO{\siga^{1/3}T^{5/6}}.
\end{align*}
This is the desired regret bound.
\qed

\end{document}